\documentclass[12pt]{article}
\usepackage[usenames,dvipsnames]{color}

\usepackage{graphicx}
\usepackage{amsthm}
\usepackage{graphics}
\usepackage{amsmath,,amssymb}
\usepackage{algorithmic}
\usepackage{algorithm}
\usepackage{caption}
\usepackage{subcaption}
\usepackage{pstricks, pst-node}
\usepackage{tikz}
\usepackage[linewidth=1pt]{mdframed}
\usepackage{multirow}
\usepackage[round,colon,authoryear]{natbib}
\usepackage{pgfplots}
\usepackage{hyperref}
\usepackage{bm}
\usepackage{enumitem}
\usepackage[title]{appendix}

\pgfplotsset{compat=1.10}
\usepgfplotslibrary{fillbetween}
\usetikzlibrary{patterns}
\usetikzlibrary{shapes,arrows}
\usetikzlibrary{arrows.meta}

\newtheorem{theorem}{Theorem}

\newtheorem{lemma}{Lemma}
\newtheorem{corollary}{Corollary}

\newtheorem{assumption}{Assumption}

\def\EE{\mathbb{E}}
\def\PP{\mathbb{P}}
\def\RR{\mathbb{R}}
\def\QQ{\mathbb{Q}}
\def\Rcal{\mathcal R}

\def\Ecal{\mathcal E}

\def\Bcal{\mathcal B}
\def\Acal{\mathcal A}
\def\Lcal{\mathcal L}

\def\Kcal{\mathcal K}

\def\Ccal{\mathcal C}
\def\Mcal{\mathcal M}

\def\Scal{\mathcal S}

\def\Tcal{\mathcal T}
\def\Ucal{\mathcal U}
\def\Ccal{\mathcal C}
\def\Gcal{\mathcal G}
\def\Wcal{\mathcal W}
\def\Hcal{\mathcal H}
\def\Vcal{\mathcal V}

\def\argmax{\mathop{\rm argmax}}

\def\bI{\mathbf I}

\def\argmin{\mathop{\rm argmin}}

\tikzstyle{block} = [rectangle, draw, fill=white!80!black, line width=2pt,
    text width=15em, text centered, rounded corners, minimum height=3em]
\tikzstyle{line} = [draw, -latex',line width=2pt]

\begin{document}
\title{Learning Representations through Token Prediction: \\ Geometry, Approximation, and Downstream Guarantees}
\author{Shulei Wang\\ University of Illinois at Urbana-Champaign}
\date{(\today)}

\maketitle

\footnotetext[1]{Address for Correspondence: Department of Statistics, University of Illinois at Urbana-Champaign, 605 E. Springfield Ave., Champaign, IL 61820 (Email: shuleiw@illinois.edu).}

\newpage

\begin{abstract}
	Token prediction is a central pre-training objective for modern language models. Despite its empirical success, why token prediction learns broadly useful representations remains incompletely understood. We develop a statistical framework connecting token prediction with representation geometry, encoder approximation, and downstream performance. Under a softmax prediction head, we show that accurate token prediction organizes token embeddings according to similarities between the distributions of contexts in which different token types appear, as measured by Hellinger distance, with explicit errors governed by prediction accuracy and token frequency. Meanwhile, the contextual representation provides a low-dimensional coordinate for the conditional distribution of the target token relative to these embeddings. We further introduce a self-consistency principle showing that repeated applications of a shared representation block can progressively refine the contextual representation without introducing additional block parameters. Among representations with the same prediction accuracy, this recurrent construction favors those that can be stably reconstructed from their contexts. Finally, we establish downstream guarantees for token generation, token community recovery, and classification by a linear probe, showing how prediction accuracy and recovered geometry translate into performance beyond the pre-training objective. Together, these results explain how the simple objective of predicting tokens can recover semantic geometry and produce broadly useful representations. A controlled simulation illustrates the theoretical mechanisms.
\end{abstract}



\newpage
\section{Introduction}
\label{sc:intro}
Token prediction is a central pre-training objective for modern language models and language representation systems \citep{radford2018improving,kenton2019bert}. Given a context, the objective is to predict a target token. The context can be defined in different ways: it may consist of the preceding tokens, as in autoregressive next-token prediction, or the surrounding unmasked tokens, as in masked-token prediction. Because the prediction targets are constructed automatically from raw text, token prediction requires no human annotation and can be scaled to massive unlabeled corpora. Large-scale pre-training with this objective has enabled modern language models to learn contextual representations that support language generation and a broad range of downstream tasks \citep{brown2020language,raffel2020exploring}. 

Despite its empirical success, why token prediction learns broadly useful representations remains incompletely understood. What information is learned through token prediction, and why should that information support tasks beyond the prediction objective itself? Recent studies have begun to answer parts of these questions from complementary perspectives \citep{lee2021predicting,wei2021pretrained,liu2022masked,liu2023same,trauger2025nexttokenpredictionllmsend}. For example, \cite{wu2023connecting} study the transfer from pre-training to binary sequence classification under a log-linear prediction model. \cite{zhao2024implicit} investigate how next-token prediction induces geometry in the tokens and context embedding spaces under an unconstrained-feature model. \cite{saunshi2021a} show that next-word prediction can benefit text classification when the classification task can be reformulated as a sentence-completion problem, and they quantify the resulting downstream classification error. More recently, \cite{chen2026the} explain how next-token prediction can facilitate post-training and test-time selection through the lens of response-level coverage. These works provide valuable explanations of particular links between token prediction and its downstream benefits. However, they do not offer a common framework that simultaneously characterizes the information identified by token prediction, its encoding in learned representations, and its consequences for different downstream tasks. This motivates a unified statistical framework for understanding how token prediction produces useful contextual representations.

In this paper, we develop a statistical framework that provides a unified answer to these questions. To introduce the framework, let $X_{j,c}$ denote the observed context of a target token $x_j$ and write its conditional probability mass function as
$$
p_j(v|X)=\PP(x_j=v|X_{j,c}),\qquad v\in\Vcal,
$$
where $\Vcal$ is the token vocabulary. Depending on the prediction objective, the context $X_{j,c}$ may consist of the tokens preceding $x_j$, as in autoregressive prediction, or the observed tokens surrounding a masked position, as in masked-token prediction. The population target of token prediction is $p_j(v|X)$, and the training objective seeks to estimate this conditional distribution. Our analysis is primarily conducted at the population level. This perspective isolates the structural consequences of the token-prediction objective from finite-sample and optimization errors and is particularly relevant to modern pre-training, where the number of observed tokens can be enormous. It should nevertheless be viewed as an idealization rather than an assertion that finite-sample effects always vanish. The main idea can be seen from the familiar softmax head
$$
\hat{p}_j(v|X)\propto\exp\left(E_v^T\Theta_j(X)\right),
$$ 
where $E_v\in \RR^s$ is the embedding vector of token $v$ and $\Theta_j(X)\in \RR^s$ is the contextual representation produced by the encoder. In modern language models, the representation map $\Theta$ is typically implemented through a stack of Transformer layers, and both the token embeddings and contextual representations can be used in downstream analysis. This leads to a central question: what information does token prediction encode in the token embeddings $E_v$ and contextual representations $\Theta_j(X)$?

Our first set of results characterizes the information encoded in these two representations. The analysis shows that the token embeddings can capture a contextual notion of semantic geometry motivated by the distributional hypothesis \citep{lenci2023distributional}. More specifically, let $p_{v}(X_{j,c}|v)=\PP(X_{j,c}|x_j=v)$ denote the conditional distribution of the context given the target token $x_j=v$, and define the squared Hellinger distance between conditional distributions of the context
$$
\Gamma_h(v_1,v_2)=H^2(p_{v_1}(X_{j,c}|v_1),p_{v_2}(X_{j,c}|v_2)).
$$
Thus, $\Gamma_h(v_1,v_2)$ is small when $v_1$ and $v_2$ occur in similar contexts. Let $\epsilon_p$ denote the population token prediction error and let $D_E(v_1,v_2)=\|E_{v_1}-E_{v_2}\|^2$ be the squared Euclidean distance between their embeddings. Under the regularity conditions specified later, we establish the two-sided comparison
$$
\sqrt{1-e^{-c_{\Theta}D_E(v_1,v_2)}}-\sqrt{\epsilon_p\over p_m(v_1,v_2)}\le \sqrt{\Gamma_h(v_1,v_2)}\le \sqrt{1-e^{-C_{\Theta}D_E(v_1,v_2)}}+\sqrt{\epsilon_p\over p_m(v_1,v_2)},
$$
where $p_m(v_1,v_2)=\min\{p(v_1),p(v_2)\}$ is the minimum of marginal probabilities of $v_1$ and $v_2$, and $c_\Theta,C_\Theta>0$ depend on regularity properties of the token prediction model. Therefore, when token prediction is accurate and the tokens are sufficiently frequent, their embedding distances recover the geometry of their conditional contextual distributions. At the context level, $\Theta_j(X)$ serves as a latent coordinate whose image under the prediction head approximates $p_j(X)$ on the manifold determined by the token embeddings. How can this contextual representation be approximated by a tractable encoder?

Our second set of results addresses this approximation problem. The representation map $\Theta$ is defined on the discrete and combinatorial space of token sequences, which has no intrinsic geometry through which information can be shared across similar contexts. A natural strategy is to first map tokens into a Euclidean representation space and then exploit smoothness in that space to construct the representation of a target token from the representations of its context. In other words, we use a representation to construct itself, motivating a compositional architecture built from repeated representation blocks, as in recurrent-depth Transformers and equilibrium models \citep{dehghani2019universal,bai2019deep}. We formalize this idea through a self-consistency principle under which a target representation satisfies
$$
\Theta_j(X)=g_\Theta\left(\Theta_{j,c}(X)\right),
$$
where $\Theta_{j,c}(X)$ denotes the collection of representations associated with the observed context $X_{j,c}$ and $g_\Theta$ is a smooth function. The token geometry identified above provides a natural starting point for representing contexts in Euclidean space, while the self-consistency principle requires the resulting contextual geometry to support a smooth map $g_\Theta$. This relation characterizes $\Theta$ as a fixed point and motivates approximating it through repeated applications of a simpler representation block. Under contraction and block-approximation conditions, we show that the error of a depth-$L$ composite encoder is bounded by $\kappa_\Theta^L$ times its initialization error plus an approximation floor of order $\delta_\Theta/(1-\kappa_\Theta)$, where $\kappa_\Theta$ is the contraction factor and $\delta_\Theta$ measures how well the self-consistency operator can be approximated by an individual block. Because the same representation block is reused at every step, increasing $L$ reduces the initialization component of the approximation error without increasing the number of trainable parameters. The resulting recurrent-depth architecture therefore trades additional computation for improved representation accuracy until the block-approximation floor is reached. Because the reused block may itself be highly expressive, it can be applied for additional refinement steps without introducing additional parameters. Token prediction alone generally identifies an equivalence class of representation pairs rather than a unique coordinate system, and predictively equivalent representations need not be equally compatible with repeated-block approximation. Restricting the encoder to the recurrent-depth class therefore introduces an approximation-theoretic implicit bias toward representations with well-conditioned embedding geometry and a self-consistency operator that is contractive and accurately approximated by the shared block. An accompanying oracle inequality translates the representation approximation error into the additional population prediction error incurred by this restriction.

\begin{figure}[h!]
	\centering
	\includegraphics[width=0.9\linewidth]{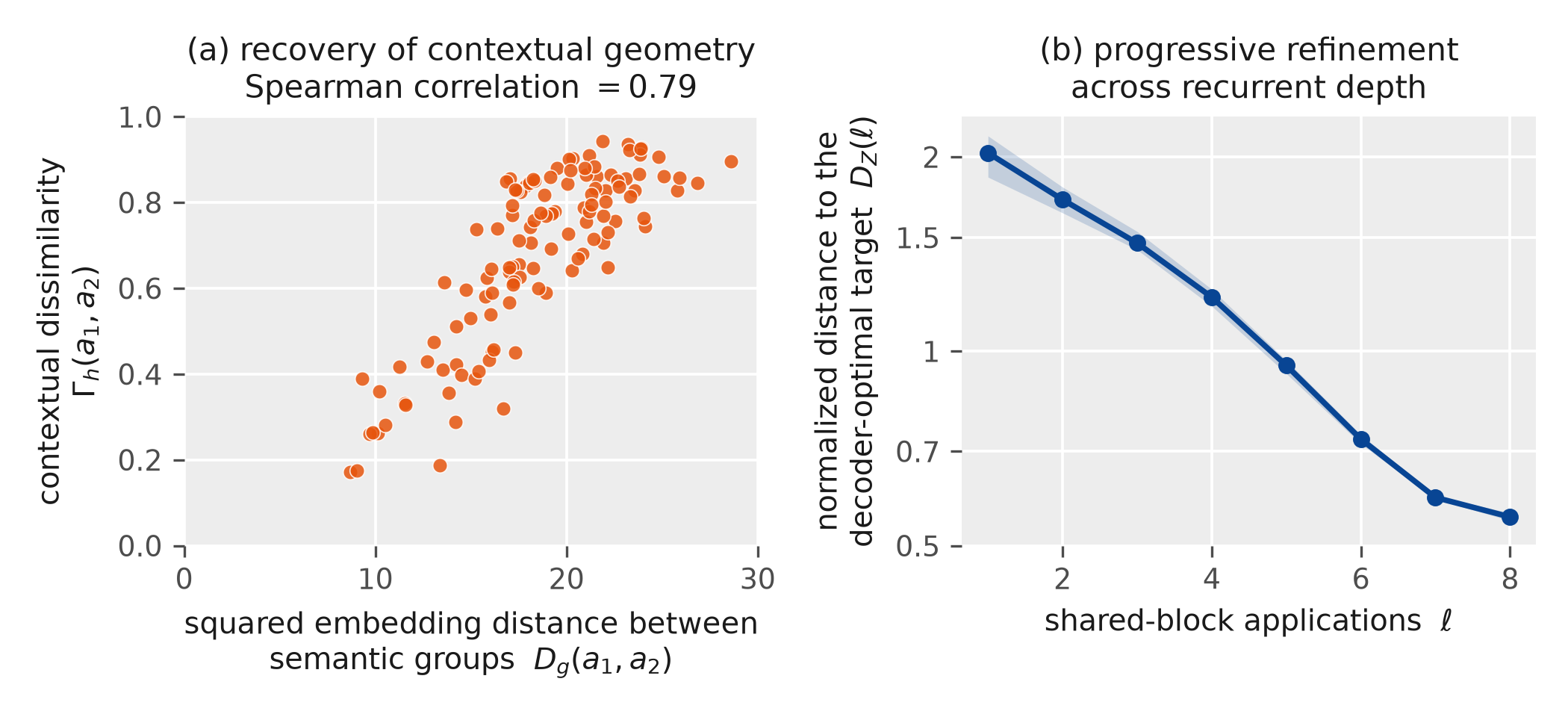}
	\caption{Simulation-based preview of geometry recovery and recurrent refinement. (a) The group embedding distance $D_g(a_1,a_2)$ is plotted against the contextual dissimilarity $\Gamma_h(a_1,a_2)$, with one point for each of the $120$ pairs of $16$ semantic groups. (b) The normalized distance $D_Z(l)$ between the intermediate representation $\Psi_j^l(X)$ and the decoder-optimal representation $Z_j(X)$ across repeated applications of a shared representation block. Precise definitions and simulation details are provided in Section~\ref{sc:numerical}.}
	\label{fg:intro}
\end{figure}

Figure~\ref{fg:intro} provides a simulation-based illustration of the preceding geometry and approximation results. Although the model is trained only through next-token prediction and never observes the semantic groups or their contextual geometry, panel~(a) shows that the group embedding distance $D_g(a_1,a_2)$ has a Spearman correlation of $0.79$ with the contextual dissimilarity $\Gamma_h(a_1,a_2)$. Panel~(b) shows that the normalized distance $D_Z(l)$ decreases as the same $4112$-parameter representation block is repeatedly applied. Thus, recurrent depth progressively moves the intermediate representation $\Psi_j^l(X)$ toward the decoder-optimal representation $Z_j(X)$, increasing computation without increasing the number of trainable parameters.

Having characterized what token prediction encodes and how the resulting contextual representation can be approximated, we next ask whether these learned representations provide benefits beyond token prediction itself. Our final set of results answers this question through three downstream problems: token generation, token community recovery, and token classification. For generation, we show that a small population prediction error forces the pre-trained generator to assign high probability to the set of plausible tokens. When the pre-trained generator is subsequently used as the reference distribution in post-training, it restricts the effective choice set from the full vocabulary of size $d$ to a plausible set of size $k$, replacing a complexity term of order $\log d$ with $\log k$ and yielding a smaller reward-regret bound. For community recovery, the two-sided relationship between embedding distance and contextual dissimilarity makes the learned token embeddings sufficiently informative for spectral clustering, and we derive an explicit bound on the resulting misclustering rate. For token classification, we show that the excess risk of a linear probe can be decomposed into a representation error inherited from pre-training and a downstream estimation error arising from the finite labeled sample. Under appropriate compatibility conditions, the representation error measures how well pre-training organizes the large discrete context space into a contextual representation for which class evidence can be approximated by linear scores. The token-embedding geometry and the simple prediction head together make this linearization possible. Taken together, these results show that accurate learning of the conditional token distribution concentrates generation on plausible tokens, organizes token embeddings according to contextual similarity, and produces contextual representations that support linear classification from a finite labeled data set.

\section{Geometry Induced by Token Prediction}
\label{sc:geometry}

\subsection{From Conditional Token Distributions to Semantic Representations}
Suppose we observe a collection of token sequences, where each sequence $X$ consists of $m$ ($m$ may vary across token sequences) tokens $X=[x_{1},\ldots, x_{m}]$, where each token $x_{j}$ is drawn from a finite vocabulary $\Vcal$ with $|\Vcal|=d$. The goal of representation learning is to learn a map $\Theta$ which converts a token sequence to a numerical vector sequence, i.e.,
$$
[z_{1},\ldots, z_{m}]=[\Theta_1(X),\ldots,\Theta_m(X)]=\Theta(x_{1},\ldots, x_{m}),
$$
where each $z_{j}\in \RR^s$ is the contextual representation at target position $j$. 

In token prediction, a natural starting point of a suitable representation is the conditional probability mass function given the context 
$$
p_j(v|X)=\PP\left(x_j=v\middle|X_{j,c}\right),\quad v\in \Vcal\qquad {\rm and}\qquad p_j(X)=\{p_j(v|X)\}_{v\in \Vcal} \in \Delta^{d-1},
$$
where $X_{j,c}$ is the context of the $j$th token, that is, an observed subset of the token sequence $X$ that excludes the target token $x_j$, and $\Delta^{d-1}$ is a $(d-1)$-dimensional simplex. Notably, the context $X_{j,c}$ is quite flexible in this definition, because it may consist of the tokens preceding $x_j$, as in autoregressive prediction, or the observed tokens surrounding a masked position, as in masked-token prediction. The conditional distribution $p_j(X)$ describes the mixture of token types that can plausibly occupy the target position. It is the population target of token prediction and provides the probability weights needed to construct a contextual representation. However, because its support is an unordered vocabulary, it does not by itself encode semantic relationships among token types.

Although $p_j(X)$ is a straightforward way to summarize plausible token types in the given context, we still need an appropriate way to measure the similarity between such conditional probability mass functions. Commonly used divergences or distances defined on the simplex, such as KL divergence and Hellinger distance, usually treat the vocabulary as an unordered discrete support and therefore do not account for similarities between token types. To incorporate such similarities, we introduce an embedding for each token type and regard the conditional probability mass function as a probability measure on the embedding space. Specifically, let $E_v\in \RR^s$ be the embedding vector of token type $v$. Given these embeddings, the conditional probability mass function $p_j(X)$ can be mapped to the following discrete probability measure on the embedding space
$$
p_j(X)\quad \to \quad p_j(X;E)=\sum_{v\in \Vcal}p_j(v|X)\delta_{E_v},
$$
where $\delta_{E_v}$ denotes the point mass at $E_v$. Equivalently, $p_j(X;E)$ is the pushforward of $p_j(X)$ under the embedding map $v \to E_v$. This construction incorporates token similarity into contextual comparison. If two conditional distributions differ mainly by reallocating probability among nearby token embeddings, then their Wasserstein distance is small, corresponding to equivalence collapse. In contrast, reallocating probability between distant embeddings produces a larger distance, preserving semantic separation.

The construction $p_j(X;E)$ shows how a conditional token distribution can become a semantic contextual representation once the vocabulary is equipped with an appropriate geometry. It also raises two related questions. How can the probability measure $p_j(X;E)$ be represented by a tractable Euclidean vector, and can the required token geometry $E$ itself be learned solely through token prediction? We address the first question through the recovery-based prediction head and the second through the geometry-recovery result below.

\subsection{Representation Learning through Token Prediction}
While $p_j(X;E)$ provides a quantitative representation that incorporates the geometry among token types encoded by $E$, directly evaluating probability metrics on the embedding space may be computationally inconvenient, particularly when representations must be compared repeatedly. Moreover, $p_j(X;E)$ is a probability measure rather than a fixed-dimensional numerical vector. These issues motivate the construction of a concise representation of $p_j(X;E)$. We consider a general strategy for constructing such a representation: recovery-based representations. A recovery-based representation seeks a vector from which the original conditional probability mass function can be approximately reconstructed. Specifically, given a prespecified nonnegative similarity kernel $\Kcal$, we seek $\Theta_j(X)\in\RR^s$ such that
$$
p_j(v|X)\approx \hat{p}_E(v|\Theta_j(X))={\Kcal(\Theta_j(X),E_v) \over \sum_{v'\in\Vcal}\Kcal(\Theta_j(X),E_{v'})},\qquad v\in\Vcal.
$$
The kernel-based recovery generalizes standard coordinate recovery. In standard coordinate recovery, each coordinate of a vector is obtained through its inner product with a canonical basis vector. Kernel-based recovery replaces the fixed canonical basis with learned token embeddings and replaces the inner product with a kernel. Thus, each coordinate is recovered by comparing the contextual representation $\Theta_j(X)$ with the token embedding $E_v$, followed by normalization over the vocabulary. Therefore, $\Theta_j(X)$ can be interpreted as a latent coordinate of the conditional distribution $p_j(X)$ within the decoder family determined by $E$. The recovery head incorporates the token embeddings into the reconstruction of the conditional distribution. In this sense, $\Theta_j(X)$, together with $E$ and the recovery head, provides a concise surrogate for $p_j(X;E)$.

Unlike many existing representation learning methods that learn an efficient map of observed objects, the desired representation described above aims to learn a map of an unobserved conditional probability mass function with unknown token embeddings. Because of the connection between the desired representation and the conditional probability mass function, representation learning can be carried out by predicting the target token from its context. Specifically, we consider the population token-prediction problem
$$
\min_{\Theta, E} \EE_{X,j}\left[L\left(p_j(X);\hat{p}_E(\Theta_j(X))\right)\right],
$$
where $L(p;q)=\sum_{v\in \Vcal}p_v\log(p_v/q_v)$ is the Kullback–Leibler divergence and $\EE_{X,j}$ is the expectation taken with respect to token sequences and target positions. This population loss is equivalent, up to a constant independent of $E$ and $\Theta$, to the expected negative log-likelihood used in token-prediction training.

Similar in spirit to singular value decomposition, the conditional token distribution recovery described above produces two related representations. In singular value decomposition, the row and column representations interact with each other to recover each entry of a matrix. Here, the contextual representation $\Theta_j(X)$ and the token representation $E_v$ interact through the recovery score
$$
\Kcal(\Theta_j(X),E_v).
$$
For fixed token embeddings, $\Theta_j(X)$ represents a context through its compatibility with all token types. Conversely, across different contexts, $E_v$ represents token type $v$ through its compatibility with all contextual representations. Normalizing these interaction scores over the vocabulary approximately recovers the conditional distribution $p_j(X)$. Thus, $\Theta_j(X)$ and $E_v$ can be viewed as the context and token factors of the same conditional-distribution recovery problem.

To focus on the geometry of the token embeddings, we first allow the representation map $\Theta$ to be sufficiently flexible to select any $Z\in\RR^s$ for each context. Given the token embeddings $E$, $\hat{p}_E(Z)$ defines a parametric map from the representation space $\RR^s$ to the simplex $\Delta^{d-1}$. Therefore,
$$
\Mcal_E=\{\hat{p}_E(Z):Z\in\RR^s\}\subset \Delta^{d-1}
$$ 
is a low-dimensional model in the simplex. Under suitable regularity conditions, it forms a smooth manifold of dimension at most $s$. Clearly, $\Mcal_E$ depends on the token embeddings $E$. For $p\in\Delta^{d-1}$, define its discrepancy from $\Mcal_E$ as $L(p;\Mcal_E)=\inf_{Z\in\RR^s} L(p;\hat{p}_E(Z))$.
The population token-prediction problem then reduces to
\begin{equation}
	\label{eq:mnffit}
	\min_{E}\EE_{X,j}\left[ L(p_j(X);\Mcal_E) \right].
\end{equation}
In other words, token prediction identifies a low-dimensional model in the simplex that approximates the collection of conditional probability mass functions $p_j(X)$. Let $E^\ast$ denote a minimizer of \eqref{eq:mnffit}. Although $E^\ast$ is selected only through prediction accuracy, it is subsequently used as a representation of token types. A natural question is therefore whether the Euclidean geometry of $E^\ast$ reflects semantic similarities among token types. We answer this question in the next subsection.

\subsection{Recovery of Semantic Geometry}
Although token prediction directly estimates the conditional distribution of a target token given its context, we show that it can also recover the conditional distribution of the context given a token type. The conditional contextual distribution is particularly relevant to semantic representation because, according to the distributional hypothesis, token types occurring in similar contexts tend to have similar meanings \citep{lenci2023distributional}. We now make this connection precise. For a given token type $v$, Bayes' rule gives the population conditional contextual distribution
$$
p_v(X_{j,c}|v)={p_j(v|X)\over p(v)}p(X_{j,c}),
$$
where $p(v)=\EE_{X,j}(p_j(v|X))$ is the marginal probability that describes how frequently token type $v$ occurs and $p(X_{j,c})$ is the probability mass function of context $X_{j,c}$. Replacing the population conditional token distribution by its fitted counterpart gives the estimated conditional contextual distribution
$$
\hat{p}_{E,v}(X_{j,c}|v)={\hat{p}_E(v|\Theta_j(X))\over \hat{p}(v)}p(X_{j,c}),
$$
where $\hat{p}(v)=\EE_{X,j}(\hat{p}_E(v|\Theta_j(X)))$ is the marginal probability induced by the fitted conditional token distributions. The connection between the two conditional-distribution problems can be seen directly from the token-prediction loss. The chain rule for Kullback-Leibler divergence gives
$$
\EE_{X,j}\left[L\left(p_j(X);\hat{p}_E(\Theta_j(X))\right)\right]=\sum_{v\in\Vcal}p(v)\log{p(v)\over \hat{p}(v)}+\sum_{v\in\Vcal}p(v)L(p_v(X_{j,c}|v);\hat{p}_{E,v}(X_{j,c}|v)).
$$
The first term measures how well the fitted model recovers the marginal frequencies of the token types. The second term measures how well it recovers their conditional contextual distributions, averaged according to the token frequencies. Since both terms are nonnegative, accurate token prediction leads to accurate estimation of the conditional contextual distribution for every sufficiently frequent token type.

We next examine how the token embedding enters the model-implied conditional contextual distribution. Consider the softmax kernel $\Kcal(\Theta_j(X),E_v)=\exp\left(E_v^T\Theta_j(X)\right)$. The model-implied conditional contextual distribution can be written as
$$
\hat{p}_{E,v}(X_{j,c}|v)=e^{E_v^T\Theta_j(X)-A(E_v)}\tilde{p}(X_{j,c}),
$$
where $A(e)=\log\sum_{X,j}\exp(e^T\Theta_j(X))\tilde{p}(X_{j,c})$ is the normalizing constant that ensures the distribution over contexts sums to one and $\tilde{p}(X_{j,c})=p(X_{j,c})/ \sum_{v\in\Vcal}e^{E_v^T\Theta_j(X)}$ is a common base measure. The base measure $\tilde{p}(X_{j,c})$ is common to all token types, while the token embedding $E_v$ varies with $v$. Therefore, the model-implied conditional contextual distributions form an exponential family whose natural parameter is $E_v$ \citep{efron2023exponential}. In the exponential family, the Hellinger distance can connect the embedding geometry with the geometry in the model-implied conditional contextual distributions. Specifically, a direct calculation gives
$$
H^2(\hat{p}_{E,v_1}(X_{j,c}|v_1),\hat{p}_{E,v_2}(X_{j,c}|v_2))=1-\exp\left(-J_A(E_{v_1},E_{v_2})\right),
$$
where $H$ is the Hellinger distance and $J_A(E_{v_1},E_{v_2})$ is the Jensen gap of $A$
$$
J_A(E_{v_1},E_{v_2})={A(E_{v_1})+A(E_{v_2})\over 2}-A\left({E_{v_1}+E_{v_2}\over 2}\right).
$$
In this equality, $1-H^2$ measures the overlap between the two model-implied conditional contextual distributions, while $J_A$ measures the lack of overlap on a logarithmic scale through the curvature of $A$. This identity is why Hellinger distance is particularly convenient for the exponential-family representation. If the curvature of $A$ is bounded above and below along the relevant embedding directions, its Jensen gap is comparable to the squared Euclidean distance between the embeddings, that is, there exist constants $0<c_\Theta\le C_\Theta$ such that
$$
c_\Theta\|E_{v_1}-E_{v_2}\|^2\le J_A(E_{v_1},E_{v_2})\le C_\Theta\|E_{v_1}-E_{v_2}\|^2.
$$
Consequently, the Euclidean distance between token embeddings characterizes the Hellinger geometry of the model-implied conditional contextual distributions, up to a monotone exponential transformation.

The preceding argument concerns the model-implied conditional contextual distributions, which have the exponential-family structure induced by the recovery head. The population conditional contextual distributions need not belong to this family. However, the token-prediction loss shows that the model-implied and population distributions are close when the prediction error is small and the token types are sufficiently frequent. Since Kullback–Leibler divergence controls Hellinger distance and $H$ satisfies the triangle inequality, the prediction-loss bound allows the embedding geometry of the model-implied conditional contextual distributions to be transferred to their population counterparts. We therefore define the population contextual dissimilarity between token types $v_1$ and $v_2$ as
$$
\Gamma_h(v_1,v_2)=H^2(p_{v_1}(X_{j,c}|v_1),p_{v_2}(X_{j,c}|v_2)).
$$
The quantity $\Gamma_h(v_1,v_2)$ is a natural measure of semantic dissimilarity because it compares the distributions of contexts in which the two token types occur. A small value means that $v_1$ and $v_2$ tend to occur in similar contexts, whereas a large value means that their contextual distributions have little overlap. The following theorem formalizes this relationship and quantifies the effects of token-prediction error and token frequency in a more general setting. To characterize the behavior of embedding vectors, we consider the following assumptions:
\begin{assumption}
	\label{asp:embedding}
	We make the following assumptions:
	\begin{enumerate}
		\item There exist $E^\ast$ and $\Theta^\ast$ such that 
		$$
		\EE_{X,j}\left[L\left(p_j(X);\hat{p}_{E^\ast}(\Theta^\ast_j(X))\right)\right]\le \epsilon_p.
		$$
		\item The kernel $\Kcal$ has the form $\Kcal(\Theta_j(X),E_v)=\exp\left(E_v^T\Theta_j(X)\right)$.
		\item There exist constants $0<b_\Theta\le B_\Theta<\infty$ such that, for any $e\in {\rm conv}\{E_v^\ast:v\in\Vcal\}$ and $\tilde{e}\in {\rm span}\{E^\ast_{v_1}-E^\ast_{v_2}: v_1,v_2\in \Vcal\}$, we have
		$$
		b^2_\Theta \|\tilde{e}\|^2\le \tilde{e}^T\nabla^2 A(e)\tilde{e}\le B^2_\Theta \|\tilde{e}\|^2.
		$$
		Here $A$ is defined by $E^\ast$ and $\Theta^\ast$.
	\end{enumerate}
\end{assumption}

The first assumption suggests that the token prediction loss is small enough. The second assumption gives a specific form of the kernel used in practice. The last assumption indicates that the curvature of $A$ is bounded above and below. The upper bound can be easily satisfied if $\Theta_j(X)$ is bounded, whereas the lower bound is the important nondegeneracy condition. With these assumptions, the following theorem establishes a connection between Euclidean distances among token embeddings and semantic dissimilarities among token types. 

\begin{theorem}
	\label{thm:emdsimilarity}
	We assume that Assumption~\ref{asp:embedding} holds and $p(v_1),p(v_2)>0$. 
	Then, we have
	$$
	\sqrt{\Gamma_h(v_1,v_2)}\le \sqrt{1-\exp\left(-{B_\Theta^2\over 8} \left\|E^\ast_{v_2}-E^\ast_{v_1}\right\|^2\right)}+\sqrt{\epsilon_p\over \min\{p(v_1),p(v_2)\}},
	$$
	and 
	$$
	\sqrt{\Gamma_h(v_1,v_2)}\ge \left(\sqrt{1-\exp\left(-{b_\Theta^2\over 8} \left\|E^\ast_{v_2}-E^\ast_{v_1}\right\|^2\right)}-\sqrt{\epsilon_p\over \min\{p(v_1),p(v_2)\}}\right)_+.
	$$
\end{theorem}

Theorem~\ref{thm:emdsimilarity} suggests that the discrepancies between token types can be approximately upper and lower bounded by a monotone transformation of token embeddings' Euclidean distance. These bounds are sharper when the conditional probability mass function can be better estimated by the model, i.e., $\epsilon_p$ is smaller. It is worth pointing out that the conclusion from Theorem~\ref{thm:emdsimilarity} is reliable when both token types are moderately frequent, i.e., $p(v)$ is not too small. If one of the token types is rare, the bounds for  $\Gamma_h(v_1,v_2)$ are not accurate, so the Euclidean distance between embedding vectors need not accurately reflect the similarity of rare token types. This is expected as rare token types do not provide enough information for us to compare them.

\subsection{Consequences for Contextual Representation Learning}
Theorem~\ref{thm:emdsimilarity} shows that token prediction produces two desirable geometric properties commonly associated with useful representations. First, it produces equivalence collapse \citep{wang2025augmentation}. Two token types $v_1$ and $v_2$ are distributionally equivalent if they occur in the same contextual distribution, that is,
$$
p_{v_1}(X_{j,c}|v_1)=p_{v_2}(X_{j,c}|v_2),
$$
or equivalently, $\Gamma_h(v_1,v_2)=0$. When the population prediction error is zero and the curvature condition in Assumption~\ref{asp:embedding} holds, Theorem~\ref{thm:emdsimilarity} implies
$$
\Gamma_h(v_1,v_2)=0\quad\Longleftrightarrow\quad E_{v_1}^\ast=E_{v_2}^\ast.
$$
Thus, token types that are indistinguishable by their conditional contextual distributions receive the same embedding. When the prediction error is nonzero, the same conclusion holds approximately, with the degree of collapse controlled by the prediction error and the marginal frequencies of the two token types.

Second, token prediction preserves meaningful contextual distinctions. If two token types occur in substantially different contextual distributions, then $\Gamma_h(v_1,v_2)$ is large. Provided that the prediction error is sufficiently small, Theorem~\ref{thm:emdsimilarity} requires their embedding vectors to remain separated. The embedding geometry therefore removes distinctions that cannot be identified from context while preserving distinctions supported by contextual evidence. These two properties correspond to the principles of equivalence collapse and semantic geometry preservation that commonly motivate representation learning \citep{bengio2013representation}.

These conclusions concern the token embeddings $E_v$. They nevertheless have an important consequence for learning the contextual representation $\Theta_j(X)$. Once token types are placed in a Euclidean space that reflects their contextual similarity, the conditional measure
$$
p_j(X;E)=\sum_{v\in\Vcal}p_j(v\mid X)\delta_{E_v}
$$
becomes a semantic representation of the possible target token. Together with the token embeddings \(E\) and the recovery head, $\Theta_j(X)$ provides a concise coordinate from which this measure can be recovered. The recovered token geometry therefore supplies the Euclidean structure needed to construct contextual representations from the representations of the observed context.

Geometry recovery alone, however, does not guarantee that the contextual representation map is stable under small changes in its context or that it can be efficiently estimated from observed sequences. Addressing these questions requires additional structure on $\Theta$. In the next section, we introduce a self-consistency principle that uses the recovered token geometry to approximate $\Theta$ through repeated applications of a simpler representation block.

\section{Recurrent-Depth Approximation of Contextual Representations}
\label{sc:approximation}

In Section~\ref{sc:geometry}, the contextual representation map was allowed to be sufficiently flexible so that we could isolate the geometry induced in the token embeddings. We now reverse the focus: fixing a population representation pair $(E^\ast,\Theta^\ast)$, we study whether $\Theta^\ast$ can be approximated by a structured encoder. The central idea is to select a self-consistent representation and approximate its defining operator through repeated applications of a shared representation block.

\subsection{Nonidentifiability and Self-Consistency}
Learning the representation map $\Theta$ is fundamentally different from learning a smooth function on a Euclidean domain. Euclidean space comes with a built-in geometry: nearby points are naturally defined, and one can borrow strength across neighboring observations by exploiting smoothness. In contrast, the space of token sequences is discrete and combinatorial. Without an additional metric or learned geometry, different sequences are simply separated objects, so observations from one sequence do not directly inform the value of $\Theta$ at another. This makes direct estimation of a function on the space of token sequences statistically challenging.

In addition, the solution of the token-prediction problem is often not unique. We still consider the kernel $\Kcal(\Theta_j(X),E_v)=\exp\left(E_v^T\Theta_j(X)\right)$. If $(E^\ast,\Theta^\ast)$ is a minimizer, then for any invertible matrix $A$,
$$
\Theta'_j(X)=A\Theta^\ast_j(X)\qquad {\rm and}\qquad E_v'=(A^{-1})^TE_v^\ast 
$$
gives the same value of each kernel score and hence the same token prediction loss. This example suggests that the token prediction objective alone cannot distinguish between $\Theta^\ast$ and $A\Theta^\ast$. Thus, the token prediction objective identifies an equivalence class of dual representations rather than a unique coordinate system. Although representations within this equivalence class are indistinguishable under the token-prediction loss, their Euclidean geometries and the corresponding curvature constants in Theorem~\ref{thm:emdsimilarity} may differ. The theorem remains valid for each representative, but the informativeness of its bounds depends on the conditioning of the selected coordinates. Which representative admits a smooth and stable construction from its contextual representations?

In this section, we show that these two challenges can be mitigated by introducing an additional criterion for favoring useful coordinates. A natural way to mitigate the first challenge is to first map tokens into a preliminary Euclidean representation space and then exploit smoothness in that space to construct a refined contextual representation. The refined representation can itself serve as the input to another representation update. Repeating this procedure produces a sequence of representations, each constructed from the preceding one. Under an appropriate contraction condition, these representations move progressively toward a self-consistent population representation. In this sense, we use a preliminary representation to construct a better representation. This leads to a contextual self-consistency principle. Following the definition of $X_{j,c}$, let $\Theta_{j,c}(X)$ denote the ordered collection of representations at the positions included in $X_{j,c}$, retaining their positional information. Although $\Theta_j(X)$ is defined on a discrete sequence space, the context representation $\Theta_{j,c}(X)$ takes values in a Euclidean space. Because $\Theta_{j,c}(X)$ lies in a Euclidean product space, smoothness can now be meaningfully imposed. We say that a representation map $\Theta$ is self-consistent if there exists a smooth function $g_\Theta$ such that
$$
\Theta_j(X)=g_\Theta\left(\Theta_{j,c}(X)\right).
$$
We call a representation map satisfying the above principle a self-consistent representation and call $g_\Theta$ the self-consistency function. The self-consistency function $g_\Theta$ is indexed by $\Theta$ because different candidate representations induce different Euclidean geometries for the contextual representations, and hence require different smooth functions. This assumption does not treat $\Theta$ as a smooth function on the original discrete space. Rather, it asserts that the representation is self-consistent: once a candidate representation induces a Euclidean geometry for token contexts, the representation at each location can be recovered smoothly from the geometry of its context. Learning $\Theta$ is therefore a fixed-point problem, because the geometry used to estimate $\Theta_j(X)$ is itself determined by $\Theta$. 
 
The fixed-point problem induced by the self-consistency principle also suggests a general way to construct approximations to the representation map. For any smooth function $g$, we define an operator $\Tcal_g$ on the space of candidate representation maps by
$$
[\Tcal_g(\Psi)]_j(X)=g\left(\Psi_{j,c}(X)\right),
$$
where $\Psi_{j,c}(X)$ denotes the contextual representation induced by the candidate map $\Psi$. The operator $\Tcal_g$ is indexed by $g$ because different smooth recovery functions induce different operators on the space of representation maps. Under the self-consistency principle, the self-consistent representation $\Theta$ is a fixed point of the operator associated with $g_\Theta$: $\Theta=\Tcal_{g_\Theta}(\Theta)$. Suppose that $\Tcal_{g_\Theta}$ is a contraction under a suitable metric $D$ on the space of representation maps, i.e., $D(\Tcal_{g_\Theta}(\Psi),\Tcal_{g_\Theta}(\Phi))\le \kappa_\Theta D(\Psi,\Phi)$ for some constant $0<\kappa_\Theta<1$. Then, starting from an initial representation map $\Psi^0$, we can iteratively define
$$
\Psi^{l+1}=\Tcal_{g_\Theta}(\Psi^l),\qquad l=0,1,\ldots,L-1.
$$
The contraction property implies
$$
D(\Psi^{L},\Theta)\le \kappa_\Theta^L D(\Psi^0,\Theta).
$$
Thus, repeated applications of $\Tcal_{g_\Theta}$ move the candidate representation toward the self-consistent representation $\Theta$. This geometric convergence suggests approximating the unknown operator $\Tcal_{g_\Theta}$ by a simpler representation block and repeatedly applying that block. Repeated parameter sharing is familiar from recurrent-depth and equilibrium architectures \citep{dehghani2019universal,bai2019deep}. Our purpose here is not to introduce this architecture, but to explain its role in approximating population representations induced by token prediction.

The contraction factor $\kappa_\Theta$ also suggests that predictively equivalent representations may have different approximation difficulty. Consequently, some representatives may be more suitable for structured approximation even though they are indistinguishable under the token-prediction loss. If $\Theta^\ast_j(X)=g_{\Theta^\ast}\left(\Theta^\ast_{j,c}(X)\right)$, then the transformed representation $\tilde\Theta=A\Theta^\ast$ satisfies $\tilde\Theta_j(X)=A g_{\Theta^\ast}\left(A^{-1}\tilde\Theta_{j,c}(X)\right)$. Therefore, the corresponding self-consistency function is $g_{\tilde\Theta}(y)=A g_{\Theta^\ast}(A^{-1}y)$. Even if $g_{\Theta^\ast}$ is well behaved, $g_{\tilde\Theta}$ may not be. In particular, the Lipschitz constant of $g_{\tilde\Theta}$ can be inflated by a factor depending on $\|A\|\|A^{-1}\|$. Thus, under the fixed Euclidean metric used for approximation, an ill-conditioned transformation may preserve the token-prediction loss while weakening the smoothness and contraction properties needed for efficient recurrent approximation. This self-consistency criterion need not distinguish coordinate transformations that preserve the relevant Euclidean geometry, such as orthogonal transformations, while its purpose is to favor well-conditioned representatives rather than to identify a unique orientation.

\subsection{Approximation by Repeated Representation Blocks}

Motivated by the above observations, we can utilize a composition of operators to construct a family of representation maps. Given a collection of initial representation maps $\Rcal^0$ and a collection of smooth functions $\Gcal$, we can consider all representation maps of depth $L$:
$$
\Rcal^L(\Gcal, \Rcal^0)=\left\{\underbrace{\Tcal_{g}(\ldots \Tcal_{g}(\Tcal_{g}}_{L\ {\rm  times}}(\Psi^0))): g\in \Gcal, \Psi^0\in \Rcal^0 \right\}.
$$
Here, $g$ represents an entire representation block and need not correspond to a single Transformer layer. For example, $g$ may itself be implemented by a stack of $h$ Transformer layers. The quantity $L$ denotes the recurrent depth, namely, the number of times that the same block is applied. Consequently, a block containing $h$ Transformer layers and repeated $L$ times has effective computational depth $hL$, while its block parameters are shared across all $L$ applications. We now show that $\Rcal^L(\Gcal, \Rcal^0)$ is quite a rich family of representation maps and can approximate a self-consistent representation map. Let $\Rcal^\ast$ denote a collection of representation maps that contains the target $\Theta$, the initial class $\Rcal^0$, and all intermediate iterates generated by the operators under consideration. To show the richness of representation maps, we consider the following assumptions:
\begin{assumption}
	\label{asp:rich}
	We make the following assumptions:
	\begin{enumerate}
		\item The target representation $\Theta$ is a self-consistent representation map, so there exists a function $g_\Theta$ such that
		$$
		\Theta=\Tcal_{g_\Theta}(\Theta).
		$$
		\item The operator $\Tcal_{g_\Theta}$ is a contraction under the metric $D$, that is, there exist $0<\kappa_\Theta<1$ such that
		$$
		D(\Tcal_{g_\Theta}(\Psi),\Tcal_{g_\Theta}(\Phi))\le \kappa_\Theta D(\Psi,\Phi), \qquad \forall \Psi,\Phi\in \Rcal^\ast.
		$$
		\item $\Tcal_{g_\Theta}$ can be well approximated by $\Tcal_{g}$ for some $g\in \Gcal$, that is,
		$$
		\inf_{g\in\Gcal}\sup_{\Psi\in \Rcal^\ast}D(\Tcal_{g_\Theta}(\Psi),\Tcal_{g}(\Psi))<\delta_\Theta.
		$$
	\end{enumerate}
\end{assumption}

The first two assumptions state that the target representation map is self-consistent and that its associated operator is contractive. The contraction assumption can be viewed as a strengthened smoothness condition on the self-consistency function $g_\Theta$: the recovery function must not only be smooth, but must also have sufficiently small sensitivity to perturbations of the context representation. The last assumption requires that the self-consistency function $g_\Theta$ can be well approximated by functions in the candidate class $\Gcal$. Although the contraction and approximation assumptions are stated at the operator level, they can be translated into more concrete conditions on the self-consistency function $g_\Theta$ once a specific discrepancy $D$ is chosen. For example, consider the $L_2$ discrepancy
$$
D(\Psi,\Phi)=\left\{\EE_{X,j}\|\Psi_j(X)-\Phi_j(X)\|^2\right\}^{1/2}.
$$
If we assume that there exists a constant $0<L_g<1$ such that 
$$
\|g_\Theta(y)-g_\Theta(y')\|^2\le {L_g\over m}\sum_{j=1}^m\|y_j-y_j'\|^2\qquad {\rm and}\qquad \inf_{g\in\Gcal}\sup_y\|g_\Theta(y)-g(y)\|\le \delta_\Theta,
$$
where $y=[y_1,\ldots,y_m]$ is a collection of token-wise representation vectors and $y_j\in \RR^s$, then the contraction and approximation assumptions are satisfied. With these assumptions, the following theorem characterizes the extent to which the composite representation class $\Rcal^L(\Gcal, \Rcal^0)$ can approximate a self-consistent representation map.

\begin{theorem}
	\label{thm:aprx}
	We assume that Assumption~\ref{asp:rich} holds. Then, we have
	$$
	\inf_{\Psi\in\Rcal^L(\Gcal, \Rcal^0)}D\left(\Psi,\Theta\right)\le \kappa_\Theta^L\inf_{\Psi^0\in \Rcal^0 }D(\Psi^0,\Theta)+{1-\kappa_\Theta^L \over 1-\kappa_\Theta}\delta_\Theta.
	$$
\end{theorem}

Theorem~\ref{thm:aprx} shows that the approximation error of a self-consistent representation map $\Theta$ by the composite class $\Rcal^L(\Gcal, \Rcal^0)$ is controlled by three quantities: the depth $L$, the quality of the initial representation class $\Rcal^0$, and the approximation error $\delta_\Theta$ of the self-consistency function. Let $D_0=\inf_{\Psi^0\in\Rcal^0}D(\Psi^0,\Theta)$ denote the initial approximation error and $D_\infty=\delta_\Theta/(1-\kappa_\Theta)$ denote the approximation floor. The bound in Theorem~\ref{thm:aprx} separates the transient error inherited from the initial representation from the approximation floor induced by the representation block. When $D_0>D_\infty$, increasing the recurrent depth $L$ reduces the upper bound geometrically toward $D_\infty$. In this regime, repeated block applications progressively refine the initial representation, with diminishing improvement as the approximation floor is approached. When $D_0\le D_\infty$, however, the theorem does not guarantee that additional applications improve the representation, because the error accumulated from approximating the self-consistency operator may dominate. Thus, recurrent depth is most beneficial when the individual block approximates the self-consistency operator sufficiently well relative to the quality of the initialization.

There are two distinct ways to improve this approximation bound. First, one can increase the internal expressive capacity of the representation block $g$, for example by implementing $g$ with a deeper stack of Transformer layers. A richer block class $\Gcal$ may reduce $\delta_\Theta$ and therefore lower the approximation floor, but generally requires additional trainable parameters. Second, one can increase the recurrent depth $L$ while keeping $g$ fixed. This reduces the transient term without introducing new block parameters, although it requires additional computation. More specifically, if $g$ contains $h$ Transformer layers and has $P_g$ parameters, applying it $L$ times produces effective computational depth $hL$ while retaining only $P_g$ distinct block parameters. An untied stack with the same effective depth would generally require approximately $LP_g$ block parameters. Recurrent depth therefore provides parameter-efficient approximation power by extracting additional refinement from a fixed representation block.

Theorem~\ref{thm:aprx} directly concerns recurrent-depth encoders in which the same representation block is reused. Standard Transformer encoders use an untied composition $\Tcal_{g_L}\circ\cdots\circ\Tcal_{g_1}(\Psi^0)$, where each $g_l$ has its own parameters. If the $l$th block approximates the self-consistency operator with error $\delta_l$, the same argument gives
$$
D(\Psi^L,\Theta)\le \kappa_\Theta^L D(\Psi^0,\Theta)+\sum_{l=1}^L\kappa_\Theta^{L-l}\delta_l.
$$
Untying the blocks enlarges the approximation class and permits different approximation errors across depth, but generally increases the number of trainable parameters. Recurrent and untied depth therefore represent different strategies: recurrent depth obtains additional computation from a fixed parameter set, whereas untied depth obtains additional flexibility by introducing new parameters at each block. In both cases, the composition can be interpreted as a finite-step approximation to the self-consistency equation, provided that the blocks approximate its defining operator \citep{dehghani2019universal,bai2019deep}.

Theorem~\ref{thm:aprx} leaves the initial representation class $\Rcal^0$ unspecified. One natural choice maps each observed token to its learned token embedding together with its positional information. When the input and output embedding matrices are tied \citep{press2017using,inan2017tying}, the initial representation is expressed in the same coordinate system used by the prediction head. Weight tying therefore provides a natural, although not necessarily optimal, initialization for recurrent refinement. The effect of this choice enters Theorem~\ref{thm:aprx} through the initial error $D_0$.

\subsection{From Representation Approximation to Prediction Error}

We have an efficient approximator for the self-consistent representation, but what is the approximation cost of restricting the representation to the class $\Rcal^L(\Gcal, \Rcal^0)$? This step separates the intrinsic predictive approximation error from the additional cost imposed by the chosen representation class. We have the following population oracle inequality.

\begin{theorem}
	\label{thm:oracle}
	Suppose $\Kcal$ has the form $\Kcal(\Theta_j(X),E_v)=\exp\left(E_v^T\Theta_j(X)\right)$ and $\Scal_j(X)$ is the support of $p_j(X)$. Then, 
	$$
	\inf_{\Theta\in \Rcal^L(\Gcal, \Rcal^0), E} \EE_{X,j}\left[L\left(p_j(X);\hat{p}_{E}(\Theta_j(X))\right)\right]\le \inf_{\Theta^\ast,E^\ast}\left[\EE_{X,j}\left[L\left(p_j(X);\hat{p}_{E^\ast}(\Theta^\ast_j(X))\right)\right]+\Acal(E^\ast,\Theta^\ast)\right],
	$$
	where the approximation error $\Acal(E^\ast,\Theta^\ast)$ is defined as
	$$
	\Acal(E^\ast,\Theta^\ast)=\sqrt{\EE\left[\max_{v\in\Scal_j(X),v'\in\Vcal}\left\|E^\ast_v-E^\ast_{v'}\right\|^2\right]\inf_{\tilde{\Theta}\in \Rcal^L(\Gcal, \Rcal^0) }\EE\left[\left\|\tilde{\Theta}_j(X)-\Theta^\ast_j(X)\right\|^2\right]}.
	$$
\end{theorem}

The oracle inequality in Theorem~\ref{thm:oracle} characterizes the cost of restricting the representation to $\Rcal^L(\Gcal, \Rcal^0)$. The right-hand side balances two quantities: the predictive error of an unrestricted model, and the cost of approximating its representation by an element of the restricted representation class. Because the infimum is taken over all possible $(E^\ast,\Theta^\ast)$, the bound is not tied to any particular non-unique minimizer of the population risk.

\begin{corollary}
	\label{cor:depthprediction}
	Suppose there exists a self-consistent and $\eta$-optimal population representation pair \((E^\ast,\Theta^\ast)\) such that
	$$
	\EE_{X,j}\left[L\left(p_j(X);\hat{p}_{E^\ast}(\Theta^\ast_j(X))\right)\right]\le \Lcal_p^\ast +\eta
	$$
	for some $\eta\ge 0$, and $\max_{v\in\mathcal V}\|E_v^\ast\|\le B_E$. Here, $\Lcal_p^\ast$ is the  unrestricted population optimum
	$$
	\Lcal_p^\ast=\inf_{\Theta,E}\EE_{X,j}\left[L\left(p_j(X);\hat{p}_{E}(\Theta_j(X))\right)\right].
	$$
	Suppose further that $\Theta^\ast$ satisfies Assumption~\ref{asp:rich} under the $L_2$ metric with contraction factor $\kappa_{\Theta^\ast}$ and block-approximation error $\delta_{\Theta^\ast}$. Then, 
	$$
		\inf_{\Theta\in \Rcal^L(\Gcal, \Rcal^0), E} \EE_{X,j}\left[L\left(p_j(X);\hat{p}_{E}(\Theta_j(X))\right)\right]\le \Lcal_p^\ast+\eta+2B_E\left\{
		\kappa_{\Theta^\ast}^L\min_{\Psi^0\in \Rcal^0 }D(\Psi^0,\Theta^\ast)+{1-\kappa_{\Theta^\ast}^L \over 1-\kappa_{\Theta^\ast}}\delta_{\Theta^\ast}\right\}.
	$$
\end{corollary}

We omit the proof as it is a direct implication of Theorems~\ref{thm:aprx} and \ref{thm:oracle}. Corollary~\ref{cor:depthprediction} separates three sources of population prediction error. The quantity $\Lcal_p^\ast$ is the intrinsic prediction error of the unrestricted recovery model, $\eta$ measures how closely the selected self-consistent pair approaches that unrestricted optimum, and the final term is the additional cost of restricting the encoder to the recurrent-depth class. Recurrent depth geometrically reduces the component inherited from the initialization, while the block-approximation error determines the limiting floor in the additional prediction-error bound.

The population oracle inequality also clarifies how the recurrent-depth restriction interacts with predictive nonidentifiability. Let
$$
\Mcal^\ast=\left\{(E,\Theta):\EE_{X,j}\left[L\left(p_j(X);\hat{p}_{E}(\Theta_j(X))\right)\right]=\Lcal_p^\ast\right\}
$$
denote the collection of unrestricted population minimizers. Although all pairs in $\Mcal^\ast$ attain the same population token-prediction loss, they do not need to have the same approximation penalty $\Acal(E,\Theta)$. Consequently, restricting the encoder to $\Rcal^L(\Gcal,\Rcal^0)$ introduces a  preference among predictively optimal coordinate systems. This preference depends jointly on the token embedding and the compatibility of the contextual representation with the recurrent-depth class. For a self-consistent representation, Theorem~\ref{thm:aprx} bounds this compatibility through the initialization error, the contraction factor $\kappa_\Theta$, and the block-approximation error $\delta_\Theta$. The recurrent-depth restriction is therefore most compatible with predictively optimal coordinates having controlled embedding geometry, a suitable initialization, and a self-consistency operator that can be accurately approximated by the chosen block.

\section{Downstream Guarantees from Token Prediction}
\label{sc:benefits}

\subsection{Token Generation}
Token prediction directly yields an autoregressive generator. In this subsection, we restrict attention to causal contexts $X_{j,c}=X_{<j}=[x_1,\ldots,x_{j-1}]$. Given a partial sequence, a generative model specifies a conditional distribution $q_{j}(Y_{<j})=\{q_{j}(v|Y_{<j})\}_{v\in \Vcal}\in \Delta^{d-1}$ and generates a new token sequence $Y=[Y_1,\ldots, Y_m]$ recursively according to
$$
Y_j\sim q_{j}(Y_{<j}),\qquad j=1,\ldots,m.
$$
Modern aligned generators are often constructed through two steps: pre-training through next-token prediction and post-training via preference fine-tuning. In the pre-training, the embedding vectors $\hat{E}$ and contextual representation map $\hat{\Theta}$ are learned through autoregressive next-token prediction
$$
(\hat{\Theta},\hat{E})\in \argmin_{\Theta\in \Rcal^L(\Gcal, \Rcal^0), E} \EE_{X,j}\left[L\left(p_j(X);\hat{p}_{E}(\Theta_j(X))\right)\right].
$$
The conditional distribution of the pre-trained generator can be written as
$$
q^{\rm pre}_{j}(v|Y_{<j})={\Kcal(\hat{\Theta}_j(Y_{<j}),\hat{E}_v)\over \sum_{v'\in\Vcal}\Kcal(\hat{\Theta}_j(Y_{<j}),\hat{E}_{v'})},\qquad v\in\Vcal.
$$
The learned representations enter the following analysis through this conditional distribution and its population prediction error. After pre-training, we usually further refine the generative conditional distribution via maximizing a regularized reward function. We study a stylized form of post-training in which, at each context, the generator maximizes an estimated token-level reward while remaining close to a reference distribution through KL regularization. KL-regularized reward maximization relative to a pre-trained reference distribution is widely used in language-model alignment \citep{ouyang2022training} and has been studied theoretically through reverse-KL-regularized contextual-bandit formulations \citep{xiong2024iterative,zhao2025sharp}. Existing analyses primarily investigate the finite-sample efficiency of reward or policy learning and the coverage supplied by the reference policy. Our analysis is complementary: we take the reward-estimation error as given and study how the population prediction accuracy of the pre-trained reference controls generation plausibility and the effective size of the post-training choice set. Specifically, let
$$
q^{\hat{r},\pi,\lambda}_j(Y_{<j})=\argmax_{q_{j}(Y_{<j})\in\Delta^{d-1}}\sum_{v\in\Vcal}q_{j}(v|Y_{<j})\hat{r}(v|Y_{<j})-\lambda L(q_{j}(Y_{<j});\pi_{j}(Y_{<j})),
$$
where $\hat{r}(v|Y_{<j})$ estimates the true reward $r^\ast(v|Y_{<j})$, $\pi_j(Y_{<j})$ is a reference distribution, and $\lambda>0$ is the tuning parameter. When the reference distribution has full support, the solution is
$$
q^{\hat{r},\pi,\lambda}_j(v|Y_{<j})={\pi_{j}(v|Y_{<j})\exp(\hat{r}(v|Y_{<j})/\lambda)\over \sum_{v'\in\Vcal}\pi_{j}(v'|Y_{<j})\exp(\hat{r}(v'|Y_{<j})/\lambda)}.
$$
Thus, post-training reweights the reference distribution according to the estimated reward. We are particularly interested in using the pre-trained generator as the reference: $q^{\rm post}_{j}(Y_{<j})=q^{\hat{r}, q^{\rm pre},\lambda}_j(Y_{<j})$. We denote by $\QQ_q$ the sequence distribution induced by the generator $q$, and write $\QQ^{\rm pre}=\QQ_{q^{\rm pre}}$ and $\QQ^{\rm post}=\QQ_{q^{\rm post}}$. The goal of generation is not only to produce a sequence that is plausible, but also to produce one that follows a desired style or preference induced by the reward function. What benefit does the pre-trained reference provide in this post-training
problem?

To answer the above question, we evaluate generation through two complementary criteria. First, we measure whether the generated tokens remain in the context-specific plausible set. Recall $\Scal_j(X_{<j})=\{v\in\Vcal:p_j(v|X_{<j})>0\}$ denotes the support of the true next-token distribution and is interpreted here as the context-specific set of plausible continuations. The plausibility score is defined as
$$
\rho_{p}(q,p)=\EE_{Y\sim\QQ_q}\left[{1\over m}\sum_{j=1}^m\bI(Y_j\in \Scal_j(Y_{<j}))\right],
$$
where $\bI(\cdot)$ is the indicator function. A larger value of $\rho_{p}$ indicates that the generated sequence is more likely to stay within the set of plausible continuations. Second, we measure whether the generator selects preferred representatives. Given a generator $q$, we define its reward regret by 
$$
\rho_s(q,r^\ast)=\EE_{Y\sim\QQ_q}\left[{1\over m}\sum_{j=1}^m\left(\max_{v\in\Vcal}r^\ast(v|Y_{<j})- \sum_{v\in\Vcal}q_{j}(v|Y_{<j})r^\ast(v|Y_{<j})\right) \right].
$$
A smaller value of $\rho_s$ indicates better alignment with the desired generation. These criteria can help evaluate the performance of generation in the pre-trained and post-trained models. We impose the following assumptions to analyze these criteria. 

\begin{assumption}
	\label{asp:generative}
	We make the following assumptions:
		\begin{enumerate}
		\item The reward function is consistent with the plausible set $\Scal_j(X_{<j})$, that is
				$$
				\zeta_j(X_{<j})=\min_{v\in \Scal_j(X_{<j})}r^\ast(v|X_{<j})-\max_{v\notin \Scal_j(X_{<j})}r^\ast(v|X_{<j})\ge 0
				$$
		and there exists at least one $v\in \Scal_j(X_{<j})$ such that $r^\ast(v|X_{<j})>\max_{v\notin \Scal_j(X_{<j})}r^\ast(v|X_{<j})$.
		\item The reward function is bounded, that is, $0\le r^\ast(v|X_{<j})\le R$ for all $v\in \Vcal$ and $X_{<j}$.
		\item The reward-estimation error is uniformly bounded:
		$$
		\sup_{v\in\Vcal}|\hat{r}(v|X_{<j})-r^\ast(v|X_{<j})|\le  \epsilon_r,\qquad a.s.
		$$
		\item The population prediction error of the fitted model is bounded by
		$$
		\EE_{X,j}\left[L\left(p_j(X);\hat{p}_{\hat{E}}(\hat{\Theta}_j(X))\right)\right]\le \epsilon_p.
		$$
		The tolerance $\epsilon_p$ may absorb model-approximation, finite-sample estimation, and optimization errors.
		\item The conditional probability mass function given the context is sparse. That is, there exists an integer $2\le k<d$ such that, for every context,
		$$
		p_j(v|X)\ge 1/k\quad {\rm if\ }v\in\Scal_j(X)\qquad {\rm and}\qquad p_j(v|X)=0\quad {\rm if\ }v\notin\Scal_j(X),
		$$
		where $\Scal_j(X)$ is the support of $p_j(v|X)$. Consequently,  $|\Scal_j(X)|\le k$.
		\item Let $J$ be uniformly distributed on $\{1,\ldots,m\}$. Denote by $\PP_{\rm pre}$ the distribution of $(J,X_{<J})$ under the pre-training population and by $\PP_{\QQ_q}$ the distribution of $(J,Y_{<J})$ when $Y\sim\QQ_q$. For each generator $q$ under consideration, assume that $\PP_{\QQ_q}$ is absolutely continuous with respect to $\PP_{\rm pre}$ and $\|d\PP_{\QQ_q}/d\PP_{\rm pre}\|_\infty\le C_{\QQ_q}$.
	\end{enumerate}
\end{assumption}

The first three assumptions require the true reward to favor plausible continuations and the estimated reward to be uniformly accurate. The fourth assumption controls the population prediction error of the fitted pre-trained model. The fifth assumption is a sparse-support idealization under which each
context admits at most $k$ plausible continuations. The exact-support assumption is adopted to make the effective choice set explicit and $\Scal_j(X)$ should be interpreted as the context-specific set of
plausible continuations. The final assumption controls the distribution shift between contexts observed during pre-training and contexts encountered during generation. Absolute continuity supplies coverage, while the bounded density ratio prevents generated contexts from being disproportionately concentrated in regions that are rare under the pre-training population. With these assumptions, the following theorem characterizes the performance of the pre-trained and post-trained models.

\begin{theorem}
	\label{thm:generative}
	Suppose Assumption~\ref{asp:generative} holds. Then, the generated sequence from both the pre-trained and post-trained generators is likely to stay within the set of plausible continuations
	$$
	\rho_{p}(q^{\rm pre},p)\ge 1-C_{\QQ^{\rm pre}}\epsilon_p\qquad {\rm and }\qquad \rho_{p}(q^{\rm post},p)\ge 1-C_{\QQ^{\rm post}}\exp\left( {2\epsilon_r\over \lambda}\right)\epsilon_p.
	$$
	In addition, the regret of the post-trained generator can be upper bounded
	$$
	\rho_s(q^{\rm post},r^\ast)\le \min\left\{R, 4 \epsilon_r+2\lambda\left(1+\log \left(2k\right)\right)+4C_{\QQ^{\rm post}}k^2\epsilon_p\right\}.
	$$
	There exist instances $(r^\ast, \hat{r}, p_j(X), q^{\rm pre})$ that satisfy Assumption~\ref{asp:generative} for which
	$$
	\rho_s(q^{\rm post},r^\ast)\ge  {1\over 2}\min\left\{R,2\epsilon_r+\lambda\log(k-1)\right\}.
	$$
	
	As a comparison, we also evaluate the performance of $q^{\hat{r},\pi,\lambda}$ when the reference distribution $\pi$ is a uniform distribution on $\Vcal$. We show
	$$
	\rho_s(q^{\hat{r},\pi,\lambda},r^\ast)\le \min\{R, 4 \epsilon_r+2\lambda\left(1+\log d\right)\},
	$$
	and there exist instances $(r^\ast, \hat{r})$ for which
	$$
	\rho_s(q^{\hat{r},\pi,\lambda},r^\ast)\ge  {1\over 2}\min\left\{R,2\epsilon_r+\lambda\log(d-1)\right\}.
	$$
\end{theorem}

Theorem~\ref{thm:generative} identifies two roles of the pre-trained generator. First, accurate next-token prediction concentrates its probability on contextually plausible continuations. The first conclusion shows that the pre-trained generator has plausibility error of order $\epsilon_p$. The post-trained generator inherits this property, up to the additional factor caused by reward-estimation error. Second, the pre-trained generator reduces the effective choice set in post-training. When the transferred pre-training error $k^2\epsilon_p$ is small, the upper bound for the post-trained generator depends on
$$
\epsilon_r+\lambda\log k,
$$
up to numerical constants and truncation at $R$. In contrast, using a uniform reference produces a corresponding dependence on
$$
\epsilon_r+\lambda\log d.
$$
The accompanying lower-bound constructions show that the dependence on $\epsilon_r$ and the logarithm of the effective choice-set size cannot generally be removed. These comparisons concern worst-case scaling and do not imply that the pre-trained reference dominates the uniform reference for every problem instance. Rather, they show that when next-token prediction is accurate and the true conditional distribution is sparse, the pre-trained generator can reduce the post-training complexity from the full vocabulary size $d$ to the context-specific plausible-set size $k$. In this sense, pre-training
acts as a context-dependent feasibility prior: it identifies plausible continuations, while post-training reweights these continuations according to the estimated reward.

\subsection{Token Community Recovery}
A natural downstream use of the learned token embeddings is to recover communities of token types. We define these communities through the conditional contextual distributions introduced in Section~\ref{sc:geometry}: token types belong to the same community when they occur in similar contexts. The geometry result in Theorem~\ref{thm:emdsimilarity} suggests that these communities can be recovered by clustering the learned embeddings. Let $(\hat{E},\hat{\Theta})$ denote the fitted representation pair obtained from the population token-prediction objective. We apply spectral clustering to the learned token embeddings $\hat{E}$. Specifically, we construct the Gaussian affinity matrix $W\in\RR^{d\times d}$ with entries
$$
W_{v_1,v_2}=\exp\left(-{\|\hat{E}_{v_1}-\hat{E}_{v_2}\|^2\over 2\sigma^2}\right).
$$
Given the affinity matrix $W$, we compute its graph Laplacian matrix $L=D-W$ where $D$ is a diagonal matrix with entry $D_{v,v}=\sum_{v'\in \Vcal}W_{v',v}$. We then find the first $K$ eigenvectors of $L$ corresponding to its $K$ smallest eigenvalues and form an eigenvector matrix $U\in \RR^{d\times K}$. With the eigenvector matrix, we apply $k$-means algorithm on the rows of $U$. In particular, $k$-means aims to partition $\Vcal$ into $K$ sets $\Vcal_1,\ldots,\Vcal_K$ so as to minimize the following optimization problem
$$
\min_{\Vcal_1,\ldots,\Vcal_K}\sum_{k=1}^K\sum_{v\in \Vcal_k}\|U_{v,\cdot}-\mu_k\|^2,
$$
where $\mu_k=|\Vcal_k|^{-1}\sum_{v\in \Vcal_k}U_{v,\cdot}$ is the mean of $U_{v,\cdot}$ in $\Vcal_k$. The partition $\Vcal_1,\ldots,\Vcal_K$ can naturally define the resulting labels $\hat{c}(v)$.

To study the performance of the clustering method, we need to introduce the underlying community structure. We use the squared Hellinger dissimilarity $\Gamma_h$ introduced in Section~\ref{sc:geometry} since it is a direct reflection of the distributional hypothesis. Specifically, we assume that all token types can be partitioned into $K$ groups $\Vcal=\Ccal_1\cup \ldots\cup\Ccal_K$. We assume that the token types within each group are close enough and token types between different groups are well separated, that is,
$$
\max_{1\le k\le K}\max_{v_1,v_2\in \Ccal_k}\Gamma_h(v_1,v_2)\le g_w^2\qquad {\rm and}\qquad \min_{1\le k_1<k_2\le K}\min_{v_1\in \Ccal_{k_1},v_2\in \Ccal_{k_2}}\Gamma_h(v_1,v_2)\ge g_b^2.
$$
Here, $g_w$ is an upper bound on within-community distributional distance and $g_b$ is a lower bound on between-community distance. Based on the partition, $c(v)\in \{1,\ldots,K\}$ represents the true clustering label. We consider misclustering rate as a measure of clustering performance
$$
\Phi(\hat{c},c)={1\over d}\min_{\pi\in \Pi_K}\sum_{v\in \Vcal}\bI(\hat{c}(v)\ne\pi(c(v))),
$$
where $\pi$ is a permutation of $\{1,\ldots,K\}$ and $\Pi_K$ denotes the set of all possible permutations. A smaller value of $\Phi(\hat{c},c)$ indicates better clustering performance. The following assumptions support the analysis of clustering performance.

\begin{assumption}
	\label{asp:cluster}
	We make the following assumptions:
	\begin{enumerate}
		\item Assumption~\ref{asp:embedding} holds with $E^\ast=\hat{E}$ and $\Theta^\ast=\hat{\Theta}$.
		\item The marginal probability of each token type is lower bounded, i.e., $p(v)\ge p_{\min}$ for every $v\in \Vcal$.
		\item The partition found by the $k$-means $\hat{\Vcal}_1,\ldots,\hat{\Vcal}_K$ satisfies
		$$
		\sum_{k=1}^K\sum_{v\in \hat{\Vcal}_k}\|U_{v,\cdot}-\hat{\mu}_k\|^2\le (1+\kappa)^2 \min_{\Vcal_1,\ldots,\Vcal_K,\mu_1,\ldots,\mu_K}\sum_{k=1}^K\sum_{v\in \Vcal_k}\|U_{v,\cdot}-\mu_k\|^2,
		$$
		where $\hat{\mu}_k=|\hat{\Vcal}_k|^{-1}\sum_{v\in \hat{\Vcal}_k}U_{v,\cdot}$ and $\kappa\ge 0$.
		\item Assume $g_w+\sqrt{\epsilon_p/p_{\min}}<1$ and $g_b>\sqrt{\epsilon_p/p_{\min}}$. Denote by $d_k=|\Ccal_k|$, $d_{\min}=\min_kd_k$, and $d_{\max}=\max_kd_k$. If we write 
		$$
		\alpha_\sigma=\left(1-\left(g_w+\sqrt{\epsilon_p\over  p_{\min}}\right)^2\right)^{4/ \sigma^2b_\Theta^2 } {\rm and}\quad \beta_\sigma=\left(1-\left(g_b-\sqrt{\epsilon_p\over  p_{\min}}\right)^2\right)^{4/\sigma^2B_\Theta^2},
		$$
		we assume
		$$
		{\alpha_\sigma\over \beta_\sigma}>{8(2+\kappa)\sqrt{Kd_{\max}}(d-d_{\min})\over d_{\min}^{3/2}}.
		$$
	\end{enumerate}
\end{assumption}

The first assumption imports the prediction-error, softmax-head, and curvature conditions used in Theorem~\ref{thm:emdsimilarity}. The lower bound on token frequencies makes the prediction-error remainder uniform over token types. Under the first two assumptions, Theorem~\ref{thm:emdsimilarity} implies
$$
\max_{1\le k\le K}\max_{v_1,v_2\in \Ccal_k}W_{v_1,v_2}\ge \alpha_\sigma \qquad {\rm and}\qquad \min_{1\le k_1<k_2\le K}\min_{v_1\in \Ccal_{k_1},v_2\in \Ccal_{k_2}}W_{v_1,v_2}\le \beta_\sigma.
$$
Thus, $\alpha_\sigma$ lower-bounds the within-community affinities, whereas $\beta_\sigma$ upper-bounds the between-community affinities. The approximate $k$-means assumption is standard in spectral-clustering analysis \citep{lei2015consistency}. The final assumption requires the between-community perturbation to be sufficiently small relative to the within-community spectral separation. With these assumptions, the following theorem can characterize the performance of token-type clustering.

\begin{theorem}
	\label{thm:cluster}
	Suppose Assumption~\ref{asp:cluster} holds. 
	Then, we have
	$$
	\Phi(\hat{c},c) \le {64(2+\kappa)^2Kd_{\max}\over d}\left({(d-d_{\min})\beta_\sigma\over \alpha_\sigma d_{\min}}\right)^2.
	$$
	In addition, if the communities are balanced, so that $d_k=d/K$ for every
	$k=1,\ldots,K$, then 
	$$
	\Phi(\hat{c},c) \le 64(2+\kappa)^2\left((K-1)\beta_\sigma/\alpha_\sigma\right)^2.
	$$
\end{theorem}

Theorem~\ref{thm:cluster} translates recovery of token geometry into a community-recovery guarantee. The quantity $\alpha_\sigma$ lower-bounds affinities within the true communities, whereas $\beta_\sigma$ upper-bounds affinities between communities. Consequently, the misclustering rate
decreases as the affinity ratio $\alpha_\sigma/\beta_\sigma$ increases. The ratio improves when the population prediction error is small, the token types are sufficiently frequent, and the conditional contextual distributions exhibit strong within-community similarity and between-community separation. Thus, token prediction does more than organize embeddings geometrically: the recovered geometry is sufficiently informative to support consistent community recovery. For balanced communities, the theorem yields consistent community recovery whenever $(K-1)\beta_\sigma/\alpha_\sigma\to0$.

\subsection{Token Classification with Linear Probes}
Finally, we study token classification using a linear probe on the learned contextual representation. Let $\{(X_i,j_i,Y_i):i=1,\ldots,n\}$ be a labeled data set, where $X_{i}$ is the $i$th token sequence, $j_i$ identifies the target position, and $Y_{i}\in \{1,2,\ldots, K\}$ is its class label. The feature $\hat{\Theta}_j(X)$ is constructed from the same observed context $X_{j,c}$ used in token prediction. Thus, although the label is associated with the realized target token, the linear probe receives only its learned contextual representation. For generic centered parameters $W=(W_1,\ldots,W_K)$ and
$b=(b_1,\ldots,b_K)$, define the predicted probability as
$$
P_{W,b,k}^{\hat{\Theta}}(X,j)={\exp(W_{k}^T\hat{\Theta}_{j}(X)+b_k)\over \sum_{k'=1}^K\exp(W_{k'}^T\hat{\Theta}_{j}(X)+b_{k'}) }.
$$
We estimate a ridge-penalized linear probe by minimizing the following empirical risk
$$
(\hat{W},\hat{b}) \in \argmin_{W,b:\sum_kW_k=0,\sum_kb_k=0}\left\{-{1\over n}\sum_{i=1}^n\log P_{W,b,Y_i}^{\hat{\Theta}}(X_i,j_i)+{1\over n}\left(\|W\|_F^2+\|b\|^2\right)\right\}.
$$
With the learned $\hat{W}$ and $\hat{b}$, the resulting classifier is 
$$
h_{\hat{W},\hat{b},\hat{\Theta}}(X,j)=\argmax_{1\le k\le K}\hat{W}_{k}^T\hat{\Theta}_{j}(X)+\hat{b}_k.
$$
One may naturally wonder how efficient the resulting classifier is and what the role of the learned representation is.

To investigate the performance of the resulting classifier, we assign each token type a preference score for each class
$$
s_k(v)={\max\{\PP_{X,j}(Y=k|x_j=v)-\tau,0\}\over 1-\tau},
$$
where $0<\tau<1$ is a preference threshold. Given the preference score, we introduce the preferred set of each class, $\Wcal_k=\{v\in\Vcal: s_k(v)>0\}$, that is, for $v\in \Wcal_k$, $\PP_{X,j}(Y=k|x_j=v)>\tau$. Thus, $v\in\Wcal_k$ when the occurrence of token type $v$ favors class $k$ over the threshold $\tau$, and $s_k(v)$ measures the strength of this preference. The preference need not be unambiguous: a token type may favor one class while occasionally receiving another label. For example, the token type ``Qatar'' strongly favors the location class, although the precise interpretation still depends on its context. Given a context $X_{j,c}$, the conditional distribution $p_j(v|X)$ describes which token types could naturally occur at position $j$. We therefore define the population evidence for class $k$ by
$$
S_k(X_{j,c})=\sum_{v\in \Vcal}s_k(v)p_j(v|X).
$$
This score is large when the context assigns substantial probability to token types that favor class $k$. For example, in the context ``traveled to Jordan,'' other location-indicating token types may also fit the position ``Jordan" and collectively provide evidence for the location class. 

This construction also explains why linear probing can be effective. If token types favoring the same class occur in similar contexts, accurate token-prediction training places their embedding vectors close together. Their contributions to the model-based evidence approximating $S_k(X_{j,c})$ can then be summarized by a common direction in the embedding space, making the logarithm of this evidence approximately linear in $\hat{\Theta}_j(X)$. To formalize this argument, we assume that token types favoring the same class have similar class-conditional contextual distributions:
$$
\max_{v_1,v_2\in \Wcal_k}H(p_{v_1,k},p_{v_2,k})\le g_{w},
$$
where $p_{v,k}(X_{j,c}|x_j=v,Y=k)$ denotes the conditional distribution of the context given the target token and label. Together with accurate token-prediction training, this contextual coherence implies that the embeddings in each preferred set are close. It is worth noting that the preference scores and population evidence scores serve only as proof devices and need not be computed by the practical downstream classifier.

To measure the performance of the resulting classifier, we consider the excess risk of misclassification error
$$
R(\hat{h})=\PP_{X,j,Y}(Y\ne \hat{h}(X,j))-\PP_{X,j,Y}(Y\ne h^\ast(X,j)),
$$
where $h^\ast$ is the optimal Bayes classification rule $h^\ast(X,j)=\argmax_{1\le k\le K}\eta_k(X_{j,c})$. Here, $\eta_k(X_{j,c})=\PP_{X,j}(Y=k|X_{j,c})$. A small value $R(\hat{h})$ indicates that $\hat{h}$ performs as well as the optimal Bayes classification rule. We introduce the following assumptions to study the excess misclassification risk.

\begin{assumption}
	\label{asp:classification}
	We make the following assumptions:
	\begin{enumerate}
		\item Suppose the loss is upper bounded by $\epsilon_p$, i.e.,
		$$
		\EE_{X,j}\left[L\left(p_j(X);\hat{p}_{\hat{E}}(\hat{\Theta}_j(X))\right)\right]\le \epsilon_p.
		$$
		The tolerance $\epsilon_p$ may absorb model-approximation, finite-sample estimation, and optimization errors.
		\item The kernel $\Kcal(\cdot,\cdot)$ is defined as $\log\Kcal(Z, E_v)=Z^TE_v$. The vector $\hat{\Theta}_j(X)$ is bounded, i.e., $\|\hat{\Theta}_j(X)\|\le B_\Theta$. The marginal probability of each token type is lower bounded, i.e., $p(v)\ge p_{\min}$ for every $v\in\Vcal$.
		\item We assume, for each $k$, 
		$$
		\tilde{e}^T\nabla^2 A(e)\tilde{e}\ge b^2_\Theta \|\tilde{e}\|^2,
		$$
		where $e\in {\rm conv}\{\hat{E}_v: v\in \Wcal_k\}$ and $\tilde{e}\in {\rm span}\{\hat{E}_{v_1}-\hat{E}_{v_2}: v_1,v_2\in \Wcal_k\}$.
		\item We assume $\sqrt{1-\tau+\tau g_w^2}+\sqrt{\epsilon_p/p_{\min}}<1$, $B_E=\max_{v\in \Vcal}\|\hat{E}_v\|<\infty$, and $0<S_{\min}\le \sum_{v\in \Wcal_k}s_k(v)\le S_{\max}<\infty$ for every $k$.
		\item Assume there exists a context-dependent evidence scale $\rho(X_{j,c})=\sum_{k=1}^KS_k(X_{j,c})$ such that $\rho(X_{j,c})\ge \rho_0>0$ and 
		$$
		\EE_{X,j}\left[1-\min_{k:S_k(X_{j,c})/\rho(X_{j,c})>0}{\PP_{X,j}(Y=k|X_{j,c})\over S_k(X_{j,c})/\rho(X_{j,c})}\right]\le \delta_s.
		$$
		\item Let $\eta_{(1)}(X_{j,c})$ and $\eta_{(2)}(X_{j,c})$ be the largest and second-largest scores of $\eta_{k}(X_{j,c})$ for $k=1,\ldots, K$. We assume 
		$$
		\PP_{X,j}\left(\eta_{(1)}(X_{j,c})=\eta_{(2)}(X_{j,c})\right)=0
		$$ 
		and
		$$
		\PP_{X,j}\left(0<\eta_{(1)}(X_{j,c})-\eta_{(2)}(X_{j,c})\le t\right)\le C_0t^\beta,
		$$
		where $C_0>0$ and $\beta>0$ are some constants.
		\item The labeled data set $(X_{i}, j_i, Y_i)$ for $1\le i\le n$ used in the downstream analysis is independent of data used to produce $\hat{\Theta}$ and $\hat{E}$. The $(X,j)$-marginal of
		$\PP_{X,j,Y}$ agrees with the context distribution appearing in the population token-prediction loss.
		\item There exists $\sigma_{\min}>0$ such that for any $e\in \RR^s$ and $a\in \RR$, $\EE_{X,j}(e^T\hat{\Theta}_j(X)+a)^2\ge \sigma_{\min}(\|e\|^2+a^2)$.
		\item The centered population logistic risk $\EE_{X,j,Y}(-\log P_{W,b,Y}^{\hat{\Theta}}(X,j))$ attains its minimum over $\{(W,b):\sum_{k=1}^K W_k=0, \sum_{k=1}^K b_k=0\}$ at $W^\ast$ and $b^\ast$, where $\|W^\ast\|_F^2+\|b^\ast\|^2\le {B^\ast}^2$ for some $B^\ast>0$.
	\end{enumerate}
\end{assumption}

The first four assumptions, together with the within-preferred-set contextual coherence introduced above, concern the embedding vectors of token types. As in Theorem~\ref{thm:cluster}, they imply that token types in each preferred set have embeddings with diameter at most $D$, while the bounds on the embedding norms and preference masses ensure that the resulting oracle linear probe is well defined and bounded. The fifth assumption is a separate compatibility condition between the pre-training evidence and the downstream labels. The quantity inside the expectation is the smallest residual mass needed to express the Bayes posterior as a mixture containing the normalized evidence distribution. Thus, $\delta_s$ measures the compatibility between the evidence induced by token prediction and the downstream Bayes rule. The sixth assumption imposes uniqueness of the Bayes class and a standard margin condition \citep{mammen1999smooth,audibert2007fast}. The seventh assumption allows the downstream analysis to condition on the pre-trained representation while the eighth assumption ensures identifiability of the centered logistic parameters. The last assumption provides a bounded population minimizer and excludes population separation \citep{albert1984existence,bach2010selfconcordant}. Together with bounded representations, the last two conditions provide the curvature needed for finite-sample estimation. The following theorem combines the linear-probe approximation enabled by pre-training, its compatibility with the downstream task, and the error of estimating the probe from finitely many labels.

\begin{theorem}
	\label{thm:classification}
	Suppose Assumption~\ref{asp:classification} holds. If $\gamma>0$ and $n\ge n_0
	$,
	then, with probability at least $1-e^{-\gamma}$,
	$$
	R(h_{\hat{W},\hat{b},\hat{\Theta}})\le 2^{4\beta+7\over\beta+2}C_0^{1\over \beta+2}\left({C_{c,1}+C_{c,2}\gamma+C_{c,4}\over n}+ {2\epsilon_p\over \rho_0}+{D^4B_\Theta^4\over 16}+C_{c,3}\delta_s\right)^{\beta+1\over \beta+2},
	$$
	where $D$ is defined by
	$$
	D^2= -{8\over b_\Theta^2}\log\left(1-\left(\sqrt{1-\tau+\tau g_w^2}+\sqrt{\epsilon_p\over  p_{\min}}\right)^2\right).
	$$
	Here, $n_0$, $C_{c,1}$, $C_{c,2}$, $C_{c,3}$, and $C_{c,4}$ are constants defined in the proof.
\end{theorem}

Theorem~\ref{thm:classification} separates the excess risk into three population approximation errors and one downstream estimation error. The term $2\epsilon_p/\rho_0$ measures the error incurred when the true conditional token distribution in the class-evidence score is replaced by the fitted prediction head. The term $D^4B_\Theta^4/16$ is a geometric linearization error: it is small when the embeddings within each preferred set are sufficiently concentrated for their contributions to be represented by a common linear direction. The term $C_{c,3}\delta_s$ measures compatibility between the normalized evidence scores and the downstream Bayes posterior. Finally, the estimation error is the cost of estimating
the linear probe from $n$ labeled observations. Consequently, the three population approximation errors determine the representation-induced error floor. Thus, when the population prediction, geometric linearization, and downstream-compatibility errors are small, token prediction can organize the discrete context space into a representation in which class-specific evidence is approximately linear, leaving a low-dimensional estimation problem for the labeled samples.

\section{Numerical Illustration}
\label{sc:numerical}

To illustrate the phenomena characterized in the preceding sections, we conduct a controlled simulation experiment. Since the simulation setting is at a small scale by design, we can exactly evaluate all population quantities of interest, such as conditional probability distributions of context and true reward regret. These population quantities allow us to compare our estimates with the ground truth and assess whether the observed phenomena are consistent with the theoretical investigation. 

In the simulation experiment, the vocabulary contains 16 semantic groups , $\Vcal_{s,1},\ldots, \Vcal_{s,16}$, arranged into 4 communities, $\Vcal_{c,1},\ldots,\Vcal_{c,4}$, and each semantic group includes three exchangeable token types. The three exchangeable token types have the same contextual distribution but different marginal frequencies. In total, we have 48 token types. The token sequence follows a stationary third-order process, so the distribution of the next token depends on the preceding three tokens through an order-sensitive nonlinear rule. In the conditional distribution of the next token, our construction can ensure a sparse plausible set $\Scal_j(X)$, which includes six semantic groups or eighteen token types. The context of a token is defined as its preceding three tokens. The resulting contextual population contains $16^3=4096$ possible semantic contexts and $48^3=110592$ possible observed-token contexts, allowing us to compute $p_j(X)$, the contextual geometry $\Gamma_h$, community labels, and downstream optimal Bayes misclassification error exactly. In the next-token prediction training, the model has access only to sampled contexts and next tokens, but not semantic groups, communities, or latent construction variables.

In the next-token prediction training, we use $5\times 10^5$ context–next-token pairs as training data set and separate $2\times 10^4$ pairs as validation data set. Similar to GPT, we add positional embeddings and process them using a stack of causal self-attention Transformer layers. The dimension of learned representation is 16 and each Transformer layer contains two attention heads and a feed-forward network of dimension 64. The encoder repeatedly applies a shared block containing either one Transformer layer, with repeated $L\in\{1,2,4,8\}$ times, or two Transformer layers, with repeated $L\in\{1,2,4\}$ times. The prediction head is a simple softmax decoder with no additional hidden layers. The input and output token embeddings are tied, so the learned matrix of initial representation is also used by the prediction head.  We evaluate the population prediction error $\epsilon_p$ and all subsequent quantities by exact enumeration. Further details of the data-generation process and model-fitting procedure can be found in Appendix~\ref{app:numerical}.

\begin{figure}[h!]
	\centering
	\includegraphics[width=\linewidth]{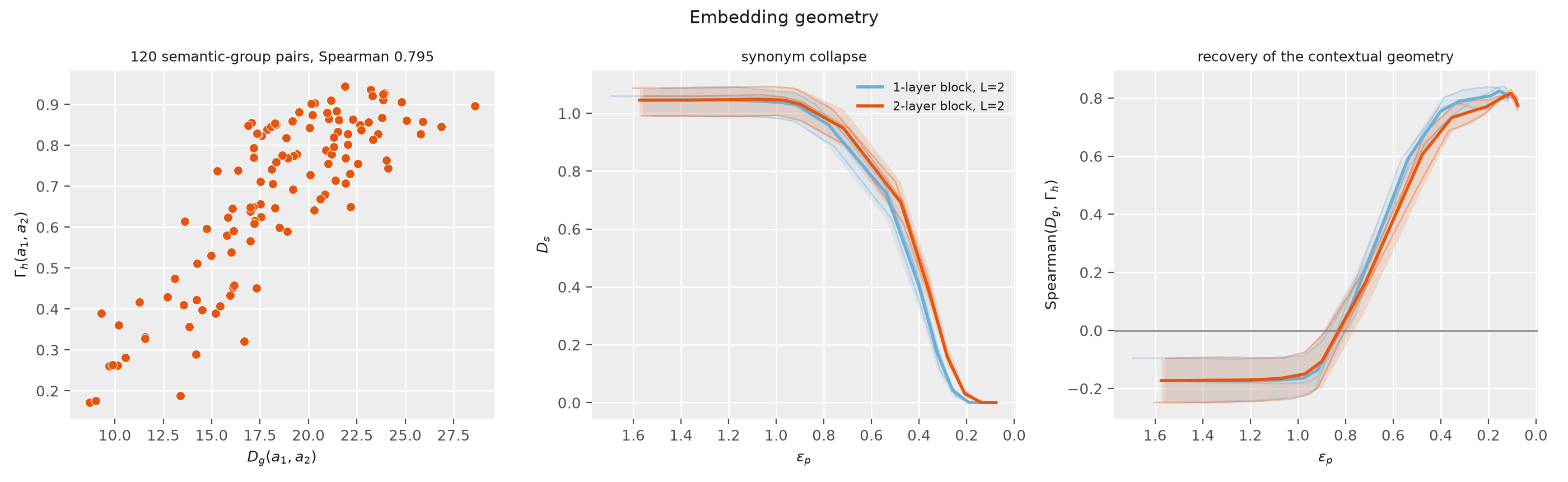}
	\caption{Recovery of the semantic geometry of token types. Left: the embedding distance $D_g(a_1,a_2)$ between semantic groups against the contextual dissimilarity $\Gamma_h(a_1,a_2)$, over all $120$ pairs of the $16$ groups. Middle: the synonym distance ratio $D_s$. Right: the Spearman correlation between $D_g$ and $\Gamma_h$. In the middle and right panels the horizontal axis is the population prediction error $\epsilon_p$, plotted in decreasing order so that training proceeds from left to right. }
	\label{fg:geometry}
\end{figure}

We first study whether the next-token prediction can help recover semantic geometry of the token types. As indicated by Theorems~\ref{thm:emdsimilarity} and \ref{thm:cluster}, a natural measure of semantic dissimilarity is the Hellinger distance between conditional contextual distributions $\Gamma_h(v_1,v_2)$. If two token types belong to the same semantic group, i.e., $v_1,v_2\in \Vcal_{s,a}$ for $1\le a\le 16$, then our construction suggests $\Gamma_h(v_1,v_2)=0$. We therefore examine whether embeddings of synonymous token types within the same semantic group collapse and whether distances between group embeddings recover the dissimilarities between distinct semantic groups. Specifically, we consider the following two criteria: synonym distance ratio and embedding distance between semantic groups
$$
D_{s}={\sum_{a=1}^{16}\sum_{v_1,v_2\in \Vcal_{s,a}}\|E_{v_1}-E_{v_2}\|^2/48\over M\{\|E_{v_1}-E_{v_2}\|^2\}_{v_1\in  \Vcal_{s,a_1}, v_2\in  \Vcal_{s,a_2},a_1\ne a_2}}\quad {\rm and}\quad D_g(a_1,a_2)=\left\|{1\over 3}\sum_{v\in \Vcal_{s,a_1}}E_{v}-{1\over 3}\sum_{v\in \Vcal_{s,a_2}}E_{v}\right\|^2,
$$
where $M\{\cdot\}$ represents the median of a set of numbers. A reasonable embedding can indicate a small $D_s$ and suggests that the level of $D_g(a_1,a_2)$ reflects the level of $\Gamma_h(a_1,a_2)=\Gamma_h(v_1,v_2)$ when $v_1\in  \Vcal_{s,a_1}$ and $v_2\in  \Vcal_{s,a_2}$. The results summarized in Figure~\ref{fg:geometry} show a high correlation between embedding distance and the Hellinger distance between conditional contextual distributions, indicating that the embeddings recover the semantic geometry well. In addition, embedding quality depends strongly on how well the fitted model approximates the population distribution, indicated by $\epsilon_p$, suggesting a well-fitted model is necessary for the embedding to recover the semantic geometry.

\begin{figure}[h!]
	\centering
	\includegraphics[width=\linewidth]{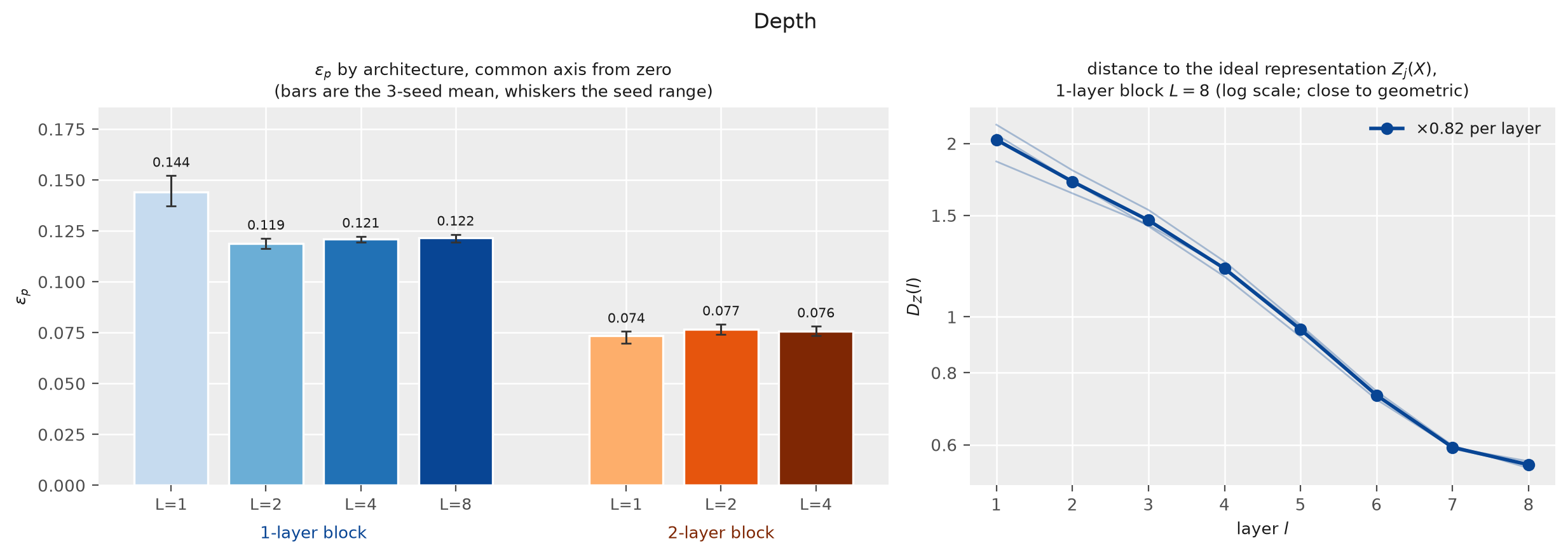}
	\caption{Role of architecture in representation maps. Left: $\epsilon_p$ for all seven architectures on a common axis anchored at zero. Right: the normalized distance $D_Z(l)$ to the ideal representation along the layers of the 1-layer block with $L=8$, on a logarithmic vertical scale; $l=0$ is omitted because $\Psi^0_j(X)$ does not depend on $X$ by construction.}
	\label{fg:depth}
\end{figure}

We then investigate the role of architecture in representation maps and evaluate whether repeated applications of the same block can accurately approximate the ideal representation. To examine whether the intermediate layers of the transformer move toward a self-consistent representation, we define the ideal representation for a given recovery head as
$$
Z_j(X)=\argmin_{\|Z\|\le 4}L(p_j(X);\hat{p}_{\hat{E}}(Z)).
$$
The ideal representation $Z_j(X)$ is the most flexible norm-constrained representation when the recovery head is given. Given the $l$th layer $\Psi^l_j(X)$ of the representation, we evaluate the normalized distance from the ideal representation
$$
D_Z(l)={\EE_{X,j}\|\Psi^l_j(X)-Z_j(X)\|^2\over \EE_{X,j}\|Z_j(X)- \EE_{X,j}Z_j(X)\|^2}.
$$
Theorem~\ref{thm:aprx} bounds the approximation by $\kappa_\Theta^L D(\Psi^0,\Theta) + \delta_\Theta/(1-\kappa_\Theta)$, so depth buys geometric decay of the first term only until the block-quality floor is reached. The two panels of Figure~\ref{fg:depth} separate these terms. The right panel exhibits the geometric decay directly: $D_Z(l)$ falls by a factor of 3.9 over $l=1,\ldots,8$, close to a constant ratio of $0.82$ per layer. The left panel shows the empirical saturation level. Repeating the 1-layer block improves $\epsilon_p$ from $0.144$ to $0.119$ at $L=2$ and no further, while the more expressive 2-layer block is already at its saturation level after a single application, at $0.074$. Reducing the single-block approximation error is therefore the effective lever.

\begin{figure}[h!]
	\centering
	\includegraphics[width=0.9\linewidth]{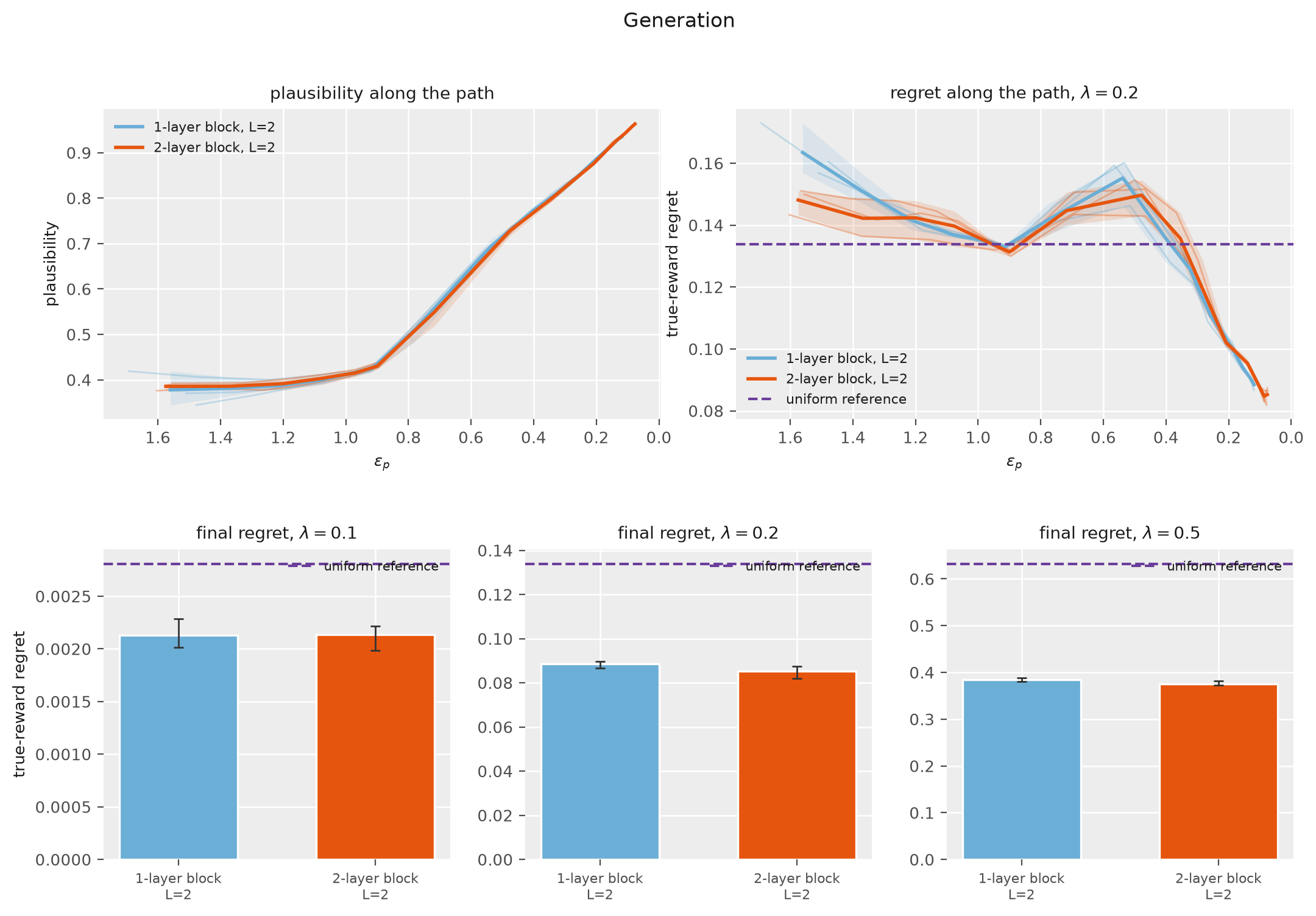}
	\caption{Performance on the generation task. Top left: plausibility, the mass the generator places on the population plausible set $\mathcal{S}_j(X)$. Top right: true-reward regret at post-training temperature $\lambda=0.2$, with the regret of the uniformly-referenced post-trained generator marked. Both are plotted against $\epsilon_p$ in decreasing order along the training path. Bottom: final regret at the three temperatures $\lambda\in\{0.1,0.2,0.5\}$, each compared with the corresponding uniformly referenced generator (note the differing vertical scales).}
	\label{fg:generation}
\end{figure}

We next evaluate the performance of generation in the generative task with the following two criteria: the pre-trained model's plausibility and the post-trained model's regret. In the post-training, the pre-trained generator is tilted using a noisy estimator of the reward. The top left panel in Figure~\ref{fg:generation} suggests that a smaller pre-training loss is accompanied by a greater tendency to remain within the population’s plausible continuation set. Similarly, as the top right panel in Figure~\ref{fg:generation} indicates, the resulting generator after post-training can achieve a small regret when the pre-training prediction error is small. In addition, we compare the resulting post-trained generator with another post-trained generator when the reference distribution is a uniform distribution in the bottom panels of Figure~\ref{fg:generation}. The comparison suggests that pre-training can substantially reduce the regret of the resulting post-trained generator.

\begin{figure}[h!]
	\centering
	\includegraphics[width=\linewidth]{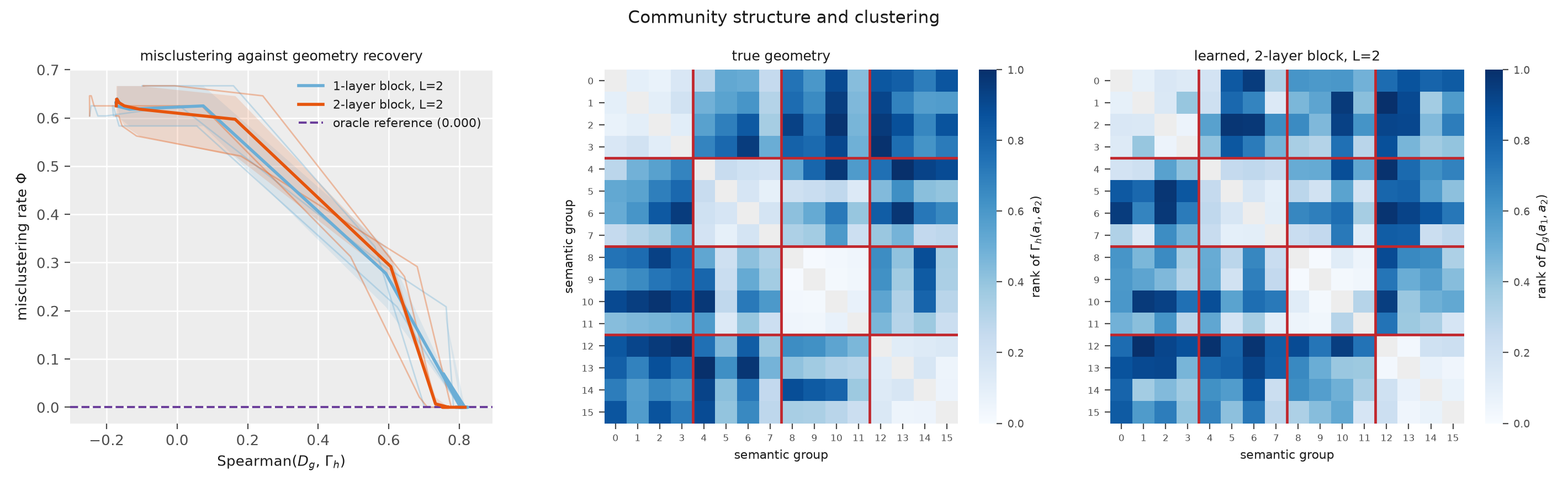}
	\caption{Recovery of the community structure. Left: the misclustering rate $\Phi$ of spectral clustering on the $48$ token embeddings, against the Spearman correlation between $D_g$ and $\Gamma_h$, so that clustering accuracy is read against the quality of the recovered geometry; the oracle reference is the same pipeline applied to the exact $\Gamma_h$. Middle and right: the exact contextual geometry $\Gamma_h$ and the learned embedding geometry $D_g$ over the $16$ semantic groups. Each is displayed on its own rank scale; red lines mark the four true communities.}
	\label{fg:clustering}
\end{figure}

We then study whether we can recover the community structure of token types by the learned embedding vectors. Specifically, we apply spectral clustering to the learned embedding vectors and adopt the misclustering rate to evaluate the performance. As illustrated in the left panel of Figure~\ref{fg:clustering}, we can better detect the token type communities when we can better recover semantic geometry. The side-by-side rank heatmaps in Figure~\ref{fg:clustering} show that the block structure of the semantic geometry is visible in the learned embedding geometry. This provides a direct numerical illustration of the mechanism underlying Theorem~\ref{thm:cluster}: accurate token prediction organizes the embedding space in a form that supports community recovery.

\begin{figure}[h!]
	\centering
	\includegraphics[width=0.93\linewidth]{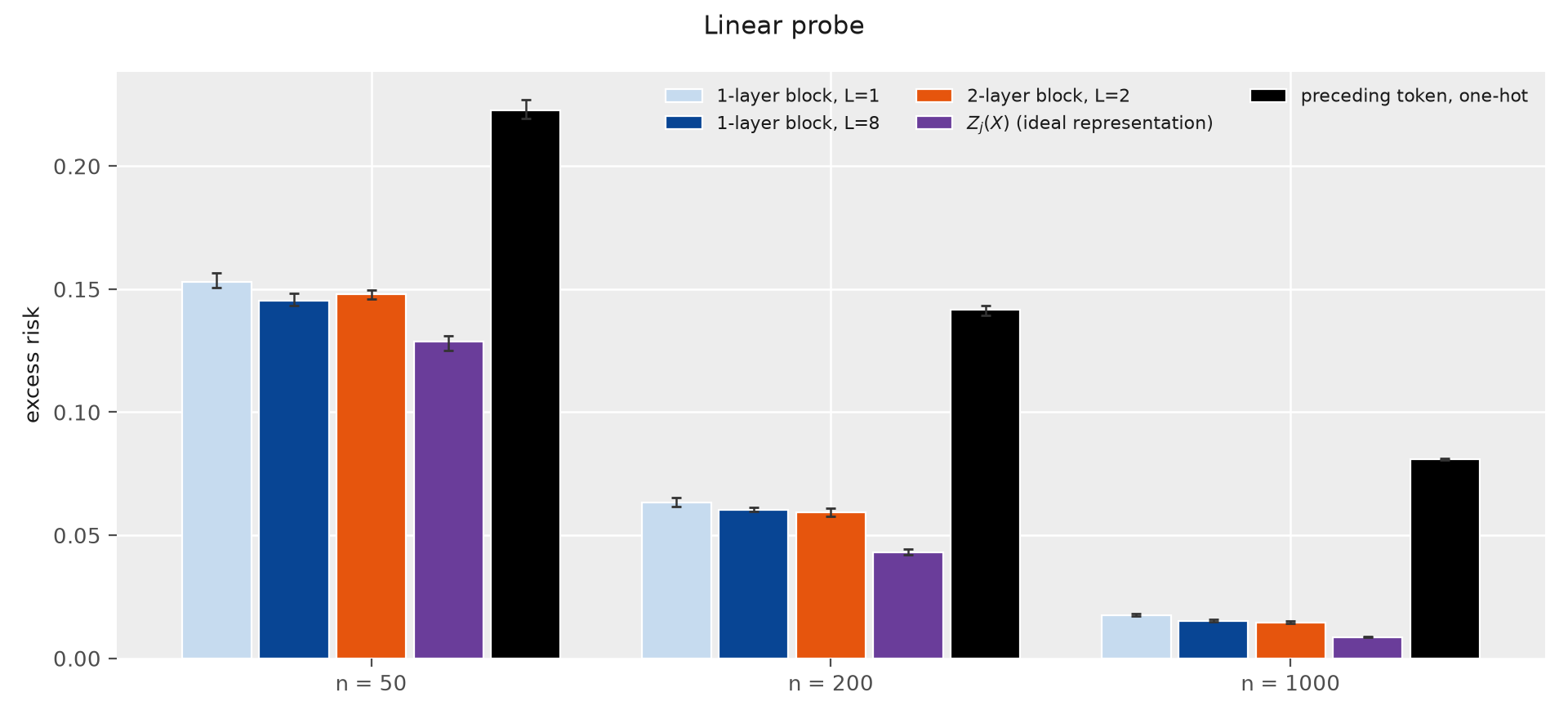}
	\caption{Excess risk of the linear probe at three downstream sample sizes. The ideal representation $Z_j(X)$ is computed for the decoder of the fitted 1-layer block with $L=8$. The preceding-token one-hot representation uses only the observed token immediately before the target and carries no learned contextual information. }
	\label{fig:numerical-probe}
\end{figure}

Finally, we use a linear probe to predict the semantic-community label of the next token. This label may differ from the community of the observed token immediately preceding the target. We consider downstream sample sizes $n\in\{50,200,1000\}$ and compare five representations: the one-layer block with $L=1$, the one-layer block with $L=8$, the two-layer block with $L=2$, the decoder-optimal representation $Z_j(X)$ for the fitted one-layer block with $L=8$, and a one-hot representation of the token immediately preceding the target. As expected, increasing the labeled sample size reduces excess risk, and the decoder-optimal representation generally yields the most accurate classifier.

\section{Concluding Remarks}
\label{sc:conclusion}

In this paper, we develop a statistical framework to clarify how token prediction learns representations and why these representations can be useful in downstream analysis. Our central finding is that token prediction can recover the geometry of a distribution that it is not directly trained to predict. Although the prediction objective estimates the distribution of a target token given its context, accurate prediction also organizes token embeddings according to the conditional distributions of contexts associated with different token types. This result reveals two coupled representations: the token embeddings recover contextual similarity among token types, while the contextual representation serves as a latent coordinate of the conditional token distribution within the model determined by these embeddings. We further introduce a self-consistency principle to characterize how repeated representation blocks can approximate this contextual representation and favor predictively equivalent representation systems that are compatible with recurrent construction. Finally, we show that prediction accuracy and the resulting representation geometry support token generation, token community recovery, and token classification. Taken together, these results characterize token prediction as a mechanism for learning the conditional-distribution structure that makes representations useful beyond the original prediction objective.

The framework also assigns different statistical roles to the encoder and prediction head, providing an interpretation of the asymmetric encoder–head designs commonly used in practice \citep{he2022masked,radford2018improving}. The prediction head maps the contextual representation to a conditional token distribution and therefore determines how contextual representations interact with token embeddings. Its structure consequently shapes the geometry that can be identified through token prediction. A simple prediction head, such as an inner-product softmax head, has limited capacity to absorb the complexity of the conditional token distribution and makes the resulting representation geometry directly interpretable. The encoder, by contrast, must construct the contextual representation from a discrete and potentially complex token sequence. The self-consistency analysis explains how this complex map can be approximated through repeated applications of expressive representation blocks. Our results therefore suggest that a simple prediction head and a flexible encoder play complementary roles: the head specifies a meaningful representation geometry, while the encoder approximates the map from observed contexts into that geometry. This interpretation does not imply that the simplest possible head is always optimal, but it highlights that the complexity of the prediction head affects what information must be carried by the learned representations.

The analysis also highlights the fundamental role of the pre-training distribution. The geometry-recovery bounds deteriorate with the marginal probabilities of the token types, reflecting the fact that accurate average prediction does not provide reliable geometric information about tokens that are rarely observed or predicted. Poorly estimated token geometry can subsequently weaken community recovery and the geometric linearization underlying a downstream linear probe. Similarly, the generation analysis requires adequate coverage of the contexts encountered after generation begins. These limitations cannot generally be overcome by a more elaborate architecture alone: the training distribution must contain sufficient information about both the relevant token types and the contexts in which the model will be used. Effective token-prediction-based representation learning therefore depends jointly on architectural design, prediction accuracy, and coverage of the relevant token and context spaces.

Most of our analysis uses the inner-product softmax head $\log\Kcal(Z,E_v)=Z^T E_v$. Many of the underlying arguments can be extended to a log-separable kernel of the form
$$
\log\Kcal(Z, E_v)=k_Z(Z)+k_E(E_v)+K_Z(Z)^TK_E(E_v),
$$
where $k_Z$, $k_E$, $K_Z$, and $K_E$ are suitable functions. After normalization over the vocabulary, $k_Z(Z)$ cancels, while $k_E(E_v)$ acts as a token-specific bias. The transformed quantities $K_Z(\Theta_j(X))$ and $K_E(E_v)$ then play the roles of the contextual representation and token embedding in the recovery head. The resulting geometry is naturally characterized in these transformed coordinates. Relating it back to distances between the original $E_v$ or $\Theta_j(X)$ requires additional regularity, such as bi-Lipschitz properties of $K_E$ and $K_Z$. Extending the downstream guarantees also requires corresponding boundedness, curvature, and compatibility conditions. Developing such extensions would clarify how the complexity of the prediction head changes the information encoded in the original representation space.

The present analysis is primarily conducted at the population level and assumes that the relevant population optimization problems can be solved. This perspective isolates the structural consequences of token prediction, but it does not characterize the amount of pre-training data required to recover the population geometry or the effect of optimization on the selected representation. Finite-sample errors may be especially important for rare token types, while optimization may introduce an additional implicit preference among the many representation systems having similar prediction loss. A complete statistical account should therefore combine the population mechanisms studied here with finite-sample pre-training theory, optimization dynamics, and model misspecification. The controlled numerical experiment illustrates the predicted mechanisms, but evaluating their quantitative relevance in large pre-trained language models is also an important direction for future work.

Finally, token prediction belongs to the broader family of self-supervised learning methods. Another major strategy constructs representations through data augmentation, including contrastive and non-contrastive Siamese methods \citep{chen2020simple,he2020momentum,zbontar2021barlow}. Recent theoretical studies suggest that augmentation-based methods can learn semantic representations by collapsing distinctions induced by meaning-preserving transformations while retaining meaningful differences among samples \citep{wang2023self,wang2025linear}. The present analysis identifies a related mechanism for token prediction: token types having indistinguishable conditional contextual distributions are encouraged to collapse, while differences between contextual distributions are preserved in the learned geometry. Token prediction and data augmentation therefore provide two different routes toward the same broad representation-learning principles of equivalence collapse and semantic separation. Understanding whether these principles also characterize other pretext tasks may provide a more unified statistical theory of self-supervised representation learning.

\section*{Acknowledgment}
This work is supported by grants from the National Science Foundation (DBI-2243257 and DMS-2515171). Generative AI tools, including OpenAI’s ChatGPT and Anthropic’s Claude Code, were used for brainstorming, language editing and coding assistance. The author reviewed all outputs and takes full responsibility for the manuscript.

\bibliographystyle{plainnat}
\bibliography{MaskedRepresentation}

\begin{thebibliography}{41}
\providecommand{\natexlab}[1]{#1}
\providecommand{\url}[1]{\texttt{#1}}
\expandafter\ifx\csname urlstyle\endcsname\relax
  \providecommand{\doi}[1]{doi: #1}\else
  \providecommand{\doi}{doi: \begingroup \urlstyle{rm}\Url}\fi

\bibitem[Albert and Anderson(1984)]{albert1984existence}
A.~Albert and J.~A. Anderson.
\newblock On the existence of maximum likelihood estimates in logistic
  regression models.
\newblock \emph{Biometrika}, 71\penalty0 (1):\penalty0 1--10, 1984.

\bibitem[Audibert and Tsybakov(2007)]{audibert2007fast}
J.~Audibert and A.~B. Tsybakov.
\newblock Fast learning rates for plug-in classifiers.
\newblock \emph{The Annals of Statistics}, 35\penalty0 (2):\penalty0 608--633,
  2007.

\bibitem[Bach(2010)]{bach2010selfconcordant}
F.~Bach.
\newblock Self-concordant analysis for logistic regression.
\newblock \emph{Electronic Journal of Statistics}, 4:\penalty0 384--414, 2010.

\bibitem[Bai et~al.(2019)Bai, Kolter, and Koltun]{bai2019deep}
S.~Bai, J.~Z. Kolter, and V.~Koltun.
\newblock Deep equilibrium models.
\newblock In \emph{Advances in Neural Information Processing Systems},
  volume~32, 2019.

\bibitem[Bartlett and Mendelson(2002)]{bartlett2002rademacher}
P.~L. Bartlett and S.~Mendelson.
\newblock Rademacher and {Gaussian} complexities: risk bounds and structural
  results.
\newblock \emph{Journal of Machine Learning Research}, 3:\penalty0 463--482,
  2002.

\bibitem[Bengio et~al.(2013)Bengio, Courville, and
  Vincent]{bengio2013representation}
Y.~Bengio, A.~Courville, and P.~Vincent.
\newblock Representation learning: A review and new perspectives.
\newblock \emph{IEEE Transactions on Pattern Analysis and Machine
  Intelligence}, 35\penalty0 (8):\penalty0 1798--1828, 2013.

\bibitem[Boucheron et~al.(2013)Boucheron, Lugosi, and Massart]{boucheron2013}
S.~Boucheron, G.~Lugosi, and P.~Massart.
\newblock \emph{Concentration Inequalities: A Nonasymptotic Theory of
  Independence}.
\newblock Oxford University Press, 2013.
\newblock ISBN 9780199535255.

\bibitem[Brown et~al.(2020)Brown, Mann, Ryder, Subbiah, Kaplan, Dhariwal,
  Neelakantan, Shyam, Sastry, Askell, Agarwal, Herbert-Voss, Krueger, Henighan,
  Child, Ramesh, Ziegler, Wu, Winter, Hesse, Chen, Sigler, Litwin, Gray, Chess,
  Clark, Berner, McCandlish, Radford, Sutskever, and Amodei]{brown2020language}
T.~B. Brown, B.~Mann, N.~Ryder, M.~Subbiah, J.~Kaplan, P.~Dhariwal,
  A.~Neelakantan, P.~Shyam, G.~Sastry, A.~Askell, S.~Agarwal, A.~Herbert-Voss,
  G.~Krueger, T.~Henighan, R.~Child, A.~Ramesh, D.~M. Ziegler, J.~Wu,
  C.~Winter, C.~Hesse, M.~Chen, E.~Sigler, M.~Litwin, S.~Gray, B.~Chess,
  J.~Clark, C.~Berner, S.~McCandlish, A.~Radford, I.~Sutskever, and D.~Amodei.
\newblock Language models are few-shot learners.
\newblock In \emph{Advances in Neural Information Processing Systems},
  volume~33, pages 1877--1901, 2020.

\bibitem[Chen et~al.(2026)Chen, Huang, Golowich, Malladi, Block, Ash,
  Krishnamurthy, and Foster]{chen2026the}
F.~Chen, A.~Huang, N.~Golowich, S.~Malladi, A.~Block, J.~T. Ash,
  A.~Krishnamurthy, and D.~J. Foster.
\newblock The coverage principle: How pre-training enables post-training.
\newblock In \emph{The Fourteenth International Conference on Learning
  Representations}, 2026.

\bibitem[Chen et~al.(2020)Chen, Kornblith, Norouzi, and Hinton]{chen2020simple}
T.~Chen, S.~Kornblith, M.~Norouzi, and G.~Hinton.
\newblock A simple framework for contrastive learning of visual
  representations.
\newblock In \emph{International Conference on Machine Learning}, pages
  1597--1607. PMLR, 2020.

\bibitem[Cover and Thomas(2006)]{cover2006elements}
T.~M. Cover and J.~A. Thomas.
\newblock \emph{Elements of Information Theory}.
\newblock Wiley-Interscience, Hoboken, NJ, 2nd edition, 2006.

\bibitem[Dehghani et~al.(2019)Dehghani, Gouws, Vinyals, Uszkoreit, and
  Kaiser]{dehghani2019universal}
M.~Dehghani, S.~Gouws, O.~Vinyals, J.~Uszkoreit, and {\L}.~Kaiser.
\newblock Universal transformers.
\newblock In \emph{International Conference on Learning Representations}, 2019.

\bibitem[Devlin et~al.(2019)Devlin, Chang, Lee, and Toutanova]{kenton2019bert}
J.~Devlin, M.~Chang, K.~Lee, and K.~Toutanova.
\newblock {BERT}: Pre-training of deep bidirectional transformers for language
  understanding.
\newblock In \emph{Proceedings of NAACL-HLT}, pages 4171--4186, 2019.

\bibitem[Efron(2023)]{efron2023exponential}
B.~Efron.
\newblock \emph{Exponential families in theory and practice}.
\newblock Cambridge University Press, 2023.

\bibitem[He et~al.(2020)He, Fan, Wu, Xie, and Girshick]{he2020momentum}
K.~He, H.~Fan, Y.~Wu, S.~Xie, and R.~Girshick.
\newblock Momentum contrast for unsupervised visual representation learning.
\newblock In \emph{Proceedings of the IEEE/CVF conference on computer vision
  and pattern recognition}, pages 9729--9738, 2020.

\bibitem[He et~al.(2022)He, Chen, Xie, Li, Doll{\'a}r, and
  Girshick]{he2022masked}
K.~He, X.~Chen, S.~Xie, Y.~Li, P.~Doll{\'a}r, and R.~Girshick.
\newblock Masked autoencoders are scalable vision learners.
\newblock In \emph{Proceedings of the IEEE/CVF Conference on Computer Vision
  and Pattern Recognition}, pages 16000--16009, 2022.

\bibitem[Inan et~al.(2017)Inan, Khosravi, and Socher]{inan2017tying}
H.~Inan, K.~Khosravi, and R.~Socher.
\newblock Tying word vectors and word classifiers: A loss framework for
  language modeling.
\newblock In \emph{International Conference on Learning Representations}, 2017.

\bibitem[Lee et~al.(2021)Lee, Lei, Saunshi, and Zhuo]{lee2021predicting}
J.~D. Lee, Q.~Lei, N.~Saunshi, and J.~Zhuo.
\newblock Predicting what you already know helps: Provable self-supervised
  learning.
\newblock In \emph{Advances in Neural Information Processing Systems},
  volume~34, pages 3094--3106, 2021.

\bibitem[Lei and Rinaldo(2015)]{lei2015consistency}
J.~Lei and A.~Rinaldo.
\newblock Consistency of spectral clustering in stochastic block models.
\newblock \emph{The Annals of Statistics}, pages 215--237, 2015.

\bibitem[Lenci and Sahlgren(2023)]{lenci2023distributional}
A.~Lenci and M.~Sahlgren.
\newblock \emph{Distributional semantics}.
\newblock Cambridge University Press, 2023.

\bibitem[Liu et~al.(2022)Liu, Hsu, Ravikumar, and Risteski]{liu2022masked}
B.~Liu, D.~Hsu, P.~Ravikumar, and A.~Risteski.
\newblock Masked prediction tasks: A parameter identifiability view.
\newblock In \emph{Advances in Neural Information Processing Systems},
  volume~35, pages 33171--33184, 2022.

\bibitem[Liu et~al.(2023)Liu, Xie, Li, and Ma]{liu2023same}
H.~Liu, S.~Xie, Z.~Li, and T.~Ma.
\newblock Same pre-training loss, better downstream: Implicit bias matters for
  language models.
\newblock In \emph{Proceedings of the 40th International Conference on Machine
  Learning}, volume 202 of \emph{Proceedings of Machine Learning Research},
  pages 21756--21779. PMLR, 2023.

\bibitem[Mammen and Tsybakov(1999)]{mammen1999smooth}
E.~Mammen and A.~B. Tsybakov.
\newblock Smooth discrimination analysis.
\newblock \emph{The Annals of Statistics}, 27\penalty0 (6):\penalty0
  1808--1829, 1999.

\bibitem[Maurer(2016)]{maurer2016vector}
A.~Maurer.
\newblock A vector-contraction inequality for rademacher complexities.
\newblock In \emph{International Conference on Algorithmic Learning Theory},
  pages 3--17. Springer, 2016.

\bibitem[Ouyang et~al.(2022)Ouyang, Wu, Jiang, Almeida, Wainwright, Mishkin,
  Zhang, Agarwal, Slama, Ray, Schulman, Hilton, Kelton, Miller, Simens, Askell,
  Welinder, Christiano, Leike, and Lowe]{ouyang2022training}
L.~Ouyang, J.~Wu, X.~Jiang, D.~Almeida, C.~L. Wainwright, P.~Mishkin, C.~Zhang,
  S.~Agarwal, K.~Slama, A.~Ray, J.~Schulman, J.~Hilton, F.~Kelton, L.~Miller,
  M.~Simens, A.~Askell, P.~Welinder, P.~Christiano, J.~Leike, and R.~Lowe.
\newblock Training language models to follow instructions with human feedback.
\newblock In \emph{Advances in Neural Information Processing Systems}, 2022.

\bibitem[Press and Wolf(2017)]{press2017using}
O.~Press and L.~Wolf.
\newblock Using the output embedding to improve language models.
\newblock In \emph{Proceedings of the 15th Conference of the European Chapter
  of the Association for Computational Linguistics: Volume 2, Short Papers},
  pages 157--163, Valencia, Spain, 2017. Association for Computational
  Linguistics.

\bibitem[Radford et~al.(2018)Radford, Narasimhan, Salimans, and
  Sutskever]{radford2018improving}
A.~Radford, K.~Narasimhan, T.~Salimans, and I.~Sutskever.
\newblock Improving language understanding by generative pre-training.
\newblock \emph{{OpenAI}}, 2018.

\bibitem[Raffel et~al.(2020)Raffel, Shazeer, Roberts, Lee, Narang, Matena,
  Zhou, Li, and Liu]{raffel2020exploring}
C.~Raffel, N.~Shazeer, A.~Roberts, K.~Lee, S.~Narang, M.~Matena, Y.~Zhou,
  W.~Li, and P.~J. Liu.
\newblock Exploring the limits of transfer learning with a unified text-to-text
  transformer.
\newblock \emph{Journal of Machine Learning Research}, 21\penalty0
  (140):\penalty0 1--67, 2020.

\bibitem[Saunshi et~al.(2021)Saunshi, Malladi, and Arora]{saunshi2021a}
N.~Saunshi, S.~Malladi, and S.~Arora.
\newblock A mathematical exploration of why language models help solve
  downstream tasks.
\newblock In \emph{International Conference on Learning Representations}, 2021.

\bibitem[Trauger and Tewari(2025)]{trauger2025nexttokenpredictionllmsend}
J.~Trauger and A.~Tewari.
\newblock On next-token prediction in {LLMs}: How end goals determine the
  consistency of decoding algorithms.
\newblock \emph{arXiv preprint arXiv:2505.11183}, 2025.

\bibitem[Tsybakov(2009)]{Tsybakov2009}
A.~B. Tsybakov.
\newblock \emph{Introduction to Nonparametric Estimation}.
\newblock Springer Series in Statistics. Springer, Dordrecht, 2009.

\bibitem[Wang(2023)]{wang2023self}
S.~Wang.
\newblock Self-supervised metric learning in multi-view data: A downstream task
  perspective.
\newblock \emph{Journal of the American Statistical Association}, 118\penalty0
  (544):\penalty0 2454--2467, 2023.

\bibitem[Wang(2025{\natexlab{a}})]{wang2025augmentation}
S.~Wang.
\newblock Augmentation invariant manifold learning.
\newblock \emph{Journal of the Royal Statistical Society Series B: Statistical
  Methodology}, 87\penalty0 (4):\penalty0 978--1000, 2025{\natexlab{a}}.

\bibitem[Wang(2025{\natexlab{b}})]{wang2025linear}
S.~Wang.
\newblock Linear separation capacity of self-supervised representation
  learning.
\newblock \emph{Journal of Machine Learning Research}, 26\penalty0
  (194):\penalty0 1--48, 2025{\natexlab{b}}.

\bibitem[Wei et~al.(2021)Wei, Xie, and Ma]{wei2021pretrained}
C.~Wei, S.~Xie, and T.~Ma.
\newblock Why do pretrained language models help in downstream tasks? an
  analysis of head and prompt tuning.
\newblock \emph{Advances in Neural Information Processing Systems},
  34:\penalty0 16158--16170, 2021.

\bibitem[Wu et~al.(2023)Wu, Lee, and Ge]{wu2023connecting}
C.~Wu, H.~Lee, and R.~Ge.
\newblock Connecting pre-trained language model and downstream task via
  properties of representation.
\newblock \emph{Advances in Neural Information Processing Systems},
  36:\penalty0 47216--47238, 2023.

\bibitem[Xiong et~al.(2024)Xiong, Dong, Ye, Wang, Zhong, Ji, Jiang, and
  Zhang]{xiong2024iterative}
W.~Xiong, H.~Dong, C.~Ye, Z.~Wang, H.~Zhong, H.~Ji, N.~Jiang, and T.~Zhang.
\newblock Iterative preference learning from human feedback: Bridging theory
  and practice for {RLHF} under {KL}-constraint.
\newblock In \emph{The Forty-first International Conference on Machine
  Learning}, 2024.

\bibitem[Yu et~al.(2015)Yu, Wang, and Samworth]{yu2015useful}
Y.~Yu, T.~Wang, and R.~J. Samworth.
\newblock A useful variant of the {Davis--Kahan} theorem for statisticians.
\newblock \emph{Biometrika}, 102\penalty0 (2):\penalty0 315--323, 2015.

\bibitem[Zbontar et~al.(2021)Zbontar, Jing, Misra, LeCun, and
  Deny]{zbontar2021barlow}
J.~Zbontar, L.~Jing, I.~Misra, Y.~LeCun, and S.~Deny.
\newblock {Barlow} {Twins}: Self-supervised learning via redundancy reduction.
\newblock In \emph{International Conference on Machine Learning}, pages
  12310--12320. PMLR, 2021.

\bibitem[Zhao et~al.(2025)Zhao, Ye, Gu, and Zhang]{zhao2025sharp}
H.~Zhao, C.~Ye, Q.~Gu, and T.~Zhang.
\newblock Sharp analysis for {KL}-regularized contextual bandits and {RLHF}.
\newblock In \emph{Advances in Neural Information Processing Systems}, 2025.

\bibitem[Zhao et~al.(2024)Zhao, Behnia, Vakilian, and
  Thrampoulidis]{zhao2024implicit}
Y.~Zhao, T.~Behnia, V.~Vakilian, and C.~Thrampoulidis.
\newblock Implicit geometry of next-token prediction: From language sparsity
  patterns to model representations.
\newblock In \emph{First Conference on Language Modeling}, 2024.

\end{thebibliography}

\begin{appendices}
\section{Proof}
\subsection{Proof of Theorem~\ref{thm:emdsimilarity}}
	Because $p_j(v|X)$ is estimated by $\hat{p}_E(v|\Theta_j(X))$, our estimator for $p_v(X_{j,c}|v)$ is 
	$$
	\hat{p}_{E,v}(X_{j,c}|v)={\hat{p}_E(v|\Theta_j(X))\over \hat{p}(v)}p(X_{j,c}),
	$$
	where $\hat{p}(v)=\EE_{X,j}(\hat{p}_E(v|\Theta_j(X)))$.
	In particular, if we assume $\Kcal(\Theta_j(X),E_v)=\exp\left(E_v^T\Theta_j(X)\right)$, then $\hat{p}_{E,v}(X_{j,c}|v)$ can be simplified as
	$$
	\hat{p}_{E,v}(X_{j,c}|v)=e^{E_v^T\Theta_j(X)-A(E_v)}\tilde{p}(X_{j,c}),
	$$
	where 
	$$
	A(E_v)=\log\sum_{X,j}e^{E_v^T\Theta_j(X)}\tilde{p}(X_{j,c})\quad {\rm and}\quad \tilde{p}(X_{j,c})={p(X_{j,c})\over \sum_{v\in\Vcal}e^{E_v^T\Theta_j(X)}}.
	$$
	So $\hat{p}_{E_v}(X_{j,c}|v)$ is an exponential family of probability distributions \citep{efron2023exponential}. The Hellinger distance between probability distributions in an exponential family is
	$$
	H^2(\hat{p}_{E,v_1}(X_{j,c}|v_1),\hat{p}_{E,v_2}(X_{j,c}|v_2))=1-\exp\left(A\left({E_{v_1}+E_{v_2}\over 2}\right)-{A(E_{v_1})+A(E_{v_2})\over 2}\right).
	$$
	Taylor’s theorem suggests
	\begin{align*}
		A(E_{v_1})=&A\left({E_{v_1}+E_{v_2}\over 2}\right)+\nabla A\left({E_{v_1}+E_{v_2}\over 2}\right)^T\left({E_{v_1}-E_{v_2}\over 2}\right)\\
		&+\int_0^1 (1-t)\left({E_{v_1}-E_{v_2}\over 2}\right)^T\nabla^2 A\left({E_{v_1}+E_{v_2}\over 2}+t\left({E_{v_1}-E_{v_2}\over 2}\right)\right)\left({E_{v_1}-E_{v_2}\over 2}\right)dt
	\end{align*}
	and
	\begin{align*}
		A(E_{v_2})=&A\left({E_{v_1}+E_{v_2}\over 2}\right)+\nabla A\left({E_{v_1}+E_{v_2}\over 2}\right)^T\left({E_{v_2}-E_{v_1}\over 2}\right)\\
		&+\int_0^1 (1-t)\left({E_{v_2}-E_{v_1}\over 2}\right)^T\nabla^2 A\left({E_{v_1}+E_{v_2}\over 2}+t\left({E_{v_2}-E_{v_1}\over 2}\right)\right)\left({E_{v_2}-E_{v_1}\over 2}\right)dt,
	\end{align*}
	so we have
	\begin{align*}
		&{A(E_{v_1})+A(E_{v_2})\over 2}-A\left({E_{v_1}+E_{v_2}\over 2}\right)\\
		=&{1\over 2}\int_0^1 (1-t)\left({E_{v_1}-E_{v_2}\over 2}\right)^T\nabla^2 A\left({E_{v_1}+E_{v_2}\over 2}+t\left({E_{v_1}-E_{v_2}\over 2}\right)\right)\left({E_{v_1}-E_{v_2}\over 2}\right)dt\\
		&+{1\over 2}\int_0^1 (1-t)\left({E_{v_2}-E_{v_1}\over 2}\right)^T\nabla^2 A\left({E_{v_1}+E_{v_2}\over 2}+t\left({E_{v_2}-E_{v_1}\over 2}\right)\right)\left({E_{v_2}-E_{v_1}\over 2}\right)dt.
	\end{align*}
	The assumption suggests for any vector $e$, we have
	$$
	b_\Theta^2\|e\|^2\le e^T\nabla^2 A(E_v)e\le B_\Theta^2\|e\|^2.
	$$
	Thus, we obtain
	$$
	{b_\Theta^2\over 8} \left\|E_{v_2}-E_{v_1}\right\|^2\le {A(E_{v_1})+A(E_{v_2})\over 2}-A\left({E_{v_1}+E_{v_2}\over 2}\right)\le {B_\Theta^2\over 8} \left\|E_{v_2}-E_{v_1}\right\|^2.
	$$
	This leads to
	$$
	1-\exp\left(-{b_\Theta^2\over 8} \left\|E_{v_2}-E_{v_1}\right\|^2\right)\le H^2(\hat{p}_{E, v_1}(X_{j,c}|v_1),\hat{p}_{E, v_2}(X_{j,c}|v_2))
	$$
	and 
	$$
	H^2(\hat{p}_{E, v_1}(X_{j,c}|v_1),\hat{p}_{E, v_2}(X_{j,c}|v_2))\le 1-\exp\left(-{B_\Theta^2\over 8} \left\|E_{v_2}-E_{v_1}\right\|^2\right).
	$$
	
	In addition, 
	$\EE_{X,j}\left[L\left(p_j(X);\hat{p}_{E^\ast}(\Theta^\ast_j(X))\right)\right]\le \epsilon_p$ suggests 
	$$
	\sum_{v\in\Vcal}p(v)\log{p(v)\over \hat{p}(v)}+\sum_{v\in\Vcal}p(v)L(p_v(X_{j,c}|v);\hat{p}_{E^\ast,v}(X_{j,c}|v))\le \epsilon_p.
	$$
	Since $p(v_1),p(v_2)\ge \min\{p(v_1),p(v_2)\}$, we have 
	$$
	2\sum_{v\in \{v_1,v_2\}} H^2(p_v(X_{j,c}|v),\hat{p}_{E^\ast,v}(X_{j,c}|v)) \le \sum_{v\in \{v_1,v_2\}}L(p_v(X_{j,c}|v);\hat{p}_{E^\ast,v}(X_{j,c}|v))\le {\epsilon_p\over \min\{p(v_1),p(v_2)\}}.
	$$
	Thus,
	$$
	\sum_{v\in \{v_1,v_2\}} H(p_v(X_{j,c}|v),\hat{p}_{E^\ast,v}(X_{j,c}|v))\le \sqrt{2\sum_{v\in \{v_1,v_2\}} H^2(p_v(X_{j,c}|v),\hat{p}_{E^\ast,v}(X_{j,c}|v))}\le \sqrt{\epsilon_p\over \min\{p(v_1),p(v_2)\}}.
	$$
	Therefore, 
	$$
	H(p_{v_1}(X_{j,c}|v_1),p_{v_2}(X_{j,c}|v_2))\le \sqrt{1-\exp\left(-{B_\Theta^2\over 8} \left\|E^\ast_{v_2}-E^\ast_{v_1}\right\|^2\right)}+\sqrt{\epsilon_p\over \min\{p(v_1),p(v_2)\}},
	$$
	and 
	$$
	H(p_{v_1}(X_{j,c}|v_1),p_{v_2}(X_{j,c}|v_2))\ge \left(\sqrt{1-\exp\left(-{b_\Theta^2\over 8} \left\|E^\ast_{v_2}-E^\ast_{v_1}\right\|^2\right)}-\sqrt{\epsilon_p\over \min\{p(v_1),p(v_2)\}}\right)_+.
	$$

\subsection{Proof of Theorem~\ref{thm:aprx}}
	For any given initial representation $\Psi^0\in \Rcal^0$, we write
	$$
	\Psi^{l,g,\Psi^0}=\underbrace{\Tcal_{g}(\ldots \Tcal_{g}(\Tcal_{g}}_{l\ {\rm  times}}(\Psi^0))),\qquad g\in\Gcal, l=1,\ldots,L.
	$$ 
	Suppose $g^\ast_\Theta$ is the following function
	$$
	g^\ast_\Theta=\argmin_{g\in\Gcal}\max_{\Psi}D(\Tcal_{g_\Theta}(\Psi),\Tcal_{g}(\Psi)).
	$$
	Since $D$ is metric, we have
	$$
	D(\Psi^{l,g^\ast_\Theta,\Psi^0},\Theta)\le D(\Psi^{l,g^\ast_\Theta,\Psi^0},\Tcal_{g_\Theta}(\Psi^{l-1,g^\ast_\Theta,\Psi^0}))+D(\Tcal_{g_\Theta}(\Psi^{l-1,g^\ast_\Theta,\Psi^0}),\Theta).
	$$
	The choice of $g^\ast_\Theta$ suggests 
	$$
	D(\Psi^{l,g^\ast_\Theta,\Psi^0},\Tcal_{g_\Theta}(\Psi^{l-1,g^\ast_\Theta,\Psi^0}))=D(\Tcal_{g_\Theta^\ast}(\Psi^{l-1,g^\ast_\Theta,\Psi^0}),\Tcal_{g_\Theta}(\Psi^{l-1,g^\ast_\Theta,\Psi^0}))\le \delta_\Theta.
	$$
	Because $\Tcal_{g_\Theta}$ is a contraction and $\Theta$ is a self-consistent representation map, we have 
	$$
	D(\Tcal_{g_\Theta}(\Psi^{l-1,g^\ast_\Theta,\Psi^0}),\Theta)=D(\Tcal_{g_\Theta}(\Psi^{l-1,g^\ast_\Theta,\Psi^0}),\Tcal_{g_\Theta}(\Theta))\le \kappa_\Theta D(\Psi^{l-1,g^\ast_\Theta,\Psi^0},\Theta).
	$$
	Putting above three inequalities together yields
	$$
	D(\Psi^{l,g^\ast_\Theta,\Psi^0},\Theta)\le \kappa_\Theta D(\Psi^{l-1,g^\ast_\Theta,\Psi^0},\Theta)+\delta_\Theta.
	$$
	Applying this inequality $L$ times, we obtain
	\begin{align*}
		D(\Psi^{L,g^\ast_\Theta,\Psi^0},\Theta)&\le \kappa_\Theta D(\Psi^{L-1,g^\ast_\Theta,\Psi^0},\Theta)+\delta_\Theta\\
		&\le \kappa_\Theta\left(\kappa_\Theta D(\Psi^{L-2,g^\ast_\Theta,\Psi^0},\Theta)+\delta_\Theta\right)+\delta_\Theta\\
		&\le \kappa_\Theta^L D(\Psi^{0},\Theta)+(\kappa_\Theta^{L-1}+\ldots+1)\delta_\Theta\\
		&\le \kappa_\Theta^L D(\Psi^{0},\Theta)+{1-\kappa_\Theta^L \over 1-\kappa_\Theta}\delta_\Theta.
	\end{align*}
	We can complete the proof when we take minimum over all possible  initial representation $\Psi^0\in \Rcal^0$.

\subsection{Proof of Theorem~\ref{thm:oracle}}
	Suppose $\Theta^\ast$ is the representation such that there exists an approximator $
	\hat{\Theta}\in  \Rcal^L(\Gcal, \Rcal^0)$.
	By the definition of Kullback-Leibler divergence, we have
	\begin{align*}
		L\left(p_j(X);\hat{p}_{E^\ast}(\Theta^\ast_j(X))\right)&=\sum_{v\in\Scal_j(X)}p_j(v|X)\log\left(p_j(v|X)/\hat{p}_{E^\ast}(v|\Theta^\ast_j(X))\right)\\
		&=\sum_{v\in\Scal_j(X)}p_j(v|X)\left[\log\left(p_j(v|X)\right)-\log\left({\Kcal(\Theta^\ast_j(X),E^\ast_v)\over \sum_{v'\in\Vcal}\Kcal(\Theta^\ast_j(X),E^\ast_{v'})}\right)\right].
	\end{align*}
	Similarly, if $\tilde{\Theta}$ is an approximator of $\Theta^\ast$, we can also have 
	$$
	L\left(p_j(X);\hat{p}_{E^\ast}(\tilde{\Theta}_j(X))\right)=\sum_{v\in\Scal_j(X)}p_j(v|X)\left[\log\left(p_j(v|X)\right)-\log\left({\Kcal(\tilde{\Theta}_j(X),E^\ast_v)\over \sum_{v'\in\Vcal}\Kcal(\tilde{\Theta}_j(X),E^\ast_{v'})}\right)\right].
	$$
	Therefore, we have
	\begin{align*}
		&L\left(p_j(X);\hat{p}_{E^\ast}(\tilde{\Theta}_j(X))\right)-L\left(p_j(X);\hat{p}_{E^\ast}(\Theta^\ast_j(X))\right)\\
		=&\sum_{v\in\Scal_j(X)}p_j(v|X)\left(\log\left({\Kcal(\Theta^\ast_j(X),E^\ast_v)\over \sum_{v'\in\Vcal}\Kcal(\Theta^\ast_j(X),E^\ast_{v'})}\right)-\log\left({\Kcal(\tilde{\Theta}_j(X),E^\ast_v)\over \sum_{v'\in\Vcal}\Kcal(\tilde{\Theta}_j(X),E^\ast_{v'})}\right)\right)\\
		=&\sum_{v\in\Scal_j(X)}p_j(v|X)\log\left({\Kcal(\Theta^\ast_j(X),E^\ast_v)\over \Kcal(\tilde{\Theta}_j(X),E^\ast_v)}\right)-\log\left({\sum_{v'\in\Vcal}\Kcal(\Theta^\ast_j(X),E^\ast_{v'})\over \sum_{v'\in\Vcal}\Kcal(\tilde{\Theta}_j(X),E^\ast_{v'})}\right)\\
		\le& \max_{v\in\Scal_j(X)}\log\left({\Kcal(\Theta^\ast_j(X),E^\ast_v)\over \Kcal(\tilde{\Theta}_j(X),E^\ast_v)}\right)-\log\left({\sum_{v'\in\Vcal}\Kcal(\Theta^\ast_j(X),E^\ast_{v'})\over \sum_{v'\in\Vcal}\Kcal(\tilde{\Theta}_j(X),E^\ast_{v'})}\right).
	\end{align*}
	Since
	$$
	\log\left({\sum_{v'\in\Vcal}\Kcal(\Theta^\ast_j(X),E^\ast_{v'})\over \sum_{v'\in\Vcal}\Kcal(\tilde{\Theta}_j(X),E^\ast_{v'})}\right)\ge \log\left(\min_{v'\in\Vcal}{\Kcal(\Theta^\ast_j(X),E^\ast_{v'})\over \Kcal(\tilde{\Theta}_j(X),E^\ast_{v'})}\right),
	$$
	we obtain 
	\begin{align*}
		&L\left(p_j(X);\hat{p}_{E^\ast}(\tilde{\Theta}_j(X))\right)-L\left(p_j(X);\hat{p}_{E^\ast}(\Theta^\ast_j(X))\right)\\
		\le& \max_{v\in\Scal_j(X)}\log\left({\Kcal(\Theta^\ast_j(X),E^\ast_v)\over \Kcal(\tilde{\Theta}_j(X),E^\ast_v)}\right)-\min_{v'\in\Vcal}\log\left({\Kcal(\Theta^\ast_j(X),E^\ast_{v'})\over \Kcal(\tilde{\Theta}_j(X),E^\ast_{v'})}\right).
	\end{align*}
	Because 
	\begin{align*}
		&\log\Kcal(\tilde{\Theta}_j(X),E^\ast_v)-\log\Kcal(\Theta^\ast_j(X),E^\ast_v)\\
		=&\left(\tilde{\Theta}_j(X)-\Theta^\ast_j(X)\right)^TE^\ast_v,
	\end{align*}
	we can know 
	\begin{align*}
		&L\left(p_j(X);\hat{p}_{E^\ast}(\tilde{\Theta}_j(X))\right)-L\left(p_j(X);\hat{p}_{E^\ast}(\Theta^\ast_j(X))\right)\\
		\le& \max_{v\in\Scal_j(X),v'\in\Vcal}\left(\tilde{\Theta}_j(X)-\Theta^\ast_j(X)\right)^T\left(E^\ast_v-E^\ast_{v'}\right)\\
		\le & \left\|\tilde{\Theta}_j(X)-\Theta^\ast_j(X)\right\|\max_{v\in\Scal_j(X),v'\in\Vcal}\left\|E^\ast_v-E^\ast_{v'}\right\|
	\end{align*}
	Since
	\begin{align*}
		&\EE_{X,j}\left[\left\|\tilde{\Theta}_j(X)-\Theta^\ast_j(X)\right\|\max_{v\in\Scal_j(X),v'\in\Vcal}\left\|E^\ast_v-E^\ast_{v'}\right\|\right]\\
		\le& \sqrt{\EE_{X,j}\left[\max_{v\in\Scal_j(X),v'\in\Vcal}\left\|E^\ast_v-E^\ast_{v'}\right\|^2\right]\EE_{X,j}\left[\left\|\tilde{\Theta}_j(X)-\Theta^\ast_j(X)\right\|^2\right]},
	\end{align*}
	we define the approximation error as
	$$
	\Acal(E^\ast,\Theta^\ast)=\sqrt{\EE\left[\max_{v\in\Scal_j(X),v'\in\Vcal}\left\|E^\ast_v-E^\ast_{v'}\right\|^2\right]\inf_{\tilde{\Theta}\in \Rcal^L(\Gcal, \Rcal^0) }\EE\left[\left\|\tilde{\Theta}_j(X)-\Theta^\ast_j(X)\right\|^2\right]}.
	$$
	Therefore, we have 
	\begin{align*}
		&\min_{\Theta\in \Rcal^L(\Gcal, \Rcal^0), E} \EE_{X,j}\left[L\left(p_j(X);\hat{p}_{E}(\Theta_j(X))\right)\right]\\
		\le& \min_{\Theta^\ast,E^\ast}\left[\EE_{X,j}\left[L\left(p_j(X);\hat{p}_{E^\ast}(\Theta^\ast_j(X))\right)\right]+\Acal(E^\ast,\Theta^\ast)\right] 
	\end{align*}

\subsection{Proof of Theorem~\ref{thm:generative}}
	We first work on $\rho_{p}$. Since we have 
	$$
	\sum_{v\notin \Scal_j(X_{<j})}q^{\rm pre}_j(v|X_{<j})\le -\log\left(\sum_{v\in \Scal_j(X_{<j})}q^{\rm pre}_j(v|X_{<j})\right)\le L(p_j(X_{<j});q_j^{\rm pre}(X_{<j})),
	$$
	we show
	$$
	\EE_{X,j}\left[\sum_{v\notin \Scal_j(X_{<j})}q^{\rm pre}_j(v|X_{<j})\right]\le \EE_{X,j}\left[L(p_j(X_{<j});q_j^{\rm pre}(X_{<j}))\right]\le \epsilon_p.
	$$
	This immediately indicates
	$$
	\EE_{Y\sim\QQ^{\rm pre}}\left[\bI(y_j\notin \Scal_j(Y_{<j}))\right]=\EE_{Y\sim\QQ^{\rm pre}}\left[\sum_{v\notin \Scal_j(Y_{<j})}q^{\rm pre}_j(v|Y_{<j})\right]\le C_{\QQ^{\rm pre}}\epsilon_p.
	$$
	Therefore, we can know
	$$
	\rho_{p}(q^{\rm pre},p)\ge 1-C_{\QQ^{\rm pre}}\epsilon_p.
	$$
	
	For any reference distribution $\pi$, we have a closed form of $q^{\hat{r},\pi,\lambda}_j(Y_{<j})$:
	$$
	q^{\hat{r},\pi,\lambda}_j(v|Y_{<j})={\pi_{j}(v|Y_{<j})\exp(\hat{r}(v|Y_{<j})/\lambda)\over \sum_{v'\in\Vcal}\pi_{j}(v'|Y_{<j})\exp(\hat{r}(v'|Y_{<j})/\lambda)}.
	$$
	Because the reward function is consistent with the plausible set $\Scal_j(X_{<j})$, we show
	\begin{align*}
		&\min_{v\in \Scal_j(Y_{<j})}\hat{r}(v|Y_{<j})-\max_{v\notin \Scal_j(Y_{<j})}\hat{r}(v|Y_{<j})\\
		\ge & \min_{v\in \Scal_j(Y_{<j})}r^\ast(v|Y_{<j})-\max_{v\notin \Scal_j(Y_{<j})}r^\ast(v|Y_{<j})-2\epsilon_r\\
		\ge &-2\epsilon_r.
	\end{align*}
	Since 
	$$
	\sum_{v\in \Scal_j(Y_{<j})}q^{\rm pre}_{j}(v|Y_{<j})\exp\left(\hat{r}(v|Y_{<j})\over \lambda\right)\ge \exp\left(\min_{v\in \Scal_j(Y_{<j})}\hat{r}(v|Y_{<j})\over \lambda\right)\sum_{v\in \Scal_j(Y_{<j})}q^{\rm pre}_{j}(v|Y_{<j})
	$$
	and 
	$$
	\sum_{v\notin \Scal_j(Y_{<j})}q^{\rm pre}_{j}(v|Y_{<j})\exp\left(\hat{r}(v|Y_{<j})\over \lambda\right)\le \exp\left(\max_{v\notin \Scal_j(Y_{<j})}\hat{r}(v|Y_{<j})\over \lambda\right)\sum_{v\notin \Scal_j(Y_{<j})}q^{\rm pre}_{j}(v|Y_{<j}),
	$$
	we have 
	\begin{align*}
		{\sum_{v\in \Scal_j(Y_{<j})}q^{\rm post}_j(v|Y_{<j})\over \sum_{v\notin\Scal_j(Y_{<j})}q^{\rm post}_j(v|Y_{<j})}\ge \exp\left( -{2\epsilon_r\over \lambda}\right) {\sum_{v\in \Scal_j(Y_{<j})}q^{\rm pre}_j(v|Y_{<j})\over \sum_{v\notin\Scal_j(Y_{<j})}q^{\rm pre}_j(v|Y_{<j})}.
	\end{align*}
	Therefore, we show
	$$
	\sum_{v\notin\Scal_j(Y_{<j})}q^{\rm post}_j(v|Y_{<j})\le \exp\left( {2\epsilon_r\over \lambda}\right) \sum_{v\notin\Scal_j(Y_{<j})}q^{\rm pre}_j(v|Y_{<j}).
	$$
	This leads to
	$$
	\EE_{Y\sim\QQ^{\rm post}}\left[\bI(y_j\notin \Scal_j(Y_{<j}))\right]=\EE_{Y\sim\QQ^{\rm post}}\left[\sum_{v\notin \Scal_j(Y_{<j})}q^{\rm post}_j(v|Y_{<j})\right]\le C_{\QQ^{\rm post}}\exp\left( {2\epsilon_r\over \lambda}\right)\epsilon_p.
	$$
	We conclude
	$$
	\rho_{p}(q^{\rm post},p)\ge 1-C_{\QQ^{\rm post}}\exp\left( {2\epsilon_r\over \lambda}\right)\epsilon_p.
	$$

	We next work on the upper bound of $\rho_{s}$. Given a reward function $r^\ast$ and $\delta\ge 0$, we define the $\delta$-optimal set of tokens given the context $Y_{<j}$ as
	$$
	\Acal_\delta(Y_{<j})=\left\{v\in \Vcal: \max_{v\in\Vcal}r^\ast(v|Y_{<j})-r^\ast(v|Y_{<j})\le \delta\right\}.
	$$
	When $\delta=0$, we also call $\Acal^\ast(Y_{<j})=\Acal_0(Y_{<j})$ the set of exactly optimal tokens given the context $Y_{<j}$. 
	
	Given the closed form of $q^{\hat{r},\pi,\lambda}_j(Y_{<j})$, we prove an upper bound of $q^{\hat{r},\pi,\lambda}_j(\Acal_\delta^c(Y_{<j})|Y_{<j})=\sum_{v\notin \Acal_\delta(Y_{<j}) }q^{\hat{r},\pi,\lambda}_j(v|Y_{<j})$. By the definition of $\Acal_\delta(Y_{<j})$, it follows that
	\begin{align*}
		&\min_{v\in \Acal_{\delta/2}(Y_{<j})}\hat{r}(v|Y_{<j})-\max_{v\notin \Acal_\delta(Y_{<j})}\hat{r}(v|Y_{<j})\\
		\ge & \min_{v\in \Acal_{\delta/2}(Y_{<j})}r^\ast(v|Y_{<j})-\max_{v\notin \Acal_\delta(Y_{<j})}r^\ast(v|Y_{<j})-2\epsilon_r\\
		\ge & \delta/2-2\epsilon_r
	\end{align*}
	Since 
	$$
	\sum_{v\in \Acal_{\delta/2}(Y_{<j})}\pi_{j}(v|Y_{<j})\exp\left(\hat{r}(v|Y_{<j})\over \lambda\right)\ge \exp\left(\min_{v\in \Acal_{\delta/2}(Y_{<j})}\hat{r}(v|Y_{<j})\over \lambda\right)\sum_{v\in \Acal_{\delta/2}(Y_{<j})}\pi_{j}(v|Y_{<j})
	$$
	and 
	$$
	\sum_{v\notin \Acal_\delta(Y_{<j})}\pi_{j}(v|Y_{<j})\exp\left(\hat{r}(v|Y_{<j})\over \lambda\right)\le \exp\left(\max_{v\notin \Acal_\delta(Y_{<j})}\hat{r}(v|Y_{<j})\over \lambda\right)\sum_{v\notin \Acal_\delta(Y_{<j})}\pi_{j}(v|Y_{<j}),
	$$
	we have 
	\begin{align*}
		{\sum_{v\in \Acal_{\delta/2}(Y_{<j})}q^{\hat{r},\pi,\lambda}_j(v|Y_{<j})\over \sum_{v\notin\Acal_\delta(Y_{<j})}q^{\hat{r},\pi,\lambda}_j(v|Y_{<j})}\ge \exp\left( \delta/2-2\epsilon_r\over \lambda\right) {\sum_{v\in \Acal_{\delta/2}(Y_{<j})}\pi_j(v|Y_{<j})\over \sum_{v\notin\Acal_\delta(Y_{<j})}\pi_j(v|Y_{<j})}.
	\end{align*}
	Therefore, 
	\begin{align*}
		q^{\hat{r},\pi,\lambda}_j(\Acal_\delta^c(Y_{<j})|Y_{<j})&\le \min\left\{1,\exp\left( -{\delta/2-2\epsilon_r\over \lambda}\right) {1\over \sum_{v\in \Acal_{\delta/2}(Y_{<j})}\pi_j(v|Y_{<j})}\right\}\\
		&\le \min\left\{1,\exp\left( -{\delta/2-2\epsilon_r\over \lambda}\right) {1\over \pi_j(\Acal^\ast(Y_{<j})|Y_{<j})}\right\},
	\end{align*}
	where $\pi_j(\Acal^\ast(Y_{<j})|Y_{<j})=\sum_{v\in \Acal^\ast(Y_{<j})}\pi_j(v|Y_{<j})$.
	
	Next, we show the upper bound for the regret of $q^{\hat{r},\pi,\lambda}$ for the given context $Y_{<j}$
	\begin{align*}
		&\max_{v\in\Vcal}r^\ast(v|Y_{<j})- \sum_{v\in\Vcal}q^{\hat{r},\pi,\lambda}_{j}(v|Y_{<j})r^\ast(v|Y_{<j})\\
		=&\sum_{v\in\Vcal}q^{\hat{r},\pi,\lambda}_{j}(v|Y_{<j})\left(\max_{v\in\Vcal}r^\ast(v|Y_{<j})-r^\ast(v|Y_{<j})\right)\\
		=&\sum_{v\in\Vcal}q^{\hat{r},\pi,\lambda}_{j}(v|Y_{<j})\int_0^R\bI\left(\max_{v\in\Vcal}r^\ast(v|Y_{<j})-r^\ast(v|Y_{<j})>\delta\right)d\delta\\
		=&\int_0^R\sum_{v\in\Vcal}q^{\hat{r},\pi,\lambda}_{j}(v|Y_{<j})\bI\left(\max_{v\in\Vcal}r^\ast(v|Y_{<j})-r^\ast(v|Y_{<j})>\delta\right)d\delta\\
		=&\int_0^R q^{\hat{r},\pi,\lambda}_j(\Acal_\delta^c(Y_{<j})|Y_{<j})d\delta\\
		\le& \int_0^R \min\left\{1,\exp\left( -{\delta/2-2\epsilon_r\over \lambda}\right) {1\over \pi_j(\Acal^\ast(Y_{<j})|Y_{<j})}\right\} d\delta\\
		=&\min\left\{R, 4 \epsilon_r+2\lambda\left(1-\log \pi_j(\Acal^\ast(Y_{<j})|Y_{<j})\right)\right\}.
	\end{align*}
	Therefore, we obtain
	$$
	\rho_s(q^{\hat{r},\pi,\lambda},r^\ast)\le \EE_{Y\sim\QQ}\left[\min\left\{R, 4 \epsilon_r+2\lambda\left(1-\log \pi_j(\Acal^\ast(Y_{<j})|Y_{<j})\right)\right\}\right].
	$$
	
	Let's consider two different choices of reference distribution $\pi$. The first reference distribution comes from pre-training, that is, $\pi=q^{\rm pre}$. Since $L(p,q)\ge 2\|p-q\|_1^2\ge 2\|p-q\|_\infty^2$, $\EE_{X,j}\left[L\left(p_j(X);\hat{p}_{\hat{E}}(\hat{\Theta}_j(X))\right)\right]\le \epsilon_p$ suggests
	$$
	\EE_{X,j}\left[\left\|p_j(X_{<j})-q^{\rm pre}_j(X_{<j})\right\|_\infty^2\right]\le {\epsilon_p\over 2}.
	$$
	Define the event $\Gcal$ as
	$$
	\Gcal=\left\{\left\|p_j(X_{<j})-q^{\rm pre}_j(X_{<j})\right\|_\infty\le {\tau\over 2}\right\}.
	$$
	On the event $\Gcal$, we have 
	$$
	q^{\rm pre}_j(v|X_{<j})\ge {\tau\over 2},\qquad \forall\ v\in \Scal_j(X_{<j}).
	$$
	Since the reward function is consistent with the plausible set $\Scal_j(X_{<j})$, we know $\Acal^\ast(X_{<j})\subset \Scal_j(X_{<j})$. This immediately indicates 
	$$
	q^{\rm pre}_j(\Acal^\ast(X_{<j})|X_{<j})\ge \sum_{v\in \Acal^\ast(X_{<j})}\left(p_j(v|X_{<j})-{\tau\over 2}\right)\ge {\tau\over 2}.
	$$
	Thus, on the event $\Gcal$, we have 
	$$
	\max_{v\in\Vcal}r^\ast(v|X_{<j})- \sum_{v\in\Vcal}q^{\hat{r},q^{\rm pre} ,\lambda}_{j}(v|X_{<j})r^\ast(v|X_{<j})\le 4 \epsilon_r+2\lambda\left(1-\log (\tau/2)\right).
	$$
	The probability of $\Gcal^c$ can be well controlled by Markov's inequality
	$$
	\PP_{X,j}(\Gcal^c)\le {\EE_{X,j}\left[\left\|p_j(X_{<j})-q^{\rm pre}_j(X_{<j})\right\|_\infty^2\right]\over \tau^2/4}\le {4\epsilon_p\over \tau^2}.
	$$
	This also suggest
	$$
	\PP_{Y\sim \QQ^{\rm post}}(\Gcal^c)\le{4C_{\QQ^{\rm post}}\epsilon_p\over \tau^2}.
	$$
	Therefore, 
	\begin{align*}
		\rho_s(q^{\rm post},r^\ast)=\rho_s(q^{\hat{r},q^{\rm pre},\lambda},r^\ast)&\le 4 \epsilon_r+2\lambda\left(1-\log \left(\tau\over 2\right)\right)+{4C_{\QQ^{\rm post}}\epsilon_p\over \tau^2}\\
		&=4 \epsilon_r+2\lambda\left(1+\log \left(2k\right)\right)+4C_{\QQ^{\rm post}}k^2\epsilon_p.
	\end{align*}
	
	The second reference distribution is a uniform distribution, i.e., $\pi(v)=1/d$. For the uniform distribution, we have
	$$
	\pi_j(\Acal^\ast(Y_{<j})|Y_{<j})\ge {1\over d}.
	$$
	Therefore, we have 
	$$
	\rho_s(q^{\hat{r},\pi,\lambda},r^\ast)\le 4 \epsilon_r+2\lambda\left(1+\log d\right).
	$$
	
	We now work on the lower bound of $\rho_s$. We can construct a combination of $r^\ast$, $\hat{r}$, $p_j(X)$, $q^{\rm pre}$ to show the lower bound of $\rho_s(q^{\rm post},r^\ast)$. Here, we choose $r^\ast$, $\hat{r}$, $p_j(X)$, $q^{\rm pre}$ are the same given any context $X_{<j}$. We name the token types as $v_1,\ldots,v_d$. For $r^\ast$, we set
	$$
	r^\ast(v_1)=R\quad {\rm and}\quad r^\ast(v_2)=\ldots=r^\ast(v_d)=R-B.
	$$
	where $B=\min\left\{R,2\epsilon_r+\lambda\log{1-e^{-k\epsilon_p}/k\over e^{-k\epsilon_p}/k}\right\}$.
	For $\hat{r}$, we choose
	$$
	\hat{r}(v_1)=R-\epsilon_r\quad {\rm and}\quad \hat{r}(v_2)=\ldots=\hat{r}(v_d)=R+\epsilon_r-B.
	$$
	For $p_j(X)$, we choose
	$$
	p_j(v_1|X)=\ldots=p_j(v_k|X)={1\over k}\quad {\rm and}\quad p_j(v_{k+1}|X)=\ldots=p_j(v_d|X)=0.
	$$
	For $q^{\rm pre}$, we choose
	$$
	q^{\rm pre}(v_1)={e^{-k\epsilon_p}\over k}, q^{\rm pre}(v_2)=\ldots=q^{\rm pre}(v_k)={1\over k}\quad {\rm and}\quad q^{\rm pre}(v_{k+1})=\ldots=q^{\rm pre}(v_d)={1-e^{-k\epsilon_p}\over k(d-k)}.
	$$
	Given $\hat{r}$ and $q^{\rm pre}$, we show $q^{\rm post}$
	$$
	q^{\rm post}(v_1)={e^{-k\epsilon_p}\over k}e^{R-\epsilon_r\over \lambda},\qquad  q^{\rm post}(v_2)=\ldots=q^{\rm post}(v_k)={1\over k}e^{R+\epsilon_r-B\over \lambda}
	$$
	and
	$$
	q^{\rm post}(v_{k+1})=\ldots=q^{\rm post}(v_d)={1-e^{-k\epsilon_p}\over k(d-k)}e^{R+\epsilon_r-B\over \lambda}.
	$$
	It follows that
	$$
	{\sum_{i=2}^dq^{\rm post}(v_i)\over q^{\rm post}(v_1)}={1-e^{-k\epsilon_p}/k\over e^{-k\epsilon_p}/k}e^{2\epsilon_r-B\over \lambda}\ge 1.
	$$
	Therefore, we have 
	$$
	\sum_{i=2}^dq^{\rm post}(v_i)\ge {1\over 2}.
	$$
	The suggests
	$$
	\max_{v\in\Vcal}r^\ast(v)- \sum_{v\in\Vcal}q^{\rm post}(v)r^\ast(v)\ge {B\over 2}.
	$$
	We actually show
	$$
	\rho_s(q^{\rm post},r^\ast)\ge {1\over 2}\min\left\{R,2\epsilon_r+\lambda\log{k-e^{-k\epsilon_p}\over e^{-k\epsilon_p}}\right\}\ge {1\over 2}\min\left\{R,2\epsilon_r+\lambda\log(k-1)\right\}.
	$$
	
	We next choose a combination of $r^\ast$ and $\hat{r}$ to show the lower bound of $\rho_s(q^{\hat{r},\pi,\lambda},r^\ast)$. Similarly, we choose $r^\ast$, $\hat{r}$ are the same given any context $X_{<j}$. We still name the token types as $v_1,\ldots,v_d$. For $r^\ast$, we set
	$$
	r^\ast(v_1)=R\quad {\rm and}\quad r^\ast(v_2)=\ldots=r^\ast(v_d)=R-B.
	$$
	where $B=\min\left\{R,2\epsilon_r+\lambda\log{1-1/d\over 1/d}\right\}$.
	For $\hat{r}$, we choose
	$$
	\hat{r}(v_1)=R-\epsilon_r\quad {\rm and}\quad \hat{r}(v_2)=\ldots=\hat{r}(v_d)=R+\epsilon_r-B.
	$$
	Given $\hat{r}$, we show $q^{\hat{r},\pi,\lambda}$
	$$
	q^{\hat{r},\pi,\lambda}(v_1)={1\over d}e^{R-\epsilon_r\over \lambda},\qquad  q^{\hat{r},\pi,\lambda}(v_2)=\ldots=q^{\hat{r},\pi,\lambda}(v_d)={1\over d}e^{R+\epsilon_r-B\over \lambda}.
	$$
	It follows that
	$$
	{\sum_{i=2}^dq^{\hat{r},\pi,\lambda}(v_i)\over q^{\hat{r},\pi,\lambda}(v_1)}={1-1/d\over 1/d}e^{2\epsilon_r-B\over \lambda}\ge 1.
	$$
	This leads to
	$$
	\sum_{i=2}^dq^{\hat{r},\pi,\lambda}(v_i)\ge {1\over 2}\quad {\rm and}\quad \max_{v\in\Vcal}r^\ast(v)- \sum_{v\in\Vcal}q^{\hat{r},\pi,\lambda}(v)r^\ast(v)\ge {B\over 2}
	$$
	Therefore, we show
	$$
	\rho_s(q^{\hat{r},\pi,\lambda},r^\ast)\ge  {1\over 2}\min\left\{R,2\epsilon_r+\lambda\log(d-1)\right\}.
	$$

\subsection{Proof of Theorem~\ref{thm:cluster}}
	We can follow the proof of Theorem~\ref{thm:emdsimilarity} to show 
	$$
	H^2(\hat{p}_{\hat{E},v_1}(X_{j,c}|v_1),\hat{p}_{\hat{E},v_2}(X_{j,c}|v_2))=1-\exp\left(A\left({\hat{E}_{v_1}+\hat{E}_{v_2}\over 2}\right)-{A(\hat{E}_{v_1})+A(\hat{E}_{v_2})\over 2}\right),
	$$
	where $A$ is defined in the proof of Theorem~\ref{thm:emdsimilarity}. An application of Taylor's theorem can yield
	\begin{align*}
		&{A(E_{v_1})+A(E_{v_2})\over 2}-A\left({E_{v_1}+E_{v_2}\over 2}\right)\\
		=&{1\over 2}\int_0^1 (1-t)\left({E_{v_1}-E_{v_2}\over 2}\right)^T\nabla^2 A\left({E_{v_1}+E_{v_2}\over 2}+t\left({E_{v_1}-E_{v_2}\over 2}\right)\right)\left({E_{v_1}-E_{v_2}\over 2}\right)dt\\
		&+{1\over 2}\int_0^1 (1-t)\left({E_{v_2}-E_{v_1}\over 2}\right)^T\nabla^2 A\left({E_{v_1}+E_{v_2}\over 2}+t\left({E_{v_2}-E_{v_1}\over 2}\right)\right)\left({E_{v_2}-E_{v_1}\over 2}\right)dt.
	\end{align*}
	Assumption~\ref{asp:cluster} suggests
	$$
	{A(E_{v_1})+A(E_{v_2})\over 2}-A\left({E_{v_1}+E_{v_2}\over 2}\right)\le {B_\Theta^2\over 8} \left\|E_{v_2}-E_{v_1}\right\|^2.
	$$
	Similarly, the Assumption~\ref{asp:cluster} suggests that for two token types in the same community $v_1,v_2\in \Ccal_k$,
	$$
	{A(E_{v_1})+A(E_{v_2})\over 2}-A\left({E_{v_1}+E_{v_2}\over 2}\right)\ge {b_\Theta^2\over 8} \left\|E_{v_2}-E_{v_1}\right\|^2.
	$$
	So we obtain
	$$
	H^2(\hat{p}_{\hat{E},v_1}(X_{j,c}|v_1),\hat{p}_{\hat{E},v_2}(X_{j,c}|v_2))\le 1-\exp\left(-{B_\Theta^2 \over 8}\left\|\hat{E}_{v_2}-\hat{E}_{v_1}\right\|^2\right)
	$$
	and, for two token types in the same community $v_1,v_2\in \Ccal_k$,
	$$
	H^2(\hat{p}_{\hat{E},v_1}(X_{j,c}|v_1),\hat{p}_{\hat{E},v_2}(X_{j,c}|v_2))\ge 1-\exp\left(-{b_\Theta^2 \over 8}\left\|\hat{E}_{v_2}-\hat{E}_{v_1}\right\|^2\right).
	$$
	
	By assumption $\EE_{X,j}\left[L\left(p_j(X);\hat{p}_{\hat{E}}(\hat{\Theta}_j(X))\right)\right]\le \epsilon_p$, we obtain
	$$
	\sum_{v\in\Vcal}p(v)\log{p(v)\over \hat{p}(v)}+\sum_{v\in\Vcal}p(v)L(p_v(X_{j,c}|v);\hat{p}_{\hat{E},\tilde{E}_v}(X_{j,c}|v))\le \epsilon_p.
	$$ 
	Because $p(v_1),p(v_2)\ge p_{\min}$, for any $v_1,v_2\in \Vcal$, we have 
	$$
	2\sum_{v\in \{v_1,v_2\}} H^2(p_v(X_{j,c}|v),\hat{p}_{\hat{E},\tilde{E}_v}(X_{j,c}|v)) \le \sum_{v\in \{v_1,v_2\}}L(p_v(X_{j,c}|v);\hat{p}_{\hat{E},\tilde{E}_v}(X_{j,c}|v))\le {\epsilon_p\over  p_{\min}},
	$$
	which leads to
	$$
	\sum_{v\in \{v_1,v_2\}} H(p_v(X_{j,c}|v),\hat{p}_{\hat{E},\tilde{E}_v}(X_{j,c}|v)) \le \sqrt{\epsilon_p\over  p_{\min}}.
	$$
	This immediately suggests, for two token types in the same community $v_1,v_2\in \Ccal_k$,
	$$
	1-\exp\left(-{b_\Theta^2 \over 8}\left\|\hat{E}_{v_2}-\hat{E}_{v_1}\right\|^2\right)\le H^2(\hat{p}_{\hat{E},\tilde{E}_{v_1}}(X_{j,c}|v_1),\hat{p}_{\hat{E},\tilde{E}_{v_2}}(X_{j,c}|v_2))\le \left(g_w+\sqrt{\epsilon_p\over  p_{\min}}\right)^2.
	$$
	This means, for two token types in the same community $v_1,v_2\in \Ccal_k$,
	$$
	\left\|\hat{E}_{v_2}-\hat{E}_{v_1}\right\|^2\le -{8\over b_\Theta^2}\log\left(1-\left(g_w+\sqrt{\epsilon_p\over  p_{\min}}\right)^2\right)
	$$
	and
	$$
	W_{v_1,v_2}\ge \alpha_\sigma=\left(1-\left(g_w+\sqrt{\epsilon_p\over  p_{\min}}\right)^2\right)^{4/\sigma^2b_\Theta^2 }.
	$$
	Similarly, for two token types in the different communities $v_1\in \Ccal_{k_1},v_2\in \Ccal_{k_2}$, we have 
	$$
	\left\|\hat{E}_{v_2}-\hat{E}_{v_1}\right\|^2\ge -{8\over B_\Theta^2 }\log\left(1-\left(g_b-\sqrt{\epsilon_p\over  p_{\min}}\right)^2\right)
	$$
	and 
	$$
	W_{v_1,v_2}\le \beta_\sigma=\left(1-\left(g_b-\sqrt{\epsilon_p\over  p_{\min}}\right)^2\right)^{4/\sigma^2B_\Theta^2}.
	$$
	
	We consider an ideal affinity matrix and its graph Laplacian matrix
	$$
	W_{v_1,v_2}^\ast=\begin{cases}
		W_{v_1,v_2},&\qquad c(v_1)=c(v_2)\\
		0,&\qquad c(v_1)\ne c(v_2)
	\end{cases}\qquad {\rm and}\qquad L^\ast=D^\ast-W^\ast,
	$$
	where $D^\ast$ is a diagonal matrix with entry $D^\ast_{v,v}=\sum_{v'\in \Vcal}W^\ast_{v',v}$. Because each block of $W^\ast$ defines a connected weighted graph, it is clear that $L^\ast$ has exactly $K$ zero eigenvalues and its null space is spanned by $U_{\cdot,k}^\ast=1_{\Ccal_k}/\sqrt{d_k}$, $1\le k\le K$, where $1_{\Ccal_k}$ is a $d$-dimensional vector where the entry is $1$ if $v\in \Ccal_k$ and $0$ otherwise. Since $W^\ast$ is a block diagonal matrix, the $(K+1)$st smallest eigenvalue of $L^\ast$ is larger than $\alpha_\sigma d_{\min}$. For $v_1\ne v_2$, we have 
	$$
	|W_{v_1,v_2}-W_{v_1,v_2}^\ast|\le \begin{cases}
		0,&\qquad c(v_1)=c(v_2)\\
		\beta_\sigma,&\qquad c(v_1)\ne c(v_2)
	\end{cases}.
	$$
	This immediately suggests 
	$$
	|D_{v,v}-D_{v,v}^\ast|\le (d-d_{\min})\beta_\sigma.
	$$
	Thus, we obtain
	$$
	\|L-L^\ast\|_2\le \sqrt{\|L-L^\ast\|_1\|L-L^\ast\|_\infty}\le 2(d-d_{\min})\beta_\sigma.
	$$
	An application of Davis–Kahan theorem \citep{yu2015useful} yields that there exist an orthogonal matrix $O\in \RR^{K\times K}$ such that
	$$
	\|U-U^\ast O\|_F\le {2\sqrt{2}\sqrt{K}\|L-L^\ast\|_2\over \alpha_\sigma d_{\min}}\le {4\sqrt{2}\sqrt{K}(d-d_{\min})\beta_\sigma\over \alpha_\sigma d_{\min}}.
	$$ 
	
	If we choose $\Vcal_k=\Ccal_k$ and $\mu_k=U_{v,\cdot}^\ast O$ for any $v\in \Ccal_k$, then 
	$$
	\sum_{k=1}^K\sum_{v\in \Vcal_k}\|U_{v,\cdot}-\mu_k\|^2=\|U-U^\ast O\|_F^2.
	$$
	The assumption in Assumption~\ref{asp:cluster} implies
	$$
	\sum_{v\in \Vcal}\|U_{v,\cdot}-\hat{\mu}_v\|^2\le (1+\kappa)^2\|U-U^\ast O\|_F^2,
	$$
	where $\hat{\mu}_v=\hat{\mu}_k$ if $v\in \hat{\Vcal}_k$.
	So we have
	\begin{align*}
		\sum_{v\in \Vcal}\|U_{v,\cdot}^\ast O-\hat{\mu}_v\|^2&\le \left(\sqrt{\sum_{v\in \Vcal}\|U_{v,\cdot}-\hat{\mu}_v\|^2}+\|U-U^\ast O\|_F\right)^2\\
		&\le (2+\kappa)^2\|U-U^\ast O\|_F^2.
	\end{align*}
	Define 
	$$
	\Bcal_k=\left\{v\in\Ccal_k: \|U_{v,\cdot}^\ast O-\hat{\mu}_v\|\ge {1\over 2}\sqrt{2\over d_{\max}} \right\}\qquad {\rm and}\qquad \Gcal_k=\Ccal_k\setminus\Bcal_k.
	$$
	With $\Bcal_k$, we have
	$$
	\sum_{v\in \Vcal}\|U_{v,\cdot}^\ast O-\hat{\mu}_v\|^2=\sum_{k=1}^K\sum_{v\in\Ccal_k}\|U_{v,\cdot}^\ast O-\hat{\mu}_v\|^2\ge \sum_{k=1}^K{|\Bcal_k|\over 2d_{\max}}.
	$$
	So it follows that
	$$
	\left|\cup_k \Bcal_k\right|\le 2d_{\max}(2+\kappa)^2\|U-U^\ast O\|_F^2.
	$$
	The assumption in Assumption~\ref{asp:cluster} suggests that $\left|\cup_k \Bcal_k\right|< d_{\min}$. 
	
	We next show that all token types in each $\Gcal_k$ have the same clustering label from $k$-means. Since $\left|\cup_k \Bcal_k\right|< d_{\min}$, it follows that $\Gcal_k\ne \emptyset$ for $k=1,\ldots,K$. In addition, suppose there exist $v_1\in \Gcal_{k_1}$ and $v_2\in \Gcal_{k_2}$ for $k_1\ne k_2$ such that $\hat{c}(v_1)=\hat{c}(v_2)$. Then, we have $\hat{\mu}_{v_1}=\hat{\mu}_{v_2}$ and
	$$
	\|U_{v_1,\cdot}^\ast O-U_{v_2,\cdot}^\ast O\|\le \|U_{v_1,\cdot}^\ast O-\hat{\mu}_{v_1}\|+\|U_{v_2,\cdot}^\ast O-\hat{\mu}_{v_2}\|< \sqrt{2\over d_{\max}}.
	$$
	On the other hand, because $v_1\in\Ccal_{k_1}$ and $v_2\in\Ccal_{k_2}$, we have
	$$
	\left\|U_{v_1,\cdot}^\ast O-U_{v_2,\cdot}^\ast O\right\|^2={1\over d_{k_1}}+{1\over d_{k_2}}\ge {2\over d_{\max}}.
	$$
	This is a contradiction, so $\hat{c}(v_1)\ne \hat{c}(v_2)$ for any $v_1\in \Gcal_{k_1}$ and $v_2\in \Gcal_{k_2}$ with $k_1\ne k_2$. In other words, the estimated clustering labels are distinct across different $\Gcal_k$. Since every $\Gcal_k$ is nonempty and the sets of estimated labels used by different $\Gcal_k$'s are disjoint, the $K$ sets $\Gcal_1,\ldots,\Gcal_K$ collectively use at least $K$ distinct labels. As the $k$-means solution has at most $K$ labels, each $\Gcal_k$ must use exactly one label, and these $K$ labels are distinct.
	Since there are $K$ true communities and at most $K$ estimated communities, these estimated labels define a permutation $\pi$ of the true labels. Therefore, the misclustered token types lie in $\cup_k \Bcal_k$ and we obtain 
	$$
	\Phi(\hat{c},c) \le {64(2+\kappa)^2Kd_{\max}\over d}\left({(d-d_{\min})\beta_\sigma\over \alpha_\sigma d_{\min}}\right)^2.
	$$

\subsection{Proof of Theorem~\ref{thm:classification}}
	We first show the conditional distribution of the context given the target token are close for the token types in the preferred set $\Wcal_k$. The conditional distribution of the context given the target token can be further decomposed
	$$
	p_{v}(X_{j,c}|v)=\sum_{k=1}^Kp_{k|v}p_{v,k}(X_{j,c})
	$$
	where $p_{k|v}=\PP_{X,j}(Y=k|X_j=v)$ and $p_{v,k}(X_{j,c})=p_{v,k}(X_{j,c}|x_j=v,Y=k)$. By the definition of Hellinger distance, we have
	\begin{align*}
		&H^2(p_{v_1}(X_{j,c}|v_1),p_{v_2}(X_{j,c}|v_2))\\
		=&1-\sum_{X,j}\sqrt{p_{v_1}(X_{j,c}|v_1)p_{v_2}(X_{j,c}|v_2)}\\
		=&1-\sum_{X,j}\sqrt{\left(\sum_{k=1}^Kp_{k|v_1}p_{v_1,k}(X_{j,c})\right)\left(\sum_{k=1}^Kp_{k|v_2}p_{v_2,k}(X_{j,c})\right)}\\
		\le& 1-\sum_{X,j}\sum_{k=1}^K\sqrt{p_{k|v_1}p_{v_1,k}(X_{j,c})p_{k|v_2}p_{v_2,k}(X_{j,c})}\\
		=& 1-\sum_{k=1}^K\sqrt{p_{k|v_1}p_{k|v_2}}\sum_{X,j}\sqrt{p_{v_1,k}(X_{j,c})p_{v_2,k}(X_{j,c})}\\
		=& 1-\sum_{k=1}^K\sqrt{p_{k|v_1}p_{k|v_2}}\left(1-H^2(p_{v_1,k}(X_{j,c}),p_{v_2,k}(X_{j,c}))\right)\\
		\le& 1-\sqrt{p_{k|v_1}p_{k|v_2}}\left(1-H^2(p_{v_1,k}(X_{j,c}),p_{v_2,k}(X_{j,c}))\right)
	\end{align*}
	We use Cauchy–Schwarz inequality in the third line. If $v_1,v_2\in \Wcal_k$, then it follows that
	$p_{k|v_1},p_{k|v_2}\ge \tau$ and $H^2(p_{v_1,k}(X_{j,c}),p_{v_2,k}(X_{j,c}))\le g_w^2$. Thus, for $v_1,v_2\in \Wcal_k$, we have
	$$
	\Gamma_h(v_1,v_2)=H^2(p_{v_1}(X_{j,c}|v_1),p_{v_2}(X_{j,c}|v_2))\le 1-\tau+\tau g_w^2.
	$$
	We can follow the same strategy in proof of Theorem~\ref{thm:cluster}, then we show that if $v_1,v_2\in \Wcal_k$, then 
	$$
	\left\|\hat{E}_{v_2}-\hat{E}_{v_1}\right\|^2\le D^2= -{8\over b_\Theta^2}\log\left(1-\left(\sqrt{1-\tau+\tau g_w^2}+\sqrt{\epsilon_p\over  p_{\min}}\right)^2\right).
	$$
	
	In the resulting model, our predicted probability is 
	$$
	q^{\hat{W},\hat{b},\hat{\Theta}}_k(X_{j,c})={\exp(\hat{W}_{k}^T\hat{\Theta}_{j}(X)+\hat{b}_k)\over \sum_{k'=1}^K\exp(\hat{W}_{k'}^T\hat{\Theta}_{j}(X)+\hat{b}_{k'})}.
	$$
	To study the performance of $h_{\hat{W},\hat{b},\hat{\Theta}}$, Lemma~\ref{lm:infty} suggests we only need to derive an upper bound of $\EE_{X,j}\|q^{\hat{W},\hat{b},\hat{\Theta}}-p^\ast\|_\infty^2$, where $p^\ast_k(X_{j,c})=\PP_{X,j}(Y=k|X_{j,c})$. By Pinsker’s inequality, we have 
	\begin{align*}
		\EE_{X,j}\|q^{\hat{W},\hat{b},\hat{\Theta}}-p^\ast\|_\infty^2&\le \EE_{X,j}\|q^{\hat{W},\hat{b},\hat{\Theta}}-p^\ast\|_1^2\\
		&\le 2\EE_{X,j}L(p^\ast;q^{\hat{W},\hat{b},\hat{\Theta}})\\
		&\le 2\left(-\EE_{X,j,Y}\log q^{\hat{W},\hat{b},\hat{\Theta}}_Y(X_{j,c})+\EE_{X,j,Y} \log p^\ast_Y(X_{j,c})\right).
	\end{align*}
	Therefore, the rest of proof is devoted to provide an upper bound of $-\EE_{X,j,Y}\log q^{\hat{W},\hat{b},\hat{\Theta}}_Y(X_{j,c})+\EE_{X,j,Y} \log p^\ast_Y(X_{j,c})$. To show such an upper bound, we adopt the following decomposition
	\begin{align*}
		&-\EE_{X,j,Y}\log q^{\hat{W},\hat{b},\hat{\Theta}}_Y(X_{j,c})+\EE_{X,j,Y}\log p^\ast_Y(X_{j,c})\\
		=&-\EE_{X,j,Y} \log q^{\hat{W},\hat{b},\hat{\Theta}}_Y(X_{j,c})+\EE_{X,j,Y}\log q^{W^r,b^r,\hat{\Theta}}_Y(X_{j,c})\\
		&-\EE_{X,j,Y}\log q^{W^r,b^r,\hat{\Theta}}_Y(X_{j,c})+\EE_{X,j,Y}\log p^\ast_Y(X_{j,c}),
	\end{align*}
	where $W^r$ and $b^r$ is the minimizer of penalized risk
	$$
	(W^r,b^r)=\argmin_{W,b:\sum_kW_k=0,\sum_kb_k=0}-\EE_{X,j,Y}\log\hat{P}_{Y}(X,j)+{1\over n}\left(\|W\|_F^2+\|b\|^2\right).
	$$
	An application of Lemma~\ref{lm:bias} and \ref{lm:variance} yields that there exit constant $C_1$ and $C_2$ such that with probability at least $1-e^{-\gamma}$,
	$$
	\EE_{X,j}\|q^{\hat{W},\hat{b},\hat{\Theta}}-p^\ast\|_\infty^2\le 2\left({C_{c,1}+C_{c,2}\gamma\over n}+ {2\epsilon_p\over \rho_0}+{D^4B_\Theta^4\over 16}+C_{c,3}\delta_s+{C_{c,4}\over n}\right).
	$$
	Therefore, if $n\ge n_0= 72e^{4\sqrt{2}B^\ast\sqrt{B_{\Theta}^2+1}}(B_\Theta^2+1)(4K^2+\gamma)K^2/(\sigma_{\min}^2{B^\ast}^2)$, then
	$$
	R(h_{\hat{W},\hat{b},\hat{\Theta}})\le 2^{(4\beta+7)/(\beta+2)}C_0^{1/(\beta+2)}\left({C_{c,1}+C_{c,2}\gamma\over n}+ {2\epsilon_p\over \rho_0}+{D^4B_\Theta^4\over 16}+C_{c,3}\delta_s+{C_{c,4}\over n}\right)^{(\beta+1)/(\beta+2)}
	$$
	with probability at least $1-e^{-\gamma}$.

\subsection{Supplemental Lemmas for Proof of Theorem~\ref{thm:classification}}

\begin{lemma}
	\label{lm:infty}
	Suppose Assumption~\ref{asp:classification} holds. 
	Let $p^\ast_k(X_{j,c})=\PP_{X,j}(Y=k|X_{j,c})$ and $q(X_{j,c})$ be any measurable probability vector such that $\sum_{k=1}^Kq_k(X_{j,c})=1$. If we define a classifier $h_q$ as $h_q(X,j)=\argmax_{1\le k\le K}q_k(X_{j,c})$. Then, we have 
	$$
	R(h_q)\le 8C_0^{1/(\beta+2)}\left(\EE_{X,j}\|q(X_{j,c})-p^\ast(X_{j,c})\|_\infty^2\right)^{(\beta+1)/(\beta+2)}.
	$$
\end{lemma}
\begin{proof}
	We follow the similar arguments in \cite{audibert2007fast}.
	We write $Y^q=h_q(X,j)$ and $Y^\ast=h^\ast(X,j)$. Clearly, we have
	\begin{align*}
		p^\ast_{Y^\ast}(X_{j,c})-p^\ast_{Y^q}(X_{j,c})&\le p^\ast_{Y^\ast}(X_{j,c})-q_{Y^\ast}(X_{j,c})+q_{Y^\ast}(X_{j,c})-q_{Y^q}(X_{j,c})+q_{Y^q}(X_{j,c})-p^\ast_{Y^q}(X_{j,c})\\
		&\le 2\|q(X_{j,c})-p^\ast(X_{j,c})\|_\infty.
	\end{align*}
	Then, we have 
	$$
	p^\ast_{k_{(1)}(X_{j,c})}(X_{j,c})-p^\ast_{k_{(2)}(X_{j,c})}(X_{j,c})\le p^\ast_{Y^\ast}(X_{j,c})-p^\ast_{Y^q}(X_{j,c})\le 2\|q(X_{j,c})-p^\ast(X_{j,c})\|_\infty.
	$$
	Therefore, 
	\begin{align*}
		R(h_q)=&\EE_{X,j}\left[\left(p^\ast_{Y^\ast}(X_{j,c})-p^\ast_{Y^q}(X_{j,c})\right)\bI(Y^\ast\ne Y^q)\right]\\
		=&\EE_{X,j}\left[\left(p^\ast_{Y^\ast}(X_{j,c})-p^\ast_{Y^q}(X_{j,c})\right)\bI(Y^\ast\ne Y^0)\bI(2\|q(X_{j,c})-p^\ast(X_{j,c})\|_\infty>t)\right]\\
		&+\EE_{X,j}\left[\left(p^\ast_{Y^\ast}(X_{j,c})-p^\ast_{Y^q}(X_{j,c})\right)\bI(Y^\ast\ne Y^0)\bI(2\|q(X_{j,c})-p^\ast(X_{j,c})\|_\infty\le t)\right]\\
		\le &t\PP_{X,j}\left(0<p^\ast_{k_{(1)}(X_{j,c})}(X_{j,c})-p^\ast_{k_{(2)}(X_{j,c})}(X_{j,c})\le t\right) +{4\EE_{X,j}(\|q(X_{j,c})-p^\ast(X_{j,c})\|_\infty^2)\over t}\\
		\le& C_0t^{\beta+1}+{4\EE_{X,j}(\|q(X_{j,c})-p^\ast(X_{j,c})\|_\infty^2)\over t}
	\end{align*}
	We can complete the proof if we choose $t=(4\EE_{X,j}(\|q(X_{j,c})-p^\ast(X_{j,c})\|_\infty^2)/C_0)^{1/(\beta+2)}$.
\end{proof}

\begin{lemma}
	\label{lm:variance}
	Suppose Assumption~\ref{asp:classification} holds. If $n\ge 72(B_\Theta^2+1)(4K^2+\gamma)/\mu^2{B^\ast}^2$, then, with probability at least $1-e^{-\gamma}$,
	$$
	-\EE_{X,j,Y}\log q^{\hat{W},\hat{b},\hat{\Theta}}_Y(X_{j,c})+\EE_{X,j,Y}\log q^{W^r,b^r,\hat{\Theta}}_Y(X_{j,c})\le {C_{c,1}+C_{c,2}\gamma\over n},
	$$
	where
	$$
	C_{c,1}=2{B^\ast}^2+{144(B_\Theta^2+1)^2K^4\over \sigma_{\min}^2e^{-4\sqrt{2}B^\ast\sqrt{B_{\Theta}^2+1}}}\qquad {\rm and}\qquad C_{c,2}={36(B_\Theta^2+1)^2K^2\over \sigma_{\min}^2e^{-4\sqrt{2}B^\ast\sqrt{B_{\Theta}^2+1}}}.
	$$
\end{lemma}
\begin{proof}
	In this proof, we write 
	$$
	f(W,b)=-\EE_{X,j,Y}\log q^{W,b,\hat{\Theta}}_Y(X_{j,c})\qquad {\rm and}\qquad \hat{f}(W,b)=-{1\over n}\sum_{i=1}^n\log q^{W,b,\hat{\Theta}}_{Y_i}(X_{i,j,c}).
	$$
	Similarly, we also write 
	$$
	f^p(W,b)=f(W,b)+{1\over n}\left(\|W\|_F^2+\|b\|^2\right)\quad {\rm and}\quad \hat{f}^p(W,b)=\hat{f}(W,b)+{1\over n}\left(\|W\|_F^2+\|b\|^2\right).
	$$
	Because $f(W,b)$ and $\hat{f}(W,b)$ are convex in $(W,b)$, it is clear that $f^p(W,b)$ and $\hat{f}^p(W,b)$ are $1/n$-strongly convex in $(W,b)$. Thus, $W^r$ and $b^r$ is the unique minimizer of $f^p(W,b)$, $\hat{W}$ and $\hat{b}$ is the unique minimizer of $\hat{f}^p(W,b)$, and $W^\ast$ and $b^\ast$ is the minimizer of $f(W,b)$. Define $\Bcal=\{(W,b):\sum_kW_k=0,\sum_kb_k=0, \|W\|_F^2+\|b\|^2\le 4{B^\ast}^2\}$, where $B^\ast$ is defined in Assumption~\ref{asp:classification}. It is clear that $(W^\ast,b^\ast), (W^r, b^r)\in \Bcal$ because 
	\begin{align*}
		f(W^r,b^r)+{1\over n}\left(\|W^r\|_F^2+\|b^r\|^2\right)&\le f(W^\ast,b^\ast)+{1\over n}\left(\|W^\ast\|_F^2+\|b^\ast\|^2\right)\\
		&\le f(W^r,b^r)+{1\over n}\left(\|W^\ast\|_F^2+\|b^\ast\|^2\right)
	\end{align*}
	leads to 
	$$
	\|W^r\|_F^2+\|b^r\|^2\le \|W^\ast\|_F^2+\|b^\ast\|^2\le {B^\ast}^2.
	$$ 
	
	Next, we show $f(W,b)$ is a $M$-smooth function and a $\mu$-strongly convex function on $\Bcal$, where $M=B_{\Theta}^2+1$ and $\mu=\sigma_{\min}e^{-2\sqrt{2}B^\ast\sqrt{B_{\Theta}^2+1}}/K$. The calculation suggests that for any $\tilde{W}$ and $\tilde{b}$ such that $\sum_k\tilde{W}_k=0$ and $\sum_kb_k=0$, we have
	\begin{align*}
		&\mathrm{vec}(\tilde{W},\tilde{b})^T\nabla^2 f(W,b)\mathrm{vec}(\tilde{W},\tilde{b})\\
		=&\EE_{X,j}\left[\sum_{k=1}^Kq^{W,b,\hat{\Theta}}_k(\tilde{W}_{k}^T\hat{\Theta}_{j}(X)+\tilde{b}_k)^2-\left(\sum_{k=1}^Kq^{W,b,\hat{\Theta}}_k(\tilde{W}_{k}^T\hat{\Theta}_{j}(X)+\tilde{b}_k)\right)^2\right].
	\end{align*}
	Since
	$$
	|\tilde{W}_{k}^T\hat{\Theta}_{j}(X)+\tilde{b}_k|\le \sqrt{\|\tilde{W}_{k}\|^2+\tilde{b}_k^2}\sqrt{\|\hat{\Theta}_{j}(X)\|^2+1}\le \sqrt{B_{\Theta}^2+1}\sqrt{\|\tilde{W}_{k}\|^2+\tilde{b}_k^2},
	$$
	we obtain 
	$$
	\mathrm{vec}(\tilde{W},\tilde{b})^T\nabla^2 f(W,b)\mathrm{vec}(\tilde{W},\tilde{b})\le \EE_{X,j}\left[\sum_{k=1}^K(\tilde{W}_{k}^T\hat{\Theta}_{j}(X)+\tilde{b}_k)^2\right]\le (B_{\Theta}^2+1)(\|\tilde{W}\|_F^2+\|b\|^2).
	$$
	So we can know $f(W,b)$ is a $(B_{\Theta}^2+1)$-smooth function. On the other hand, we have
	\begin{align*}
		&\sum_{k=1}^Kq^{W,b,\hat{\Theta}}_k(\tilde{W}_{k}^T\hat{\Theta}_{j}(X)+\tilde{b}_k)^2-\left(\sum_{k=1}^Kq^{W,b,\hat{\Theta}}_k(\tilde{W}_{k}^T\hat{\Theta}_{j}(X)+\tilde{b}_k)\right)^2\\
		= & \min_a \sum_{k=1}^Kq^{W,b,\hat{\Theta}}_k\left((\tilde{W}_{k}^T\hat{\Theta}_{j}(X)+\tilde{b}_k)-a\right)^2\\
		\ge & \min_kq^{W,b,\hat{\Theta}}_k\min_a \sum_{k=1}^K\left((\tilde{W}_{k}^T\hat{\Theta}_{j}(X)+\tilde{b}_k)-a\right)^2\\
		=& \min_kq^{W,b,\hat{\Theta}}_k \sum_{k=1}^K(\tilde{W}_{k}^T\hat{\Theta}_{j}(X)+\tilde{b}_k)^2.
	\end{align*}
	Because
	$$
	\min_kq^{W,b,\hat{\Theta}}_k=\min_k {\exp(W_{k}^T\hat{\Theta}_{j}(X)+b_k)\over \sum_{k'=1}^K\exp(W_{k'}^T\hat{\Theta}_{j}(X)+b_{k'})}\ge {1\over K}{\exp\left(\min_k W_{k}^T\hat{\Theta}_{j}(X)+b_k \right)\over \exp\left(\max_k W_{k}^T\hat{\Theta}_{j}(X)+b_k \right)}
	$$
	and $(W,b)\in \Bcal$ suggests
	$$
	\max_{k,k'} \left(W_{k}^T\hat{\Theta}_{j}(X)+b_k-W_{k'}^T\hat{\Theta}_{j}(X)-b_{k'}\right)\le \sqrt{B_{\Theta}^2+1} \sqrt{2\|W\|_F^2+2\|b\|^2}\le 2\sqrt{2}B^\ast\sqrt{B_{\Theta}^2+1},
	$$
	we obtain 
	$$
	\min_kq^{W,b,\hat{\Theta}}_k\ge {1\over K}\exp\left(-2\sqrt{2}B^\ast\sqrt{B_{\Theta}^2+1}\right).
	$$
	This immediately suggests
	$$
	\mathrm{vec}(\tilde{W},\tilde{b})^T\nabla^2 f(W,b)\mathrm{vec}(\tilde{W},\tilde{b})\ge {e^{-2\sqrt{2}B^\ast\sqrt{B_{\Theta}^2+1}}\over K}\EE_{X,j}\left[\sum_{k=1}^K(\tilde{W}_{k}^T\hat{\Theta}_{j}(X)+\tilde{b}_k)^2\right].
	$$
	By Assumption~\ref{asp:classification}, we have 
	$$
	\EE_{X,j}\left[\sum_{k=1}^K(\tilde{W}_{k}^T\hat{\Theta}_{j}(X)+\tilde{b}_k)^2\right]\ge \sigma_{\min}(\|\tilde{W}\|_F^2+\|b\|^2).
	$$
	Putting these bounds together yields
	$$
	\mathrm{vec}(\tilde{W},\tilde{b})^T\nabla^2 f(W,b)\mathrm{vec}(\tilde{W},\tilde{b})\ge {e^{-2\sqrt{2}B^\ast\sqrt{B_{\Theta}^2+1}}\over K}\sigma_{\min}(\|\tilde{W}\|_F^2+\|b\|^2).
	$$
	So $f(W,b)$ is a $\sigma_{\min}e^{-2\sqrt{2}B^\ast\sqrt{B_{\Theta}^2+1}}/K$-strongly convex function on $\Bcal$.
	
	Then, we show $\hat{f}(W,b)$ and $f(W,b)$ are close with high probability when $(W,b)$ is near $(W^r,b^r)$. Define 
	$$
	\Psi(r)=\sup_{(W,b)\in \Bcal, \|(W,b)-(W^r,b^r)\|\le r}\left| \hat{f}(W,b)-f(W,b)-\hat{f}(W^r,b^r)+f(W^r,b^r) \right|.
	$$
	If we write $Z_i=\hat{\Theta}_{j}(X_i)$ and 
	\begin{align*}
		\Hcal_r=\bigg\{&h(Z,Y)=-\log\left(\exp(W_{Y}^TZ+b_Y)\over \sum_{k'=1}^K\exp(W_{k'}^TZ+b_{k'}) \right)+\log\left(\exp({W_{Y}^r}^TZ+b_Y^r)\over \sum_{k'=1}^K\exp({W_{k'}^r}^TZ+b_{k'}^r) \right),\\
		&(W,b)\in \Bcal, \|(W,b)-(W^r,b^r)\|\le r\bigg\},
	\end{align*}
	then
	$$
	\Psi(r)=\sup_{h\in \Hcal_r}\left|{1\over n}\sum_{i=1}^nh(Z_i,Y_i)-\EE h(Z,Y)\right|.
	$$
	If we apply the standard symmetrization tool \citep{bartlett2002rademacher}, we have 
	$$
	\EE\sup_{h\in \Hcal_r}\left|{1\over n}\sum_{i=1}^nh(Z_i,Y_i)-\EE h(Z,Y)\right|\le 2\EE_{X,j,Y}\EE_\sigma\sup_{h\in \Hcal_r}{1\over n}\sum_{i=1}^n\sigma_ih(Z_i,Y_i),
	$$
	where $\sigma_i$ is independent Rademacher random sign. Because $h(Z_i,Y_i)=\tilde{h}_i(u(Z_i))$ with
	$$
	u(Z)=\left\{W_{k}^TZ+b_k\right\}_{k=1}^K,\ \tilde{h}_i(u)=-\log {e^{u_{Y_i}}\over \sum_{k'=1}^Ke^{u_{k'}}}+\log\left(\exp({W_{Y_i}^r}^TZ_i+b_{Y_i}^r)\over \sum_{k'=1}^K\exp({W_{k'}^r}^TZ_i+b_{k'}^r) \right),
	$$
	we can apply Maurer’s vector contraction inequality \citep{maurer2016vector}. In fact, $\tilde{h}_i(u)$ is a $\sqrt{2}$-Lipschitz function because
	$$
	|\tilde{h}_i(u)-\tilde{h}_i(u')|=\left|\log {e^{u'_{Y_i}}\over \sum_{k'=1}^Ke^{u'_{k'}}}-\log {e^{u_{Y_i}}\over \sum_{k'=1}^Ke^{u_{k'}}}\right|\le \sqrt{2}\|u-u'\|.
	$$
	Therefore, we have 
	$$
	\EE_\sigma\sup_{h\in \Hcal_r}\sum_{i=1}^n\sigma_ih(Z_i,Y_i)\le 2\EE_\sigma \sup_{u\in \Ucal_r}\sum_{i=1}^n\sum_{k=1}^K\sigma_{ik}u_k(Z_i)\le 2\sum_{k=1}^K\EE_{\sigma_k} \sup_{u\in \Ucal_r^{(k)}}\sum_{i=1}^n\sigma_{ik}u_k(Z_i),
	$$
	where $\Ucal_r^{(k)}=\left\{u_k(Z): \|W_k-W^r_k\|^2+(b_k-b_k^r)^2\le r^2\right\}$ and
	$$
	\Ucal_r=\left\{u(Z):(W,b)\in \Bcal, \|(W,b)-(W^r,b^r)\|\le r\right\}.
	$$
	It is sufficient to bound the Rademacher complexity of $\Ucal_r^{(k)}$
	\begin{align*}
		\EE_{\sigma_k} \sup_{u\in \Ucal_r^{(k)}}\sum_{i=1}^n\sigma_{ik}u_k(Z_i)&=\EE_{\sigma_k}\sup_{\|W_k-W^r_k\|^2+(b_k-b_k^r)^2\le r^2}\sum_{i=1}^n\sigma_{ik}(W_{k}^TZ_i+b_k)\\
		&=\EE_{\sigma_k}\sup_{\|\tilde{W}_k\|^2+(\tilde{b}_k)^2\le r^2}\sum_{i=1}^n\sigma_{ik}(\tilde{W}_{k}^TZ_i+\tilde{b}_k)\\
		&\le \EE_{\sigma_k}\sup_{\|\tilde{W}_k\|^2+(\tilde{b}_k)^2\le r^2}\sqrt{\|\tilde{W}_k\|^2+\tilde{b}_k^2}\sqrt{\|\sum_{i=1}^n\sigma_{ik}Z_i\|^2+(\sum_{i=1}^n\sigma_{ik})^2}\\
		&=r \EE_{\sigma_k}\sqrt{\|\sum_{i=1}^n\sigma_{ik}Z_i\|^2+(\sum_{i=1}^n\sigma_{ik})^2}\\
		&\le r\sqrt{\EE_{\sigma_k}\|\sum_{i=1}^n\sigma_{ik}Z_i\|^2+(\sum_{i=1}^n\sigma_{ik})^2}\\
		&= r\sqrt{\sum_{i=1}^n\|Z_i\|^2+n}\le r\sqrt{n}\sqrt{B_\Theta^2+1}.
	\end{align*}
	Here, we use Cauchy–Schwarz inequality and Jensen's inequality.
	Therefore, 
	$$
	\EE\Psi(r)\le {4Kr\sqrt{B_\Theta^2+1}\over\sqrt{n}}.
	$$
	Because for any $h\in \Hcal_r$, a simialr analysis for $\tilde{h}_i$ suggests
	\begin{align*}
		\sup_{\|Z\|\le B_\Theta,Y}|h(Z,Y)|^2&\le 2\sum_{k=1}^K(W_{k}^TZ+b_k-{W_{k}^r}^TZ-b^r_k)^2\\
		&\le 2\sum_{k=1}^K(\|W_k-W_k^r\|^2+(b_k-b_k^r)^2)(\|Z\|^2+1)\\
		&\le 2(B_\Theta^2+1)(\|W-W^r\|_F^2+\|b-b^r\|^2)\\
		&\le 2r^2(B_\Theta^2+1),
	\end{align*}
	it follows that if we change $(Z_i,Y_i)$ to $(Z'_i,Y'_i)$, then $\Psi(r)$ change at most $2\sqrt{2}r\sqrt{B_\Theta^2+1}/n$. An application of McDiarmid’s inequality \citep{boucheron2013} yields 
	$$
	\PP\left(\Psi(r)-\EE\Psi(r)>t\right)\le \exp\left(-{nt^2\over 4r^2(B_\Theta^2+1)}\right).
	$$
	So 
	$$
	\PP\left(\Psi(r)>{2r\sqrt{B_\Theta^2+1}\over\sqrt{n}}(2K+t)\right)\le \exp\left(-t^2\right).
	$$
	
	Having shown that $\hat{f}(W,b)$ and $f(W,b)$ are close with high probability, we show that  $(\hat{W},\hat{b})$ is close to $(W^r,b^r)$ with high probability. We choose a $t$ such that
	$$
	r_t={6\sqrt{B_\Theta^2+1}\over \mu\sqrt{n}}(2K+t)\le B^\ast.
	$$ 
	We now show that $\|\hat{W}-W^r\|_F^2+\|\hat{b}-b^r\|^2\le r_t^2$ with high probability. The main strategy we use here is contradiction. Suppose that $\|\hat{W}-W^r\|_F^2+\|\hat{b}-b^r\|^2> r_t^2$ on the event $\Ecal_t=\{\Psi(r_t)\le 2r_t\sqrt{B_\Theta^2+1}(2K+t)/\sqrt{n}\}$. We define
	$$
	W_t=W^r+\tilde{r}_t(\hat{W}-W^r)\quad {\rm and}\quad b_t=b^r+\tilde{r}_t(\hat{b}-b^r),
	$$
	where $\tilde{r}_t=r_t/\sqrt{\|\hat{W}-W^r\|_F^2+\|\hat{b}-b^r\|^2}$.
	The construction suggests $(W_t,b_t)\in \Bcal$ because
	$$
	\|W_t\|_F^2+\|b_t\|^2\le \left(\sqrt{\|W^r\|_F^2+\|b^r\|^2}+r_t\right)^2\le 4{B^\ast}^2.
	$$
	On the one hand, we have
	\begin{align*}
		&f^p(W_t,b_t)-f^p(W^r,b^r)\\
		=&\hat{f}^p(W_t,b_t)-\hat{f}^p(W^r,b^r)+(f^p(W_t,b_t)-\hat{f}^p(W_t,b_t))-(f^p(W^r,b^r)-\hat{f}^p(W^r,b^r))\\
		=&\hat{f}^p(W_t,b_t)-\hat{f}^p(W^r,b^r)+(f(W_t,b_t)-\hat{f}(W_t,b_t))-(f(W^r,b^r)-\hat{f}(W^r,b^r))\\
		\le& \hat{f}^p(W_t,b_t)-\hat{f}^p(W^r,b^r)+\Psi(r_t)\\
		=& \hat{f}^p((1-\tilde{r}_t)W^r+\tilde{r}_t\hat{W},(1-\tilde{r}_t)b^r+\tilde{r}_t\hat{b})-\hat{f}^p(W^r,b^r)+\Psi(r_t)\\
		\le& (1-\tilde{r}_t)\hat{f}^p(W^r,b^r)+\tilde{r}_t\hat{f}^p(\hat{W},\hat{b})-\hat{f}^p(W^r,b^r)+\Psi(r_t)\\
		\le &\Psi(r_t)\le {2r_t\sqrt{B_\Theta^2+1}\over \sqrt{n}}(2K+t)={12(B_\Theta^2+1)\over \mu n}(2K+t)^2
	\end{align*}
	Here we use the convexity of $\hat{f}^p$. On the other hand, the strongly convexity of $f$ suggests
	\begin{align*}
		f^p(W_t,b_t)&\ge f^p(W^r,b^r)+\nabla f^p(W^r,b^r)^T\mathrm{vec}(W_t-W^r,b_t-b^r)+{\mu+n^{-1}\over 2}(\|W_t-W^r\|_F^2+\|b_t-b^r\|^2)\\
		&= f^p(W^r,b^r)+{\mu+n^{-1}\over 2}r_t^2.
	\end{align*}
	This suggests
	$$
	f^p(W_t,b_t)-f^p(W^r,b^r)\ge {\mu\over 2}{36(B_\Theta^2+1)\over \mu^2n}(2K+t)^2={18(B_\Theta^2+1)\over \mu n}(2K+t)^2.
	$$
	Putting the two inequality of $f^p(W_t,b_t)-f^p(W^r,b^r)$ leads to the contradiction
	$$
	{18(B_\Theta^2+1)\over \mu n}(2K+t)^2\le f^p(W_t,b_t)-f^p(W^r,b^r)\le {12(B_\Theta^2+1)\over \mu n}(2K+t)^2.
	$$
	So we can know when $r_t<B^\ast$, we have
	$$
	\PP\left(\|\hat{W}-W^r\|_F^2+\|\hat{b}-b^r\|^2>{36(B_\Theta^2+1)\over \mu^2n}(2K+t)^2\right)\le \exp\left(-t^2\right).
	$$

	We are ready to bound the difference between $f(\hat{W},\hat{b})$ and $f(W^r,b^r)$. Again, we do the analysis on the event $\Ecal_t$. The previous analysis suggests 
	$$
	\|\hat{W}\|_F^2+\|\hat{b}\|^2\le \left(\sqrt{\|W^r\|_F^2+\|b^r\|^2}+{6\sqrt{B_\Theta^2+1}\over \mu\sqrt{n}}(2K+t)\right)^2\le 4{B^\ast}^2,
	$$
	so we know $(\hat{W},\hat{b})\in \Bcal$. Because $f(W,b)$ is a $M$-smooth function, it follows that 
	\begin{align*}
		f(\hat{W},\hat{b})&\le f(W^r,b^r)+\nabla f(W^r,b^r)^T\mathrm{vec}(\hat{W}-W^r,\hat{b}-b^r)+{M\over 2}(\|\hat{W}-W^r\|_F^2+\|\hat{b}-b^r\|^2).
	\end{align*}
	Because $(W^r,b^r)$ is the minimizer of $f^p(W,b)$, we have 
	$$
	\nabla f(W^r,b^r)=-{2\over n}(W^r,b^r),
	$$
	which suggests
	\begin{align*}
		\nabla f(W^r,b^r)^T\mathrm{vec}(\hat{W}-W^r,\hat{b}-b^r)&\le {2\over n}\sqrt{\|W^r\|_F^2+\|b^r\|^2}\sqrt{\|\hat{W}-W^r\|_F^2+\|\hat{b}-b^r\|^2}\\
		&\le {12B^\ast\sqrt{B_\Theta^2+1}\over \mu n^{3/2}}(2K+t).
	\end{align*}
	Putting together yields
	\begin{align*}
		f(\hat{W},\hat{b})-f(W^r,b^r)&\le {12B^\ast\sqrt{B_\Theta^2+1}\over \mu n^{3/2}}(2K+t)+{18(B_\Theta^2+1)^2\over \mu^2n}(2K+t)^2\\
		&\le {2{B^\ast}^2\over n}+{18(B_\Theta^2+1)^2\over \mu^2n}(2K+t)^2.
	\end{align*}
	Here, we use $r_t<B^\ast$.
	Thus, it follows that if $r_t<B^\ast$, then
	$$
	\PP\left(f(\hat{W},\hat{b})-f(W^r,b^r)>{2{B^\ast}^2\over n}+{18(B_\Theta^2+1)^2\over \mu^2n}(2K+t)^2\right)\le \exp\left(-t^2\right).
	$$
	Therefore, if $n\ge 72(B_\Theta^2+1)(4K^2+\gamma)/\mu^2{B^\ast}^2$, there exists some constant $C_{c,1}$ and $C_{c,2}$ such that with probability at least $1-e^{-\gamma}$,
	$$
	-\EE_{X,j,Y}\log q^{\hat{W},\hat{b},\hat{\Theta}}_Y(X_{j,c})+\EE_{X,j,Y}\log q^{W^r,b^r,\hat{\Theta}}_Y(X_{j,c})\le {C_{c,1}+C_{c,2}\gamma\over n},
	$$
	where
	$$
	C_{c,1}=2{B^\ast}^2+{144(B_\Theta^2+1)^2K^4\over \sigma_{\min}^2e^{-4\sqrt{2}B^\ast\sqrt{B_{\Theta}^2+1}}}\qquad {\rm and}\qquad C_{c,2}={36(B_\Theta^2+1)^2K^2\over \sigma_{\min}^2e^{-4\sqrt{2}B^\ast\sqrt{B_{\Theta}^2+1}}}.
	$$
	
\end{proof}

\begin{lemma}
	\label{lm:bias}
	Suppose Assumption~\ref{asp:classification} holds. Then, we have
	$$
	-\EE_{X,j,Y}\log q^{W^r,b^r,\hat{\Theta}}_Y(X_{j,c})+\EE_{X,j,Y}\log p^\ast_Y(X_{j,c})\le {2\epsilon_p\over \rho_0}+{D^4B_\Theta^4\over 16}+C_{c,3}\delta_s+{C_{c,4}\over n},
	$$	
	where 
	$$
	C_{c,3}=\log K+2B_EB_\Theta+\log{S_{\max}\over S_{\min}}\quad {\rm and}\quad C_{c,4}=K(B_E^2+\max\{\log^2S_{\min},\log^2S_{\max}\}).
	$$
\end{lemma}
\begin{proof} When the analysis does not involve $X_{j,c}$, we may omit it for simplicity. 
	First, we introduce an oracle linear classifier based on the preferred set $\Wcal_k$ and then show this classifier can help achieve the bound in the conclusion. Define the weight of the oracle linear classifier as
	$$
	W_k^0=\sum_{v\in \Wcal_k}{s_k(v)\over A_k}\hat{E}_v-{1\over K}\sum_{k=1}^K\sum_{v\in \Wcal_k}{s_k(v)\over A_k}\hat{E}_v\qquad {\rm and}\qquad b^0_k=\log A_k-{1\over K}\sum_{k=1}^K\log A_k,
	$$
	where $A_k=\sum_{v'\in \Wcal_k}s_k(v')$.
	The oracle linear classifier is then defined as 
	$$
	h_{W^0,b^0,\hat{\Theta}}(X,j)=\argmax_{1\le k\le K}q^{W^0,b^0,\hat{\Theta}}(X_{j,c})\quad{\rm where}\ q^{W^0,b^0,\hat{\Theta}}(X_{j,c})={\exp({W^0_{k}}^T\hat{\Theta}_{j}(X)+b^0_k)\over \sum_{k'=1}^K \exp({W^0_{k'}}^T\hat{\Theta}_{j}(X)+b^0_{k'})}.
	$$
	
	To study the accuracy of $h_{W^0,b^0,\hat{\Theta}}(X,j)$, we introduce unnormalized evidence score as
	$$
	\tilde{S}_k(X_{j,c})=\sum_{v\in \Vcal}s_k(v)\exp(\hat{E}_{v}^T\hat{\Theta}_{j}(X)).
	$$ 
	Similarly, since $S_k(X_{j,c})=\sum_{v\in \Vcal}s_k(v)p_j(v|X)$, we can also introduce the estimated evidence score as 
	$$
	\hat{S}_{k}(X_{j,c})=\sum_{v\in \Vcal}s_k(v){\exp(\hat{E}_{v}^T\hat{\Theta}_{j}(X))\over \sum_{v'\in \Vcal}\exp(\hat{E}_{v'}^T\hat{\Theta}_{j}(X)) }={\tilde{S}_{k}(X_{j,c}) \over \sum_{v'\in \Vcal}\exp(\hat{E}_{v'}^T\hat{\Theta}_{j}(X)) }.
	$$
	It is clear that 
	$$
	q_k^{\hat{S}}(X_{j,c})={\hat{S}_{k}(X_{j,c})\over \sum_{k'=1}^K \hat{S}_{k'}(X_{j,c})}={\tilde{S}_{k}(X_{j,c})\over \sum_{k'=1}^K \tilde{S}_{k'}(X_{j,c})}.
	$$
	Given the evidence score, we can also define 
	$$
	q_k^{S}(X_{j,c})={S_{k}(X_{j,c})\over \sum_{k'=1}^K S_{k'}(X_{j,c})}={S_{k}(X_{j,c})\over\rho(X_{j,c})}.
	$$
	
	We can rewrite $\tilde{S}_k(X_{j,c})$ as
	$$
	\tilde{S}_k(X_{j,c})=\exp({W^0_k}^T\hat{\Theta}_{j}(X)+b_k^0)\sum_{v\in \Vcal}{s_k(v)\over S_k}\exp((\hat{E}_{v}-W^0_k)^T\hat{\Theta}_{j}(X)).
	$$
	Thus, we have
	$$
	\Delta S(X_{j,c})=\log \tilde{S}_k(X_{j,c})-({W^0_{k}}^T\hat{\Theta}_{j}(X)+b^0_k)=\log\left(\sum_{v\in \Vcal}{s_k(v)\over S_k}\exp((\hat{E}_{v}-W^0_k)^T\hat{\Theta}_{j}(X))\right).
	$$
	By Jensen’s inequality, we obtain
	$$
	\Delta S_k(X_{j,c})\ge \sum_{v\in \Vcal}{s_k(v)\over S_k}(\hat{E}_{v}-W^0_k)^T\hat{\Theta}_{j}(X)=0.
	$$
	On the other hand, we can apply Hoeffding’s lemma \citep{boucheron2013}, which states if $a<Z<b$ and $\EE(Z)=0$, then $\log(\EE(e^Z))\le (b-a)^2/8$, so we have 
	$$
	\Delta S_k(X_{j,c})\le {(2DB_\Theta)^2\over 8}={D^2B_\Theta^2\over 2}.
	$$
	Therefore, it follows that $q^{W^0,b^0,\hat{\Theta}}(X_{j,c})$ is an exponential tilt of $q^{\hat{S}}(X_{j,c})$
	$$
	q_k^{W^0,b^0,\hat{\Theta}}(X_{j,c})={\hat{S}_{k}(X_{j,c})e^{-\Delta S_k(X_{j,c})}\over \sum_{k'=1}^K \hat{S}_{k'}(X_{j,c})e^{-\Delta S_{k'}(X_{j,c})}}={q_k^{\hat{S}}(X_{j,c})e^{-\Delta S_k(X_{j,c})} \over \sum_{k'=1}^Kq_{k'}^{\hat{S}}(X_{j,c})e^{-\Delta S_{k'}(X_{j,c})}  }.
	$$
	If we evaluate Kullback–Leibler divergence between $q^{S}$ and $q^{W^0,b^0,\hat{\Theta}}$, then we have 
	\begin{align*}
		L(q^{S};q^{W^0,b^0,\hat{\Theta}})&=\sum_{k=1}^Kq_k^{S}\log(q_k^{S}/q_k^{W^0,b^0,\hat{\Theta}})\\
		&=\sum_{k=1}^Kq_k^{S}\log(q_k^{S}/q_{k'}^{\hat{S}})+\sum_{k=1}^Kq_k^{S}\Delta S_k+\log\left(\sum_{k'=1}^Kq_{k'}^{\hat{S}}e^{-\Delta S_{k'}}\right)\\
		&=L(q^{S};q^{\hat{S}})+\sum_{k=1}^Kq_k^{S}\Delta S_k+\log\left(\sum_{k'=1}^Kq_{k'}^{\hat{S}}e^{-\Delta S_{k'}}\right).
	\end{align*}
	Similarly, 
	$$
	L(q^{\hat{S}};q^{W^0,b^0,\hat{\Theta}})=\sum_{k=1}^Kq_k^{\hat{S}}\Delta S_k+\log\left(\sum_{k'=1}^Kq_{k'}^{\hat{S}}e^{-\Delta S_{k'}}\right).
	$$
	Therefore, we have 
	$$
	L(q^{S};q^{W^0,b^0,\hat{\Theta}})=L(q^{\hat{S}};q^{W^0,b^0,\hat{\Theta}})+L(q^{S};q^{\hat{S}})+\sum_{k=1}^K(q_k^{S}-q_k^{\hat{S}})\Delta S_k.
	$$
	If we write $Y_k=\sum_{k=1}^Kq_k^{\hat{S}}\Delta S_k-\Delta S_k$ and apply Hoeffding’s lemma, then we show $Y_k\in [\sum_{k=1}^Kq_k^{\hat{S}}\Delta S_k-D^2B_\Theta^2/2,\sum_{k=1}^Kq_k^{\hat{S}}\Delta S_k]$ and 
	$$
	L(q^{\hat{S}};q^{W^0,b^0,\hat{\Theta}})=\log \left(\sum_{k'=1}^Kq_{k'}^{\hat{S}}e^{Y_{k'}}\right)\le {D^4B_\Theta^4\over 32}.
	$$
	In addition, by Pinsker’s inequality \citep{Tsybakov2009}, we obtain
	$$
	\left|\sum_{k=1}^K(q_k^{S}-q_k^{\hat{S}})\Delta S_k\right|\le {D^2B_\Theta^2\over 4}\sum_{k=1}^K|q_k^{S}-q_k^{\hat{S}}|\le {D^2B_\Theta^2\over 2}\sqrt{L(q^{S};q^{\hat{S}})\over 2}
	$$
	Therefore, 
	$$
	L(q^{S};q^{W^0,b^0,\hat{\Theta}})\le \left(\sqrt{L(q^{S};q^{\hat{S}})}+{D^2B_\Theta^2\over 4\sqrt{2}}\right)^2.
	$$
	
	Next, we provide a bound for $L(q^{S};q^{\hat{S}})$. To the end, we introduce a markov kernel from $\Vcal$ to $\{0,1,\ldots,K\}$: $T(k|v)=s_k(v)$ if $1\le k\le K$ and $T(0|v)=1-\sum_{k=1}^Ks_k(v)$. $T(k|v)$ is a valid markov kernel because 
	$$
	\sum_{k=1}^Ks_k(v)={\sum_{k=1}^K\max\{\PP_{X,j}(Y=k|X_j=v)-\tau,0\}\over 1-\tau}\begin{cases}
		=0& {\rm if}\ \PP_{X,j}(Y=k|X_j=v)<\tau\ \forall k\\
		\le 1& {\rm if}\ \PP_{X,j}(Y=k|X_j=v)>\tau\ \exists k
	\end{cases}.
	$$
	We also consider two distributions on $\Vcal$: $p_j(v|X)$ and $\hat{p}_{\hat{E}}(v|\hat{\Theta}_j(X))$. Then, we have
	$$
	Tp_j(v|X)=\begin{cases}
		S_k(X_{j,c}), &1\le k\le K\\
		1-\sum_{k=1}^KS_k(X_{j,c}), &k=0
	\end{cases}
	$$
	and 
	$$
	T\hat{p}_{\hat{E}}(v|\hat{\Theta}_j(X))=\begin{cases}
		\hat{S}_k(X_{j,c}), &1\le k\le K\\
		1-\sum_{k=1}^K\hat{S}_k(X_{j,c}), &k=0
	\end{cases}.
	$$
	If we apply data-processing inequality \citep{cover2006elements}, then we have
	$$
	L(Tp_j(v|X);T\hat{p}_{\hat{E}}(v|\hat{\Theta}_j(X)))\le L(p_j(v|X);\hat{p}_{\hat{E}}(v|\hat{\Theta}_j(X))).
	$$
	If we write $\hat{\rho}(X_{j,c})=\sum_{k=1}^K\hat{S}_{k}(X_{j,c})$, we obtain 
	\begin{align*}
		&L(Tp_j(v|X);T\hat{p}_{\hat{E}}(v|\hat{\Theta}_j(X)))\\
		=&(1-\rho(X_{j,c}))\log{1-\rho(X_{j,c})\over 1-\hat{\rho}(X_{j,c})}+\sum_{k=1}^KS_k(X_{j,c})\log{S_k(X_{j,c})\over \hat{S}_k(X_{j,c})}\\
		=&(1-\rho(X_{j,c}))\log{1-\rho(X_{j,c})\over 1-\hat{\rho}(X_{j,c})}+ \sum_{k=1}^K\rho(X_{j,c})q^S_k(X_{j,c})\log{\rho(X_{j,c})q^S_k(X_{j,c})\over \hat{\rho}(X_{j,c})q^{\hat{S}}_k(X_{j,c})}\\
		=&\rho(X_{j,c})\sum_{k=1}^Kq^S_k(X_{j,c})\log{q^S_k(X_{j,c})\over q^{\hat{S}}_k(X_{j,c})}+\rho(X_{j,c})\sum_{k=1}^Kq^S_k(X_{j,c})\log{\rho(X_{j,c})\over \hat{\rho}(X_{j,c})}\\
		&+(1-\rho(X_{j,c}))\log{1-\rho(X_{j,c})\over 1-\hat{\rho}(X_{j,c})}\\
		\ge& \rho(X_{j,c})\sum_{k=1}^Kq^S_k(X_{j,c})\log{q^S_k(X_{j,c})\over q^{\hat{S}}_k(X_{j,c})}\\
		=&\rho(X_{j,c})L(q^{S};q^{\hat{S}})
	\end{align*}
	Therefore, we have
	$$
	\rho_0\EE_{X,j}L(q^{S};q^{\hat{S}})\le \epsilon_p,
	$$
	which immediately suggests
	$$
	\EE_{X,j}L(q^{S};q^{W^0,b^0,\hat{\Theta}})\le \left(\sqrt{\epsilon_p\over \rho_0}+{D^2B_\Theta^2\over 4\sqrt{2}}\right)^2.
	$$
	
	Next, we aim to show a bound for $L(p^{\ast};q^{W^0,b^0,\hat{\Theta}})$. Let 
	$$
	\delta_s(X_{j,c})=1-\min_{k:S_k(X_{j,c})/\rho(X_{j,c})>0}{p^\ast_k(X_{j,c})\over S_k(X_{j,c})/\rho(X_{j,c})}.
	$$
	and 
	$$
	q^{\delta_s}_k(X_{j,c})=\begin{cases}
		[p^\ast_k(X_{j,c})-(1-\delta_s(X_{j,c}))q^S_k(X_{j,c})]/\delta_s(X_{j,c}),&{\rm if }\ \delta_s(X_{j,c})>0\\
		q^S_k(X_{j,c}),& {\rm if }\ \delta_s(X_{j,c})=0.
	\end{cases}
	$$
	By the construction, it is clear that $p^\ast_k(X_{j,c})=\delta_s(X_{j,c})q^{\delta_s}_k(X_{j,c})+(1-\delta_s(X_{j,c}))q^S_k(X_{j,c})$.
	Since Kullback–Leibler divergence is convex when the second distribution is fixed, we obtain
	$$
	L(p^\ast;q^{W^0,b^0,\hat{\Theta}})\le (1-\delta_s(X_{j,c}))L(q^S;q^{W^0,b^0,\hat{\Theta}})+\delta_s(X_{j,c})L(q^{\delta_s};q^{W^0,b^0,\hat{\Theta}}).
	$$
	Because $\log q^{\delta_s}_k\le 0$, we have 
	$$
	L(q^{\delta_s};q^{W^0,b^0,\hat{\Theta}})=\sum_{k=1}^Kq^{\delta_s}_k\log{q^{\delta_s}_k\over q_k^{W^0,b^0,\hat{\Theta}}}\le \log{1\over \min_{1\le k\le K} q_k^{W^0,b^0,\hat{\Theta}}}
	$$
	By the definition,
	\begin{align*}
		\min_{1\le k\le K} q_k^{W^0,b^0,\hat{\Theta}}(X_{j,c})&=\min_{1\le k\le K}{\exp({W^0_{k}}^T\hat{\Theta}_{j}(X)+b^0_k)\over \sum_{k'=1}^K \exp({W^0_{k'}}^T\hat{\Theta}_{j}(X)+b^0_{k'})}\\
		&\ge {\min_{1\le k\le K}\exp({W^0_{k}}^T\hat{\Theta}_{j}(X)+b^0_k)\over K\max_{1\le k'\le K} \exp({W^0_{k'}}^T\hat{\Theta}_{j}(X)+b^0_{k'})}\\
		&\ge \exp(-2B_EB_\Theta+\log S_{\min}-\log S_{\max})/K
	\end{align*}
	Therefore,  
	$$
	L(p^\ast;q^{W^0,b^0,\hat{\Theta}})\le L(q^S;q^{W^0,b^0,\hat{\Theta}})+\delta_s(X_{j,c})\left(\log K+2B_EB_\Theta+\log{S_{\max}\over S_{\min}}\right).
	$$
	If we take the expectation on the both sides, we have 
	$$
	\EE_{X,j} L(p^\ast;q^{W^0,b^0,\hat{\Theta}})\le \left(\sqrt{\epsilon_p\over \rho_0}+{D^2B_\Theta^2\over 4\sqrt{2}}\right)^2+\delta_s\left(\log K+2B_EB_\Theta+\log{S_{\max}\over S_{\min}}\right).
	$$
	
	Because $(W^r,b^r)$ is the minimizer of penalized risk, we can know 
	$$
	-\EE_{X,j,Y}\log q^{W^r,b^r,\hat{\Theta}}_Y(X_{j,c})+{1\over n}\left(\|W^r\|_F^2+\|b^r\|^2\right)\le -\EE_{X,j,Y}\log q^{W^0,b^0,\hat{\Theta}}_Y(X_{j,c})+{1\over n}\left(\|W^0\|_F^2+\|b^0\|^2\right).
	$$
	So we obtain 
	\begin{align*}
		&-\EE_{X,j,Y}\log q^{W^r,b^r,\hat{\Theta}}_Y(X_{j,c})+\EE_{X,j,Y}\log p^\ast_Y(X_{j,c})\\
		\le& -\EE_{X,j,Y}\log q^{W^0,b^0,\hat{\Theta}}_Y(X_{j,c})+\EE_{X,j,Y}\log p^\ast_Y(X_{j,c})+{1\over n}\left(\|W^0\|_F^2+\|b^0\|^2\right)\\
		\le& \EE_{X,j}L(p^\ast;q^{W^0,b^0,\hat{\Theta}})+{1\over n}\left(\|W^0\|_F^2+\|b^0\|^2\right).
	\end{align*}
	Because 
	$$
	\|W^0\|_F^2+\|b^0\|^2\le K(B_E^2+\max\{\log^2S_{\min},\log^2S_{\max}\}),
	$$
	so we can know there exists constants $C_{c,3}$, $C_{c,4}$ such that 
	$$
	-\EE_{X,j,Y}\log q^{W^r,b^r,\hat{\Theta}}_Y(X_{j,c})+\EE_{X,j,Y}\log p^\ast_Y(X_{j,c})\le {2\epsilon_p\over \rho_0}+{D^4B_\Theta^4\over 16}+C_{c,3}\delta_s+{C_{c,4}\over n},
	$$	
	where 
	$$
	C_{c,3}=\log K+2B_EB_\Theta+\log{S_{\max}\over S_{\min}}\quad {\rm and}\quad C_{c,4}=K(B_E^2+\max\{\log^2S_{\min},\log^2S_{\max}\}).
	$$
\end{proof}

\section{Details of the Numerical Illustrations}
\label{app:numerical}

\subsection{Data Generation Details}

In the simulation experiment there are $K=4$ semantic communities, $\Vcal_{c,1},\ldots,\Vcal_{c,4}$, $16$ semantic groups, $\Vcal_{s,1},\ldots,\Vcal_{s,16}$, and three exchangeable token types per
semantic group, giving vocabulary size $d=48$. We choose
$$
\Vcal_{s,4k-3},\ldots, \Vcal_{s,4k}\subset \Vcal_{c,k},\qquad k=1,\ldots, 4.
$$
The three exchangeable token types of every semantic group occur with probabilities
$$
\rho_m=\frac{m^{-0.8}}{\sum_{r=1}^3r^{-0.8}},\quad m=1,2,3,
\qquad (\rho_1,\rho_2,\rho_3)\simeq(0.503,0.289,0.209).
$$
Because these weights do not depend on context, the exchangeable token types of every semantic group have unequal marginal frequencies but the same conditional contextual distributions. Hence $\Gamma_h(v_1,v_2)=0$ exactly if $v_1,v_2\in \Vcal_{s,a}$ for some $a$.

We assume there exists an underlying semantic geometry. Specifically, we draw four orthonormal vectors as community centers $\mu_1,\ldots,\mu_4\in \RR^{16}$. We then draw a unit vector $\delta_a$ for each semantic group, $a=1,\ldots,16$, in the orthogonal complement of their span. Within each community we assign the four strengths $\alpha_a\in\{0.15,0.25,0.35,0.45\}$ to the four semantic groups and set
$$
u_a=\frac{\mu_{k}+\alpha_a\delta_a}{\|\mu_{k}+\alpha_a\delta_a\|},
\qquad a=4k-3,\ldots,4k,\quad k=1,\ldots,4.
$$
These hidden vectors connect a token to its context. Write the context as $X_{j,c}=(X_{j-3},X_{j-2},X_{j-1})$ and the semantic group of each contextual
token as $(a_{j-3},a_{j-2},a_{j-1})$. The context feature is
$$
c(X_{j,c})=\frac{\tilde c(X_{j,c})}{\|\tilde c(X_{j,c})\|},\qquad
\tilde c(X_{j,c})=\sum_{i=1}^{3}\omega_i R_i u_{a_{j-4+i}}
+0.25\,B\bigl(u_{a_{j-3}}\odot u_{a_{j-2}}+u_{a_{j-2}}\odot u_{a_{j-1}}\bigr),
$$
with position weights $(\omega_1,\omega_2,\omega_3)=(0.2,0.3,0.5)$ ordered from the oldest contextual position to the most recent, where $R_1,R_2,R_3$ are orthogonal matrices drawn independently for each position, $B$ is a matrix with operator norm one, and $\odot$ denotes elementwise multiplication. The unequal position weights, the positional orthogonal matrices, and the interaction term together make the conditional distribution depend on both the order of and the interactions among the three previous tokens.

To ensure an irreducible and aperiodic process we introduce an order-three de Bruijn cycle on the $16$ semantic groups. Let $r_{B}(X_{j,c})$ be the successor of $(a_{j-3},a_{j-2},a_{j-1})$ on this cycle. Let $\Gcal(X_{j,c})$ contain the five remaining semantic groups with the largest compatibility scores
$c(X_{j,c})^{T}Ou_a$, where $O$ is a fixed orthogonal matrix, and let
$$
\mathcal S_{a}(X_{j,c})=\{r_{B}(X_{j,c})\}\cup \Gcal(X_{j,c}).
$$
At the single designated context $(a_{j-3},a_{j-2},a_{j-1})=(1,1,1)$ we force semantic group $1$ into $\Gcal$ and compete only the remaining four slots, which creates a self-loop and makes the chain aperiodic. Every context therefore has exactly $6$ plausible semantic groups and $18$ plausible tokens, and the exploration edge $r_B$ is always distinct from the five geometric slots. The
conditional probability is
$$
p_j(v|X_{j,c})=\begin{cases}
	\eta\,\rho_m, & v\in \Vcal_{s,r_{B}(X_{j,c})},\\[2mm]
	(1-\eta)\displaystyle
	\frac{\exp\{\beta c(X_{j,c})^TOu_a\}}
	{\sum_{a'\in \Gcal(X_{j,c})}\exp\{\beta c(X_{j,c})^TOu_{a'}\}}\,\rho_m,
	& v\in \Vcal_{s,a},\ a\in \Gcal(X_{j,c}),\\[4mm]
	0, & \text{otherwise},
\end{cases}
$$
with exploration mass $\eta=0.2$. The scale $\beta$ is fixed before any model is fitted, by matching the lower $1\%$ quantile of the $4096\times 5$ geometric-branch token probabilities at the rarest variant, $\{(1-\eta)p_j(a|X_{j,c};\beta)\rho_3\}$, to $1/54$. The quantile is decreasing in $\beta$, so the value is found by bisection; it is $\beta\simeq1.664$. The exploration branch is excluded from the quantile deliberately: at $\beta=0$ every geometric slot already carries mass $(1-\eta)/5$, so including the flat exploration values would collapse the calibration to $\beta=0$ and remove the geometric signal entirely.

Given $p_j(v|X_{j,c})$ we obtain the marginal distribution of the context $p(X_{j,c})$ by solving
\begin{align*}
	&p(X_{j-2}=v_2,X_{j-1}=v_1,X_{j}=v_0)\\
	=&\sum_{v_3\in \Vcal}p_j(v_0|X_{j-3}=v_3,X_{j-2}=v_2,X_{j-1}=v_1)p(X_{j-3}=v_3,X_{j-2}=v_2,X_{j-1}=v_1),
\end{align*}
by power iteration to an $\ell_1$ residual below $10^{-12}$. With $p_j(v|X_{j,c})$ and $p(X_{j,c})$ we evaluate every quantity of interest exactly, by enumeration over the $16^3=4096$ semantic contexts or the $48^3=110592$ observed-token contexts.

\subsection{Model Fitting Details}

The representation approximator uses causal attention over four input positions: the three contextual tokens and one query slot, from which the representation is estimated. Each Transformer sub-layer uses two attention heads, representation width $16$, a width-$64$ GELU feed-forward network, two pre-normalizations, and no dropout. The initial representation and the token embeddings in the recovery
head are tied. The representation is read from the query slot by exact normalization, $h=4r/\|r\|$, so $\|h\|=4$ identically rather than approximately; this is why the ideal representation $Z_j(X)$ in Section~\ref{sc:numerical} is constrained to the same ball. The augmented recovery-head rows $(E_v,b_v)$ are constrained to a ball of radius $4$ and projected onto it after every update.

We use two block families. The first shared block contains one Transformer sub-layer and is repeated $L\in\{1,2,4,8\}$ times; the second contains two sub-layers and is repeated $L\in\{1,2,4\}$ times. All parameters inside a block are shared across its $L$ applications, so comparisons across $L$ within a family change depth without changing the parameter count. The one- and two-layer families have $4112$ and $7328$ parameters respectively, at every $L$. For every architecture we use three matched seeds. A seed fixes the training count table, the validation sample, the parameter initialization, and the minibatch stream, and these are held identical across depths within a block family. Initialization is deliberately depth-independent: all weight matrices and token embeddings are drawn $N(0,1/16)$ and no $1/\sqrt{2L}$ residual scaling is applied, so a given seed produces byte-identical initial parameters at every $L$ and the depth comparison is paired. 

The training and validation samples contain $5\times 10^5$ and $2\times 10^4$ context and next-token pairs. Models are optimized with AdamW at a constant learning rate $3\times10^{-4}$ with weight decay
$10^{-4}$, batch size $512$ drawn i.i.d. with replacement, and gradient clipping at one. Training runs for at most $3\times10^5$ updates, with validation every $50$ updates and early stopping after $5\times10^4$ updates without improvement. Checkpoints are evaluated at updates
$$
0,20,50,100,200,400,800,2000,5000,10^4,2\times10^4,5\times10^4, 10^5,
2\times10^5,
$$
together with the selected final update, defined as the argmin of validation loss over the evaluated updates with ties broken toward the earliest. All computation is in double precision on CPU.

\end{appendices}

\end{document}